\documentclass[Afour, sageh, times]{style/sagej}

\usepackage{tikz}
\usetikzlibrary{arrows}
\usetikzlibrary{arrows.meta, positioning}

\usepackage[utf8]{inputenc}
\usepackage[notransparent]{svg}

\usepackage{amssymb}
\usepackage{tabularx}
\usepackage{multirow}

\usepackage{subcaption}

\usepackage{listings}
\usepackage{algorithm}
\usepackage{algpseudocode}
\usepackage{hyperref}
\usepackage{xcolor}
\usepackage{amsmath}
\usepackage{amsfonts}
\usepackage{ifmtarg}
\usepackage{listings}
\usepackage{siunitx}
\usepackage{graphbox}
\usepackage{xspace}
\usepackage[toc,page]{appendix}
\usepackage{balance}

\def\volumeyear{2025}
\runninghead{Orthey and Pokorny}

\usepackage{enumitem}
\newlist{todolist}{itemize}{2}
\setlist[todolist]{label=$\square$}
\usepackage{pifont}
\def\R{\mathbb{R}}
\def\N{\mathbb{N}_{\geq 0}}
\def\SO{SO}
\def\SE{SE}

\renewcommand\SO[1]{\ensuremath{\mathrm{SO}(#1)}}
\renewcommand\SE[1]{\ensuremath{\mathrm{SE}(#1)}}
\renewcommand\S{\ensuremath{\mathbb{S}}}

\newcommand\Rnonneg{\ensuremath{\mathrm{R_{\geq 0}}}}

\def\X{X}
\def\xi{x_I}
\def\xg{X_G}
\def\PriorityQueue{\ensuremath{\mathbf{X}}}

\def\tree{T}
\newtheorem{theorem}{Theorem}

\def\xb{x_{\text{base}}}

\def\xbase{\xb}

\def\pb{p_{\text{base}}}

\def\FibrationRRT{Fibration-RRT\xspace}
\DeclareMathOperator{\atantwo}{atan2}

\algnewcommand{\algorithmicparams}{\textbf{Parameters:}}
\algnewcommand\Params{\item[\algorithmicparams]}

\algnewcommand\True{\textbf{true}\xspace}
\algnewcommand\False{\textbf{false}\xspace}

\makeatletter
\renewcommand{\toprule}{\hrule height.8pt depth0pt \kern2pt} 
\renewcommand{\midrule}{\kern2pt\hrule\kern2pt} 
\renewcommand{\bottomrule}{\kern2pt\hrule\relax}
\newcommand{\algcaption}[2][]{%
  \refstepcounter{algorithm}%
\textbf{{\raggedright\fname@algorithm~\thealgorithm}}\ #2\par 
  \midrule
}
\makeatother
\algnewcommand{\algorithmicbreak}{\textbf{break}}
\algnewcommand{\BREAK}{\algorithmicbreak}
\algnewcommand{\algorithmicor}{\textbf{ or }}
\algnewcommand{\OR}{\algorithmicor}
\algnewcommand{\algorithmicand}{\textbf{ and }}
\algnewcommand{\AND}{\algorithmicand}

\makeatletter
\newcommand*{\algrule}[1][\algorithmicindent]{\makebox[#1][l]{\hspace*{.5em}\vrule height .75\baselineskip depth .25\baselineskip}}%

\newcount\ALG@printindent@tempcnta
\def\ALG@printindent{%
    \ifnum \theALG@nested>0
    \ifx\ALG@text\ALG@x@notext
    \addvspace{-3pt}
    \else
    \unskip
    \ALG@printindent@tempcnta=1
    \loop
    \algrule[\csname ALG@ind@\the\ALG@printindent@tempcnta\endcsname]%
    \advance \ALG@printindent@tempcnta 1
    \ifnum \ALG@printindent@tempcnta<\numexpr\theALG@nested+1\relax
    \repeat
    \fi
    \fi
}%
\usepackage{etoolbox}
\patchcmd{\ALG@doentity}{\noindent\hskip\ALG@tlm}{\ALG@printindent}{}{\errmessage{failed to patch}}
\makeatother
\makeatletter
\algnewcommand{\LineComment}[1]{\Statex \hskip\ALG@thistlm \(\triangleright\) #1}
\makeatother

\definecolor{parallel_color}{RGB}{220,90,0}
\definecolor{sequential_color}{RGB}{0,120,0}
\definecolor{partial_color}{RGB}{128,0,128}

\newcommand{\sqboxs}{1.5ex}
\newcommand{\sqbox}[1]{\textcolor{#1}{\rule{\sqboxs}{\sqboxs}}}
\def\parallelColorBox{\sqbox{parallel_color}}
\def\sequentialColorBox{\sqbox{sequential_color}}
\def\partialColorBox{\sqbox{partial_color}}

\tikzset{
    inner sep=2mm,
    outer sep=0mm,
    sibling distance=4em,
    none/.style={
        line width = 0pt, 
        draw=none,
        edge from parent/.style={none}
    },
    partial/.style={
        color=partial_color,
        edge from parent/.style={color = partial_color, line width = 1pt, draw}
    },
    sequential/.style={
        color=sequential_color,
        edge from parent/.style={color = sequential_color, line width = 1pt, draw}
    },
    parallel/.style={
        color=parallel_color,
        edge from parent/.style={color = parallel_color, line width = 1pt, draw}
    },
    root/.style={
    color=black,
    line width=1.5pt,
   }
}
\definecolor{nonexisting_property}{rgb}{0.95, 0.95, 0.95}
\definecolor{existing_property}{rgb}{0.01, 0.75, 0.24}
\definecolor{unknown_property}{rgb}{0.9, 0.9, 0.9}

\newcommand{\sqboxsTable}{1.5ex}
\newcommand{\sqboxTable}[1]{\textcolor{#1}{\rule{\sqboxsTable}{\sqboxsTable}}}
\def\markYes{\sqboxTable{existing_property}}
\def\markNo{\sqboxTable{nonexisting_property}}

\newlength{\subfigheight}
\newlength{\scenarioheight}

\definecolor{colorFTD}{RGB}{76,178,76} 
\definecolor{colorFTP}{RGB}{102,229,102} 
\definecolor{colorRRT}{RGB}{255,178,102} 
\definecolor{colorLBT}{RGB}{255,237,140} 
\definecolor{colorEST}{RGB}{255,107,104} 
\definecolor{colorFMT}{RGB}{199,217,239} 

\definecolor{colorQRRT}{RGB}{255,178,102}
\definecolor{colorDiscreteRRT1}{RGB}{255,230,102}
\definecolor{colorDiscreteRRT5}{RGB}{255,204,77}
\definecolor{colorDiscreteRRT10}{RGB}{255,179,51}
\definecolor{colorDiscreteRRT50}{RGB}{255,153,26}

\newcommand\BibTeX{{\rmfamily B\kern-.05em \textsc{i\kern-.025em b}\kern-.08em
T\kern-.1667em\lower.7ex\hbox{E}\kern-.125emX}}

\def\volumeyear{2026}
\begin{document}

\title{Fibration Trees: A Unified Approach to Multi-Robot Motion Planning}

\author{Andreas Orthey\affilnum{1} and Florian T. Pokorny\affilnum{2} and Lydia E. Kavraki\affilnum{3}}
\runninghead{Orthey et al.}

\affiliation{
\noindent\affilnum{1}Technical University of Berlin, Germany\\
\affilnum{2}KTH Royal Institute of Technology, Stockholm, Sweden\\
\affilnum{3}Rice University and the Ken Kennedy Institute, Houston, TX, US
}
\corrauth{Andreas Orthey}
\email{\{andreas\}@orthey.net}

\newcommand\todo[1]{{\color{blue}{\textbf{TODO:} #1}}}

\begin{abstract}
    State space projections and decompositions have emerged as powerful tools to
    tackle the curse of dimensionality in high-dimensional, multi-robot motion planning problems.
    However, existing methods lack a unified framework which seamlessly handles combinations of projections (prioritization or task-space) and decompositions (parallel or decoupled subspaces). 
    To fill this gap, we introduce fibration trees, which are trees consisting of state spaces as nodes and fibrations as edges, whereby a fibration models a projection from a higher-dimensional space to a lower-dimensional (or simplified) space.    
    By modeling projections as fibrations, we unify sequential prioritization, parallel decomposition, and task-space projections under a single, coherent formalism.
    Building on this, we develop the rapidly-exploring random fibration trees (Fibration-RRT) planner, a sampling-based motion planner that generalizes strategies from quotient-space RRT (for sequential prioritizations) and discrete RRT (for parallel decompositions), while allowing the inclusion of task-space projections. Fibration-RRT operates on user-defined fibration trees and is proven to be probabilistically complete.
    To test the generality and efficiency of Fibration-RRT, we provide an open-source implementation and conduct experiments on 32 scenarios using multi robot teams with up to 96 degrees of freedom. 
    Our results indicate that Fibration-RRT efficiently solves high-dimensional problems by exploiting user-defined fibration trees, thereby establishing fibration trees as a powerful, unified framework
    for multi-robot motion planning.
\end{abstract}


\maketitle

\section{Introduction}

State space projections and decompositions are useful tools to break the complexity of high-dimensional motion planning and optimization problems. In the context of sampling-based motion planning, projections and decompositions have been used for multi-robot scenarios, where state spaces can have tens or hundreds of dimensions, like in the coordination of robot swarms~\citep{dorigo2021swarm, kaiser2022innate},
spacecrafts~\citep{Nino2025Spacecraft}, robot construction
teams~\citep{Hartmann2021TRO}, underwater vehicles~\citep{zhou2022survey}, or quadrotor swarms~\citep{hoenig_2018}.
In previous works, efficient solvers have been devised to solve heterogeneous
multi-robot planning problems using projections (prioritized planning)~\citep{Orthey2024IJRR, Vidal2019, Ferbach1997} or decompositions (parallel or decoupled planning)~\citep{Solovey2016, Wagner2015, Shome2020}. To us, those are examples of hierarchical representations which have been imposed onto the state space of a problem to more efficiently solve high-dimensional, multi-robot motion planning problems.

\begin{figure}[t]
\centering
  \resizebox{0.8\linewidth}{!}{
\begin{tikzpicture}[
    every node/.style={inner sep=0}
]
    \node (fibration) {\includegraphics[width=1.0\linewidth]{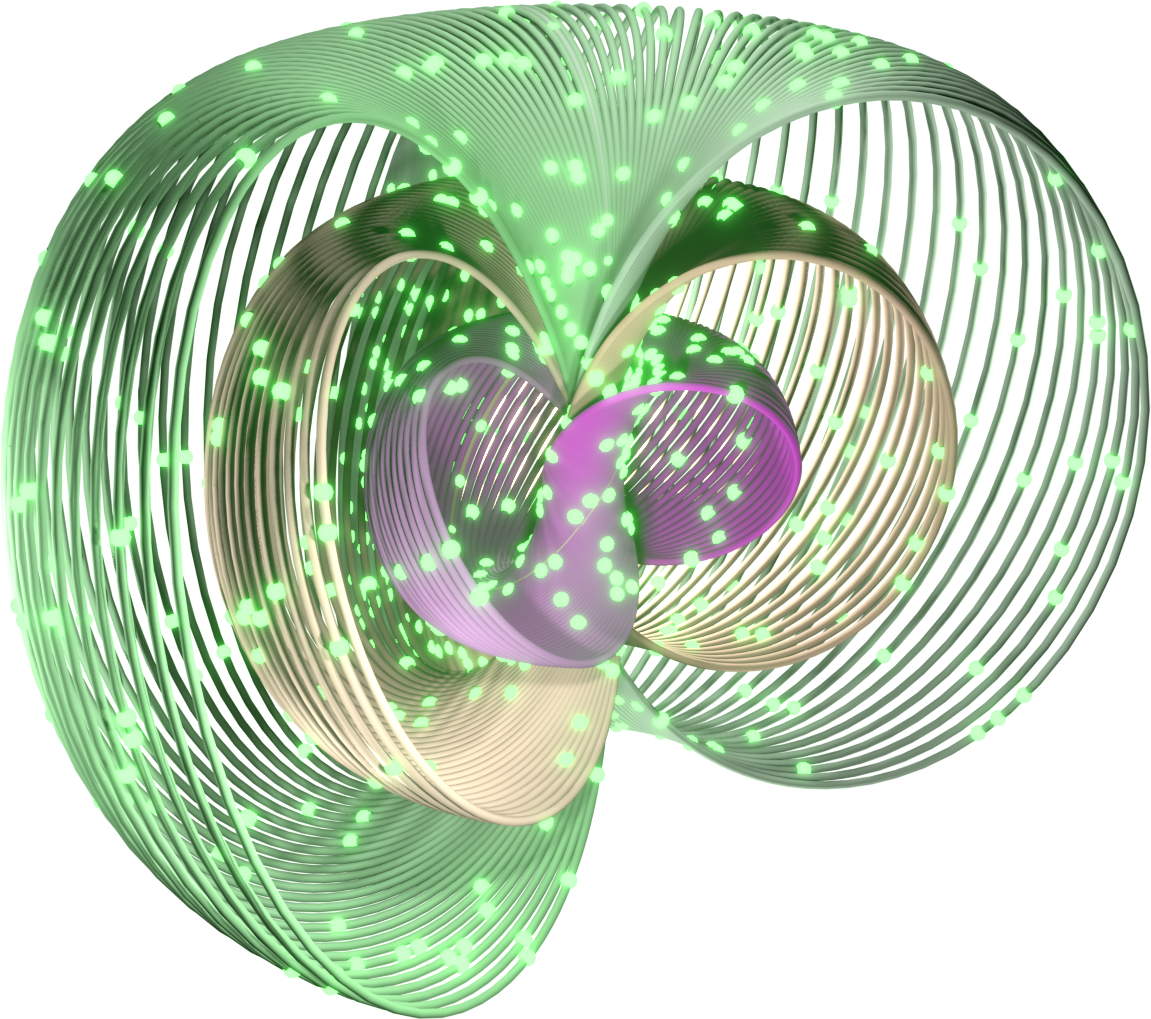}};

    \node (sphere) [below right=-2.0cm and -0.0cm of fibration.south east, anchor=north west]
        {\includegraphics[width=0.3\linewidth]{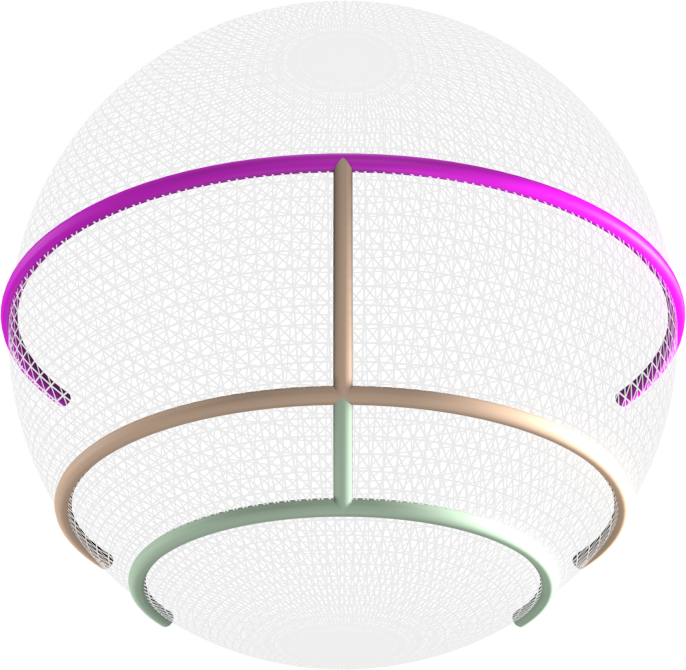}};

    \draw[->, ultra thick, bend right=25, >=Latex]
        ([xshift=-3.0cm,yshift=+1.3cm]fibration.south east) to
        ([xshift=-0.2cm,yshift=-1.5cm]sphere.north west);
\end{tikzpicture}
}
\caption{
We simplify multi-robot motion planning problems using fibration trees, which are hierarchical representations of state spaces where nodes model state space simplifications. An example is the Hopf fibration $\S^3 \rightarrow \S^2$, which
  can be understood as a simplification of $\S^3$ by $\S^2$. 
  A graph on $\S^2$ is shown (bottom right), with the graph restriction on
  $\S^3$ visualized as the fibers of the Hopf fibration.
  The color gradient of the edges on $\S^2$ reflects the corresponding fibers on $\S^3$.
}
\vspace*{-0.3cm}
\label{fig:pullfigure}
\end{figure}

\def\textIndent{\indent\hspace{0.5cm}}

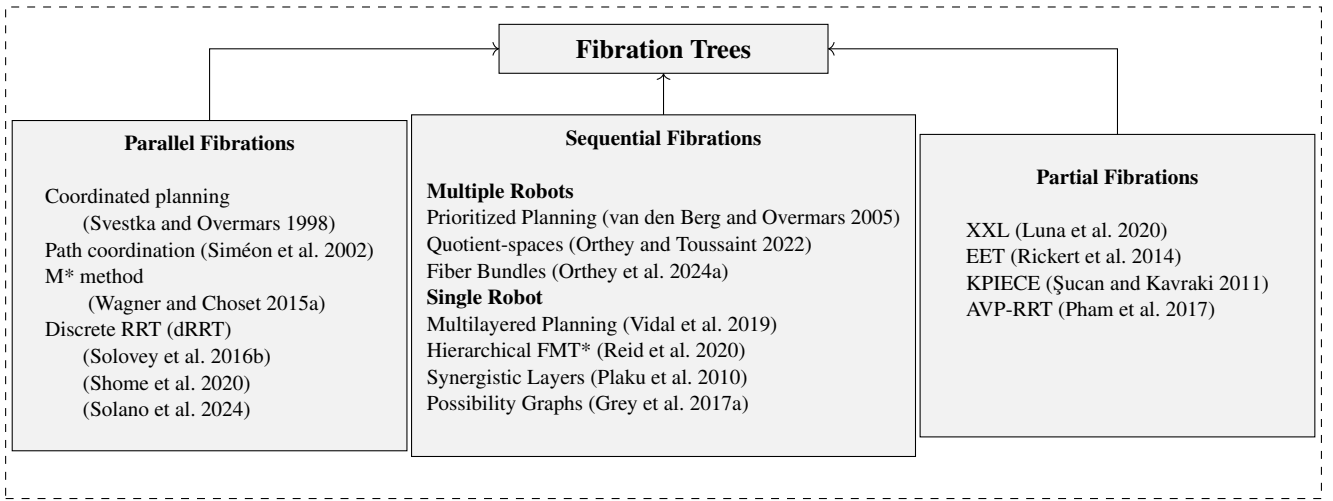
\begin{figure*}[t]
\centering
\tikzstyle{fullblock}=[draw, fill=white!50, 
text centered, dashed, minimum height=6.5cm, minimum width=\linewidth]
\tikzstyle{block}=[draw, fill=gray!10, 
text centered, text depth=2cm, minimum height=4cm, minimum width=0.3\linewidth]
\begin{tikzpicture}
\node[fullblock] (main) at (0,+0.025\linewidth) {\small Fibration Trees};
\node[block] (parallel) at (-0.345\linewidth,0) {
\begin{minipage}{0.25\linewidth}
\footnotesize 
\begin{center}
\textbf{Parallel Fibrations}
\end{center}
Coordinated planning\\
\textIndent\citep{svestka_1998}\\
Path coordination~\citep{simeon_2002}\\
M* method\\
\textIndent~\citep{Wagner2015}\\
Discrete RRT (dRRT)\\
\textIndent\citep{solovey_2016}\\
\textIndent\citep{Shome2020}\\
\textIndent\citep{Solano2024}
\end{minipage}
};

\node[block] (sequential) at (-0.0\linewidth,0) {
\begin{minipage}{0.36\linewidth}
\begin{center}
\textbf{\footnotesize Sequential Fibrations}\\[0.3cm]    
\end{center}
\footnotesize
\textbf{Multiple Robots}\\
Prioritized Planning~\citep{VanDenBerg2005Prioritized}\\
Quotient-spaces~\citep{Orthey2019ISRR}\\
Fiber Bundles~\citep{Orthey2024IJRR}\\
\textbf{Single Robot}\\
Multilayered Planning~\citep{Vidal2019}\\
Hierarchical FMT*~\citep{Reid2020}\\
Synergistic Layers~\citep{Plaku2010motion}\\
Possibility Graphs~\citep{Grey2017Footstep}\\
\end{minipage}
};
\node[block] (partial) at (+0.345\linewidth,0) {
\begin{minipage}{0.23\linewidth}
\begin{center}
\textbf{\footnotesize Partial Fibrations}\\[0.3cm]    
\end{center}
\footnotesize
XXL~\citep{Luna2020XXL}\\
EET~\citep{Rickert2014}\\
KPIECE~\citep{Sucan2011}\\
AVP-RRT~\citep{pham2017admissible}
\end{minipage}
};
\node[draw, fill=gray!10, 
text centered, minimum width=0.25\linewidth] (fibtree) at (0,+0.18\linewidth) {\textbf{Fibration Trees}};
\draw[->] (parallel.north) |- (fibtree.west);
\draw[->] (sequential.north) -- (fibtree.south);
\draw[->] (partial.north) |- (fibtree.east);
\end{tikzpicture}
\caption{Fibration trees as unification tool to represent parallel, sequential, and partial fibrations. For each fibration, we list a non-exhaustive list of methods which make use of this particular instantiation of a fibration tree together with a tailor-made solver to exploit the particular structure. Please see Sec.~\ref{sec:related-work} for a more comprehensive list of methods and Sec.~\ref{sec:fibration-types} for related definitions.
\label{fig:fibration-trees-as-unification}}
\end{figure*}

However, those hierarchical representations for multi-robot motion planning have so far been treated separately and lack a unified framework.
This is problematic.
First, when a user wants to change the hierarchical representation, they would also need the change the planner associated with it. This is inconvenient and makes it difficult to benchmark multi-robot planners, because we do not understand if the change in hierarchical representation or the change in the planner contributes to a lower runtime.
Second, it is currently not possible to switch hierarchical representations seamlessly without changing the planner. For example, when users understand a planning problem better and want to change the representation to improve runtime, they would also need to change the planner itself, because planners in the literature are often tailor-made for a particular representation (e.g. rapidly-exploring random quotient-space trees (QRRT)~\citep{Orthey2024IJRR} for sequential prioritizations or discrete-RRT (dRRT)~\citep{Shome2020} for decoupled representations).
Finally, there is no unified way to create hierarchical representations for multi-robot motion planning problems in the Open Motion Planning Library (OMPL)~\citep{Sucan2012}. This makes it difficult to select, compare, and evaluate different representations to multi-robot motion planning and to improve upon them in the robotics community.

To close those gaps, we propose the framework of fibration trees. Fibration
trees are trees consisting of state spaces as nodes and
fibrations~\citep{hatcher2002algebraic} as edges, whereby a fibration models a
projection from a
higher-dimensional space to a lower-dimensional (or simplified) space. An example is the Hopf
fibration~\citep{hopf1931abbildungen} (see Fig.~\ref{fig:pullfigure}), which is a fibration from $\S^3$ to $\S^2$,
which we use as a simplification of the space $\SO{3}$ by the space $\S^2$ (see
details in Appendix~\ref{appendix:hopffibration}). In
this example, we leverage the fibration by planning on $\S^2$
to create a tree on the sphere (visualized in lower right corner). By using this
tree, we then build up the space of
points on $\S^3$ which project onto the tree (called a restriction). By sampling exclusively this restriction, we can bias sampling towards the existing tree structure, which is similar to the way an admissible heuristic prioritizes paths which have lower cost on the simplified space~\citep{Orthey2019ISRR}. The resulting planned motions lie on $\SO{3}$ and represent those motions which project onto the existing tree on $\S^2$. Simplifications like this are helpful to derive admissible heuristics which have been shown to improve runtime significantly~\citep{Orthey2019ISRR}.

By developing fibration trees, we directly contribute to multi-robot motion planning by integrating hierarchical representations into a single, unified framework. This works as follows. Existing representations are analyzed and cast as specific fibration types. Prioritized planning is represented as sequential fibrations; Decomposition-based planning is represented as parallel fibrations; Task-space projections are represented as partial fibrations, which are projections which might have empty lifts.
This is shown in Fig.~\ref{fig:fibration-trees-as-unification}, where the three categories of fibrations are shown together with algorithms which use those representations to simplify planning problems. In each case, users of the algorithm have to (implicitly) make a decision about which representation to use for a given problem. We make this decision explicit by requiring users to explicitly construct fibration trees, which requires users for example to decide between a prioritization or a decomposition.

To efficiently exploit a given fibration tree, we propose a new algorithm, the Rapidly-Exploring Random Fibration Trees method (\FibrationRRT). \FibrationRRT can solve motion planning problems with an added structure like projection-based planning, decomposition-based planning, or planning with task-space projections. On top of that, it can also deal with any combination of projections, for example when users have domain knowledge and want a custom hierarchical representation.
To implement \FibrationRRT, we propose an extension to the OMPL library, such that multi-robot motion planning problems can be defined both using prioritization-based approaches, decomposition-based planning, using task-space projections, or using a combination of them. This opens the way towards more robust, comparable multi-robot planning algorithms.

To summarize, this paper makes the following contributions:
\begin{enumerate}
    \item We define the framework of fibration trees unifying decomposition-based frameworks, projection-based frameworks, and task-space projections.
    \item We develop a novel planner, \FibrationRRT, which combines elements
      from dRRT \citep{Solovey2016}, and QRRT \citep{Orthey2024IJRR} to exploit fibration trees and we show that \FibrationRRT is probabilistically complete.
    \item We evaluate \FibrationRRT on $32$ scenarios divided into $5$ benchmarks, involving comparisons to classical planners, prioritization-based planners, decomposition-based planners, and task-space planners on different multi-robot planning scenarios.
    \item We implement \FibrationRRT in the open motion planning library (OMPL)
      \citep{Moll2015, Sucan2012} with a front-end based upon the Dynamic
      Animation and Robotics Toolkit (DART)~\citep{Lee2018} to create fibration trees. 
\end{enumerate}

To our knowledge, this is the first time that one single planner (\FibrationRRT) can solve all scenarios by taking the different projection types into account. We thereby demonstrate that fibration trees provide a unified framework for multi-robot planning algorithms.

\begin{figure*}
    \centering
    \includegraphics[width=\textwidth]{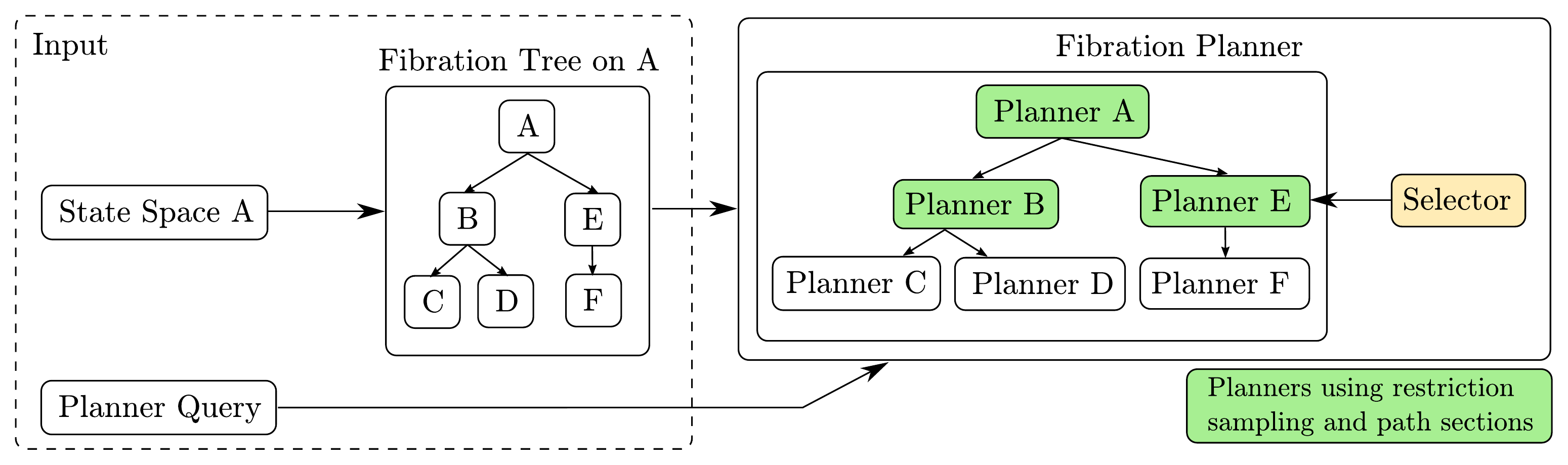}
    \caption{Overview about planning with fibration trees. Given a state space $A$, a fibration tree formalizes abstractions imposed onto $A$. This fibration tree contains, as an example, additional state spaces $B, C, D, E, F$ which are simplifications of the original state space $A$. A number of projections are attached as edges between state spaces, which describe how state spaces are related. Given such a fibration tree and a planner query, a fibration planner is created, which attaches an individual planner to each state space, and uses a selector to choose which planner to run. The non-leaf nodes in the fibration tree have attached planners which can plan efficiently using restriction sampling (see Sec.~\ref{sec:unified-restriction-sampling}) and path sections (see Sec.~\ref{sec:unified-path-section-search}).\label{fig:system}}
\end{figure*}

\section{Related Work\label{sec:related-work}}
{
\newcolumntype{Y}{>{\centering\arraybackslash}m{.001\linewidth}}
\newcolumntype{T}{>{\centering\arraybackslash}m{.6\linewidth}}
\newcolumntype{L}{>{\raggedright\let\newline\\\arraybackslash\hspace{0pt}}m{.63\linewidth}}

\def\customIndent{\hspace{3mm}}
\begin{table}[t]
    \vspace{1em}
\caption{Comparison of planners using projections to solve motion planning problems. For each planner, we note if a certain property exists (green) or does not exists (grey). Sec.~\ref{sec:fibration-axioms} for the definition of projections.}
    \footnotesize
    \center
    \begin{tabularx}{\linewidth}{| L | *{6}X |}
         \hline
         Method & 
         \rotatebox{90}{Admissible} &  
         \rotatebox{90}{Liftable} &  
         \rotatebox{90}{Non-Euclidean} &  
         \rotatebox{90}{Non-trivial} &
         \rotatebox{90}{Partial} &
         \rotatebox{90}{Decompositions} \\
         \hline
        \textbf{Parallel Fibrations} (Decompositions) & \multicolumn{6}{c|}{}\\
        \customIndent  M* ~\citep{Wagner2015}
        & \markYes & \markYes & \markYes & \markNo & \markNo & \markYes\\
        \customIndent Discrete RRT (dRRT*)~\citep{Shome2020}
        & \markYes & \markYes & \markNo & \markNo & \markNo & \markYes\\   
        \customIndent Fast-dRRT*~\citep{Solano2024}
        & \markYes & \markYes & \markNo & \markNo & \markNo & \markYes\\   
        \textbf{Partial Fibrations} & \multicolumn{6}{c|}{}\\
        \customIndent XXL~\citep{Luna2020XXL}
        & \markYes & \markYes & \markYes & \markNo & \markYes & \markNo\\
        \customIndent KPIECE~\citep{Sucan2011}
        & \markYes & \markYes & \markYes & \markNo & \markYes & \markNo\\
        Path-Velocity Decomposition& \multicolumn{6}{c|}{}\\
        \customIndent AVP-RRT~\citep{pham2017admissible}
        & \markYes & \markYes & \markNo & \markNo & \markYes & \markNo\\
        \textbf{Sequential Fibrations}& \multicolumn{6}{c|}{}\\
        \customIndent Hierarchical FMT*~\citep{Reid2020}
        & \markYes & \markYes & \markNo & \markNo & \markNo & \markNo\\
        \customIndent Possibility Graphs~\citep{Grey2017Footstep}
        & \markYes & \markYes & \markNo & \markNo & \markNo & \markNo\\
        \customIndent Quotient-space~\citep{Orthey2018IROS} 
        & \markYes & \markYes & \markYes & \markNo & \markNo & \markNo\\
        \customIndent Fiber Bundles~\citep{Orthey2024IJRR} 
        & \markYes & \markYes & \markYes & \markYes & \markNo & \markNo\\
        \hline
        Fibration Trees (\textbf{this paper}) 
        & \markYes & \markYes & \markYes & \markYes & \markYes & \markYes\\
        \hline
    \end{tabularx}
\end{table}
}

Fibration trees form a framework for sampling-based motion planning, which is
based upon and extends prioritization-based and decomposition-based planning. Moreover, it integrates task-space and workspace
projections to create a unified framework. To put fibration trees into the wider scientific context, we give a
brief overview about sampling-based multi-robot motion planning and review the underlying methods of prioritization-based and decomposition-based planning.
Finally, we discuss how fibration trees are related to learning-based
approaches in multi-robot motion planning.

\subsection{Sampling-based Multi-Robot Motion Planning}
Sampling-based motion planning~\citep{Orthey2023AnnualReview} is an approach to
motion planning~\citep{Kavraki2016Motion, Choset2005PrinciplesOfRobotMotion, lavalle_2006}, whereby motions are found through random
sampling of robot configurations. Early works in the field formulated this
problem in terms of planning through configuration
spaces~\citep{Lozano1979,Lozano1983}, whereby sampling-based methods like
graph-based approaches~\citep{Kavraki1996} and tree-based
approaches~\citep{Kuffner2000, Hsu1999EST} have become the de facto standard way to
conduct planning~\citep{lavalle_2006, Orthey2023AnnualReview}.

Multi-robot motion planning is a generalization of single-robot motion planning,
whereby planning is conducted for multiple disconnected kinematic chains. 
This
can be formulated as planning in a composite configuration
space~\citep{Schwartz1983PianoMoversIII}. However, since the complexity of
motion planning scales
exponentially with the number of
dimensions~\citep{Reif1979,hopcroft_1984,Canny1988complexity}, research on multi-robot motion
planning has focused on a smarter fragmentation of the composite configuration
space, such that planners can operate more efficiently. For some specific cases,
like homogeneous robot teams, this can be accomplished by reducing the problem
to the pebbles-on-a-graph problem~\citep{kornhauser_1984}, which can for example be solved using methods like push and swap~\citep{luna_2011},
linear programming~\citep{yu_2016}, subgraph partitioning~\citep{ryan_2010}, or
conflict-based search (CBS)~\citep{sharon2015conflict, hoenig_2018,
moldagalieva2024db}.
Our work explicitly tackles the heterogeneous case, where robot teams can have different compositions. 

For heterogeneous robot teams, most
planning approaches can be divided into either prioritized or decomposition-based
planning. Prioritization-based
approaches~\citep{Schwartz1983PianoMoversIII,erdmann_1987} plan first
for a single robot, then use the resulting path as a constraint for the next robot, until we
accounted for all robots~\citep{Orthey2024IJRR}. This creates a sequential planning problem where the
ordering of the robots has to be prespecified~\citep{VanDenBerg2005Prioritized} or optimized~\citep{Bennewitz2001PrioritizedPathPlanning}.
Decomposition-based
(or decoupled)
planning~\citep{svestka1995coordinated,van2009centralized,Solovey2016} differs by
planning in parallel for each robot, and then leveraging the results to compute
a solution in
the composite configuration space. 
We review both approaches and discuss their individual trade-offs.

\subsection{Prioritization-based and Projection-based planning}

Instead of planning in the composite state space, it is often advantageous to
plan motions sequentially, where the robots are ordered and planning is
prioritized~\citep{erdmann_1987,Tournassoud1986MRS,VanDenBerg2005Prioritized}.
Prioritized motion planning is particularly effective if robots are near-independent and do not block each other~\citep{Orthey2024IJRR}.

It is often crucial to choose the right ordering of robots to avoid any deadlocks~\citep{heselden2023heuristics}. While most methods utilize a fixed ordering~\citep{VanDenBerg2005Prioritized}, one can also search over the space of all robot orderings to find an order which leads to a collision-free path~ \citep{ma2019searching}. Methods such as hill-climbing can be utilized~\citep{Bennewitz2001PrioritizedPathPlanning}, which swap two robots in the order on each iteration step.

Once a prioritization (an ordering of robots) has been chosen, there are multiple (heuristic) ways to efficiently
compute solutions and resolve conflicts.
A frequently used approach is based on the analysis of the robot workspace. By using the workspace of the robot, one can for instance compute an environment skeleton~\citep{McBeth2023MRMPTopologicalGuidance}. This skeleton can then be used to quickly reroute robots if conflicts arise~\citep{zhang5219981hierarchical}. Another method is based upon homotopy classes~\citep{bhattacharya_2018}. Here, robots compute path prospects (paths in different homotopy classes) to avoid and resolve conflicts on the fly~\citep{wu2020multi}. Another strategy is to use distributed planning, where robots plan in parallel 
\citep{velagapudi2010decentralized} and conflicts are resolved according to priority. Some research has also focused on using approximations of the area swept by a lower-priority robot to improve planning time for subsequent higher-priority robots~\citep{keppler2020prioritized}. 

Often it is even possible to generate sufficient conditions
on the solvability of prioritized motions~\citep{vcap2015prioritized}. This involves low-priority robots avoiding the start configurations of higher-priority robots, while higher-priority robots avoiding the goal configuration of the lower-priority robots. If such paths are found, higher-priority robots could just wait until lower priority robots have reached their goal. This would be a sufficient condition for a solution.
Our work is complementary by providing a unified framework in which those heuristics can be implemented and compared.

\subsubsection{Projection-Based Planning}
Prioritization-based approaches for multiple robots are also closely related to projection-based planning for single robots.
While prioritization-based approaches simplify problems by removing robots, projection-based approaches simplify problems by removing degrees of freedom~\citep{Orthey2018IROS,Orthey2019ISRR}. This is known in the literature under different names like multi-layered planning~\citep{Vidal2019}, quotient-space
planning~\citep{Orthey2018IROS,Orthey2019ISRR}, or
subspace projections~\citep{Reid2019, Reid2020}. 

A widely used method to solve projection-based problems more efficiently is the path section method (also called root- or lead-path
method)~\citep{wermelinger2016navigation, Tonneau2018, Vidal2019,
Orthey2021TRO}. This method works by first computing a solution on a simplified
space (a base path). Using this base path, the search space is restricted to
allow only solutions "above" the base path, i.e. solutions which lie in the path
restriction (see our terminology in Sec.~\ref{sec:restrictions-and-sections}). 
This can speed up planning problems significantly~\citep{Orthey2019ISRR}. Another popular
method leverages probabilistic roadmaps (PRM)~\citep{Kavraki1996} by using it as
a method to solve individual spaces~\citep{Orthey2018IROS}.
This is similar to subspace projections, a method which has been used to extend
the fast marching trees (FMT)~\citep{Janson2015FMT} planner to create the hierarchical fast
marching trees (HFMT*) planner~\citep{Reid2019, Reid2020}. HFMT* can optimally
solve sequential fibrations in euclidean spaces. Our method is similar in that we can integrate projections as input while implementing a unified path section method (see Sec.~\ref{sec:unified-path-section-search}). 

Fibration trees are also closely related to the fiber bundle
approach~\citep{Orthey2024IJRR}. 
This approach used quotient-spaces to create
sequential projections. 
Such projections can be efficiently exploited using
dedicated planners like QMP~\citep{Orthey2018IROS}, QRRT~\citep{Orthey2019ISRR}, QMP*, and
QRRT*~\citep{Orthey2024IJRR}. 
Our approach generalizes this approach to
provide a unified perspective including not only sequential, but also parallel and partial fibrations~\citep{Sucan2011}. 

\subsection{Decomposition-based planning}

An efficient method to solve multi-robot motion planning problems are
decomposition-based planners~\citep{simeon_2002}. Instead of planning in the
combined state space of all robots, decomposition-based planners solve first the
state spaces of each individual robot (while ignoring the other robots), and then lift those solutions into the
combined state space (while checking collisions in-between robots). This lifting
action can either use path coordination, where only the solution paths are
considered~\citep{simeon_2002}, or use graph coordination, where all graphs on
the individual spaces are considered~\citep{svestka_1998}.

Since graph coordination leads to probabilistically complete planners~\citep{svestka_1998}, it has
been the main approach for decomposition-based planning.
Several
planners have been proposed for graph coordination, such as
M*~\citep{Wagner2015}, which extends the A* planner, by using the individual
solutions as admissible heuristics to guide the search  
A similar idea is used in the discrete RRT (dRRT)~\citep{Solovey2016} and dRRT*
planner~\citep{Dobson2017scalable, Shome2020}, which improves upon M* by using a more
efficient neighbor evaluation function (oracle) for Euclidean spaces.

The dRRT planner has been particularly successful and has been extended by
multiple works. For example, one can combine it with a conflict-based search
where a sampling-based motion planner is spawned in a subspace whenever a conflict between two robots arises~\citep{solis2024adaptive}. 
For lower dimensional spaces, it can also be helpful to use lattice-based roadmaps~\citep{Parque2023LatticeGuidedDiscreteRRT} or simultaneous sampling~\citep{Okumura2023CombiningSamplingAndSearch} on the individual spaces.

Recently, planners have been developed which pursue a hybrid approach. Scheduling to Avoid Collisions (StAC)~\citep{guo2026stac} is such a planner which uses a hybrid strategy between coupled (prioritization-based) and decoupled (decomposition-based) approaches. They plan motions for each individual robot, then do path coordination with random pauses to find a valid solution for all robots. If this fails, the individual planners are updated with this information and planning continues until a valid solution is found.  

Our work is complementary to approaches like dRRT or StAC by providing a unified framework
in which decomposition-based or hybrid planning can be implemented as a special case.
Since our framework can handle
any fibration, it allows us to compare decomposition-based
planners~\citep{Solovey2016, Shome2020} with prioritization-based
approaches---to understand which approach is more appropriate for a given task.

\subsection{Learning-Based Approaches to Multi-Robot Planning}

While this work focuses on a unified framework for sampling-based multi-robot motion planning, recent research has focused on learning-based multi-robot motion generation.  
Promising methods are diffusion models~\citep{shaoul2025multirobot,liang2025simultaneous,jiang2023motiondiffuser}, which are networks which learn trajectories by gradually adding noise to the input data. Other notable learning-based directions include hierarchical reinforcement learning approaches~\citep{bettini2023heterogeneous,bettini2024benchmarl,sebastian2025physics,wu2024state}, imitation learning~\citep{sun2025realtime}, and dedicated multi-robot reinforcement-learning environments~\citep{hu2023neuronsmae}.
Closely related to our work are also quotient-based Markov decision processes (Quotient-MDP)~\citep{Welde2025symmetry}, which are ways to train a policy on a lower-dimensional quotient-MDP which is then lifted using a tracking controller to the original space. This is particularly effective for homogeneous robot teams, since robots can
share neural network weights between themselves.

While learning-based approaches have shown significant improvements, they require large datasets to train a policy. Our approach is complementary in that we provide tools to provide abstractions and methods to efficiently solve those problems. Our methods can then be leveraged to generate data points to learn and extend robot policies~\citep{ha2020multiarm, lai2025roboballet}.

\section{Multi-Robot Motion Planning}

Let $X = Y_1 \times \cdots \times Y_M$ be the composite state space of $M$
robots with individual state spaces $Y_m$. Each state space $X$ is defined by the
set of configurations a robot can attain plus the following structures:
\begin{itemize}
  \item Distance metric $d: X \times X \rightarrow \R_{\geq 0}$ which measure
    the distance between two states.
    \item Steering function $f$ which extends path segments from a state $x \in
      X$ towards another state $y \in X$. 
    \item Boolean constraint function $\phi: X \rightarrow \{0,1\}$ which
      evaluates to $0$ if a state $x$ is constraint free and $1$ otherwise.
    \item Sampling sequence $s = \{x_1, x_2, \ldots\}$ which produces a uniform, infinite set of (dense) samples on $X$.
\end{itemize} 

The constraint function $\phi$ decomposes the state space into two regions,
whereby the states for which $\phi$ evaluates to true are called \emph{feasible
states} and $X_f = \{ x \in X\mid \phi(x) = 0\}$ denotes the free state space.

Let $(x_I, X_G)$ be a planner query consisting of an initial state $x_I \in
X_f$ and a goal region $X_G \subseteq X_f$. A motion planning problem is then defined
as the tuple $(X, x_I, X_G)$ which asks us to find a path (a continuous mapping
$p: [0,1] \rightarrow X_f$) from the start state to the goal region. 

\section{Fibration Trees\label{sec:fibration-trees}}
In this section, we introduce the concept of fibration trees. A fibration tree is a user-defined structure, which can be imposed on the state space and which serves as a hierarchical representation of it.
Fibration trees can be exploited by a dedicated planner like \FibrationRRT (see Sec.~\ref{sec:fibrationrrt}) to efficiently search the state space and lower runtime significantly. Examples of fibration trees include decompositions of the state space, prioritizations of robots in multi-robot planning, or task-space projections.

Let us define fibration trees more precisely. Let $X$ be a state space. 
A fibration tree on $X$ is a directed tree $\tree_X = (V, E)$ consisting of
nodes $V$ and edges $E$. Nodes represent state spaces together with an associated node planner, internal datastructure (tree, graph), and initial state and goal region. The root node
contains the original state space $X$. Edges represent fibrations, which are
tuples $(A, B, \pi_{AB})$ (written as $\pi_{AB}: A \rightarrow B$ or just $A
\rightarrow B$), consisting of a state space $A$ (the total space of a
fibration), a state space $B$ (the base space of a fibration), and a projection $\pi_{AB}$ which maps elements of $A$ to elements of $B$. Those concepts are summarized in Table.~\ref{tab:notations}. The fibration mapping have to fulfill two axioms, which we detail next.

\subsection{Projection axioms\label{sec:fibration-axioms}}

Let $\pi_{AB}: A \rightarrow B$ be a fibration. To be useful for motion planning, a fibration has to be liftable and admissible. 

\subsubsection{Liftable}

A fibration is liftable if its projection has an inverse. 
Given the projection $\pi_{AB}$ mapping elements of $a \in A$ to $b \in B$, we
say that $\pi_{AB}$ is liftable, if $\pi_{AB}$ has an inverse $\pi_{AB}^{-1}$
mapping an element $b \in B$ to an element $a \in A$. 

\textbf{Raison d'être:} Motion planning with non-liftable projections would prevent us from transfering information upwards in a fibration tree. 

\newcolumntype{Z}{>{\centering\arraybackslash}m{.25\linewidth}}
\begin{table}
    \centering
    \begin{tabularx}{\linewidth}{|Z|X|}
        \hline
       Name  & Description \\
       \hline
        Node & A structure including a state space, a local planner, a datastructure (tree), plus a local planning problem including an initial state and a goal region.\\
        Edge & Representing a Fibration between nodes.\\
        Node Space & The corresponding space associated to a node.\\
        Node Planner & The corresponding planner associated to a node.\\
        Root Node & The root of a fibration tree\\
        Fibration $\pi: A \rightarrow B$ & A mapping between a node $A$ (total space) and a node $B$ (base space)\\
        Total Space & The domain of a fibration\\
        Base Space & The co-domain of a fibration\\
        \hline        
    \end{tabularx}
    \caption{Notations used in this paper.}
    \label{tab:notations}
\end{table}

\subsubsection{Admissible}

A projection is called admissible, if it preserves validity. A projection $A \rightarrow B$ is admissible, if, for any $a \in A$, $\phi(a) \Rightarrow \phi(\pi_{AB}(a))$~\citep{Orthey2019ISRR}. This means that every feasible element $a \in A$ will be projected onto a feasible element $b \in B$. 

\textbf{Raison d'être:} Motion planning with non-admissible projections would prevent us from keeping theoretical guarantees like probabilistic completeness.

\begin{figure*}
    \centering
    \begin{subfigure}[t]{0.33\textwidth}
        \vskip 0pt
        \centering
        \includegraphics[align=c,width=\textwidth, height=0.9\textwidth]{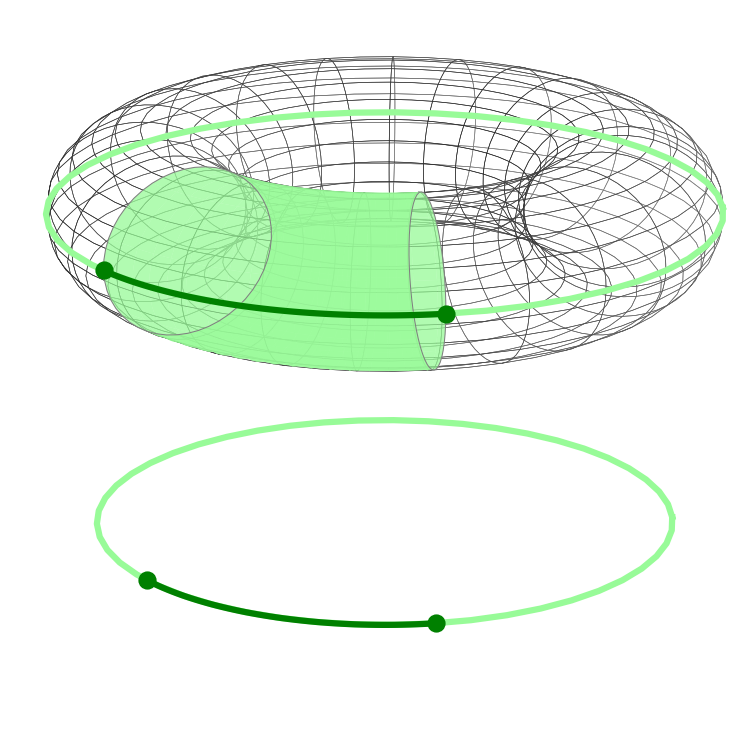}
        \caption{Sequential Fibration}
    \end{subfigure}
    \begin{subfigure}[t]{0.33\textwidth}
        \vskip 0pt
        \centering
        \includegraphics[align=c,width=\textwidth, height=0.9\textwidth]
        {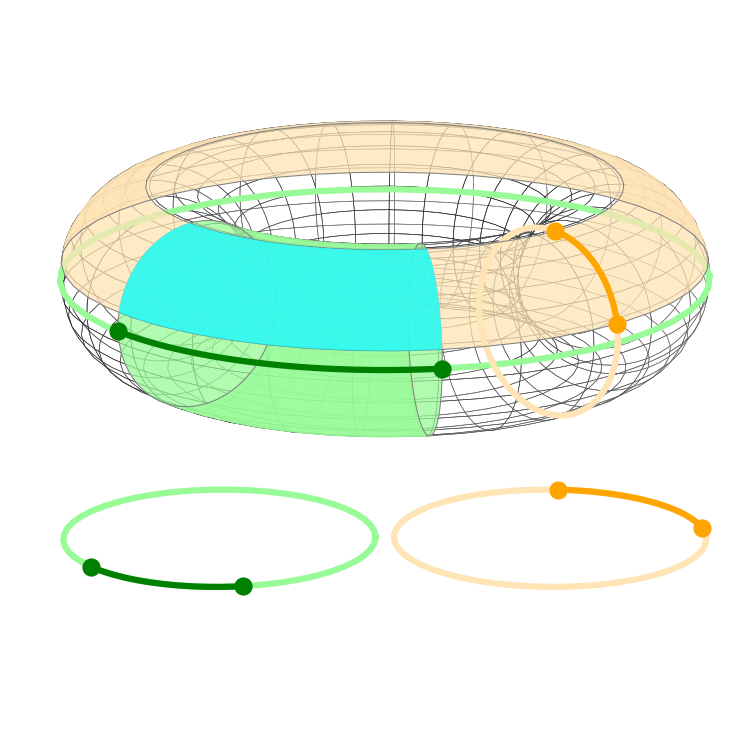}
        \caption{Parallel Fibration}
    \end{subfigure}
    \begin{subfigure}[t]{0.33\textwidth}
        \vskip 0pt
        \centering
        \includegraphics[align=c,width=\linewidth, height=0.6\linewidth]{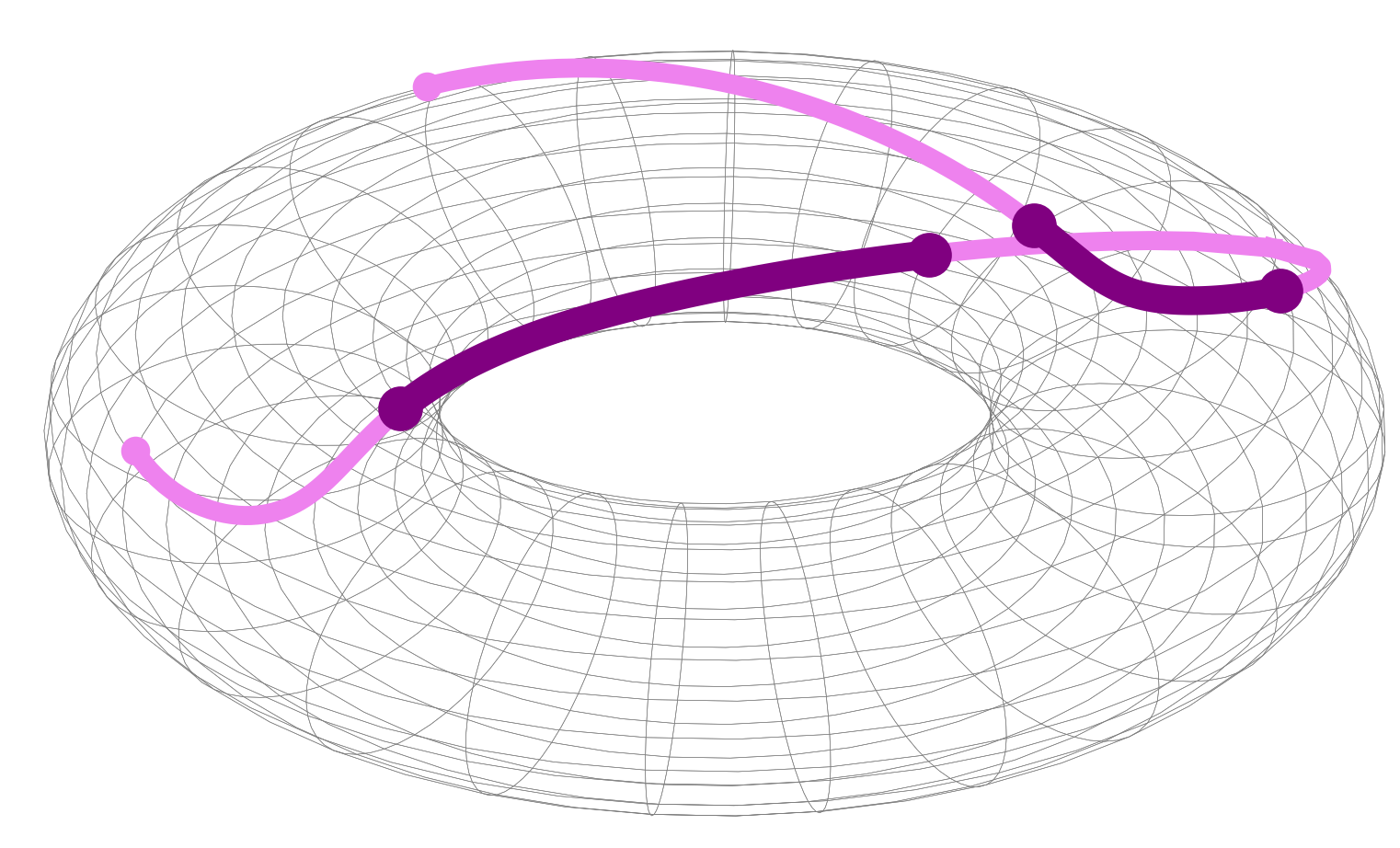}
        \includegraphics[align=c,width=\linewidth, height=0.3\linewidth]{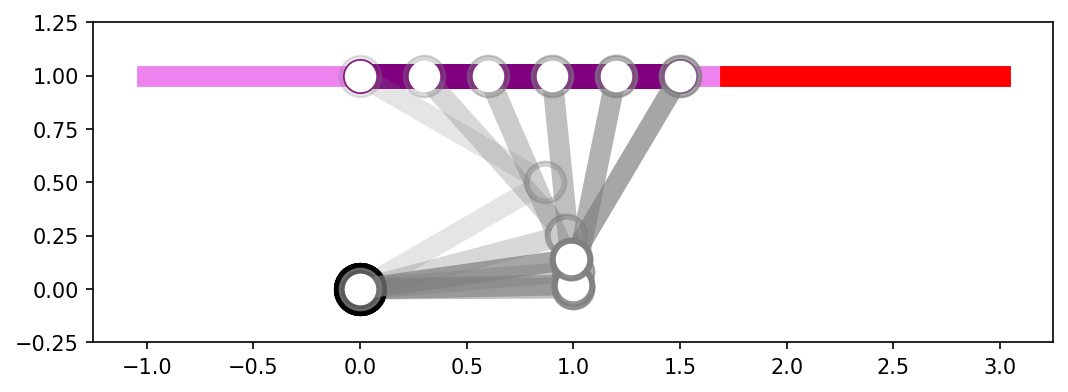}
        \caption{Partial Fibration}
    \end{subfigure}

    \caption{Examples of the different fibration types we consider in this paper as visualized on the torus $T^{2}$. \textbf{Left:} Sequential fibration with a projection $T^2 \rightarrow S^1$. The circle (green) has a path segment (darkgreen), which is lifted to a path restriction (green, on torus). \textbf{Middle:} A parallel fibration of $T^{2}$ to two circle spaces, one for the major axis (green), one for the minor axis (orange) of the torus. Two path segments are shown and their corresponding path restrictions on the torus. The intersection of both path restrictions (cyan) is shown on the torus. \textbf{Right:} A partial fibration $T^{2} \rightarrow R^1$ of the torus to a line segment in the workspace of a 2-dof robot. Note that there are segments not liftable (red) in the workspace. The path on the torus (light purple) is a two-on-one mapping, since every point on the depicted workspace path (dark purple) has two IK solutions (corresponding to the dark purple segments on the torus).}
    \label{fig:fibration_types}
\end{figure*}

\subsection{Fibration Types\label{sec:fibration-types}}

While many fibrations are liftable and admissible, there are three types which
are particularly interesting for motion planning applications. Those are
sequential, parallel, and partial fibrations. 
Those fibrations types are visualized in Fig.~\ref{fig:fibration_types}, and further detailed below.

\subsubsection{Sequential Fibrations}

Sequential fibrations are fiber bundles~\citep{steenrod_1951, lee_2003}, which
are tuples $(A, B, \pi_{AB}, F)$ consisting of a fibration with an additional fiber $F$. A fiber is the preimage $\pi^{-1}_{AB}(b)$, i.e. the subspace of all configurations on $A$ which project onto the same point $b$ on the base space $B$.
The defining property of a fiber bundle is that $A$ is locally a product space $B \times F$. 
This means that the preimage of $\pi_{AB}$ is homotopically equivalent to the fiber, i.e. for any $b \in B$, the preimage $\pi^{-1}_{AB}(b)$ is a copy of $F$. 
For fiber bundles, the total space $A$ can be thought of as a union of identical copies of $F$ glued together by the properties of $\pi$. For details, see~\citep{Orthey2024IJRR}.

\subsubsection{Parallel Fibrations}

Parallel fibrations are product spaces~\citep{solovey_2016}, such that $A$ factors as $A \rightarrow A_1 \times \cdots \times A_M$, whereby $A = A_1 \times \cdots \times A_M$ is the product space, together with $M$ projections  $\pi_{A_1}, \ldots, \pi_{A_M}$ onto the individual factors as $\pi_{A_m} : A \rightarrow A_m$. Parallel fibrations model for example multi-robot decomposition-based planning problems, where $A$ is the compound space, and $A_1,\cdots,A_M$ are the individual robot subspaces.

\subsubsection{Partial Fibrations}

A partial fibration $(A, B, \pi_{AB})$ consists of a projection which can only lift a subset of elements of $B$, i.e. some points $b \in B$ lift to an empty set.
As an example, consider the fibration $\R^6 \rightarrow T$ of a $6$-dof manipulator arm with a projection onto a Tool-center point (Tcp) region $T \subseteq \R^3$ of the robot. For each point $b$ on the base space (Tcp region), the states projecting onto $b$ are given by the set $F(b)=\{x \in \X\mid \pi^{-1}(b)=x\}$ (the fiber of $b$). 
For the manipulator, this involves solving an inverse kinematics (IK) problem. 
However, the fiber depends on the IK solver, and can be either an empty space (no IK solution), contain a single point, contain multiple points, or even contain a subspace of IK solutions. 
This means that there is only a subspace $B' \subset B$ on which the fibration returns a solution.

\subsection{Restrictions and Sections on Fibration Trees\label{sec:restrictions-and-sections}}

Fibrations provide the useful concepts of restrictions and sections, which we summarize in the next two sections. 

Those two concepts are used in two primitives of the \FibrationRRT planner, namely the unified restriction sampling (Sec.~\ref{sec:unified-restriction-sampling}) and the unified path section search (Sec.~\ref{sec:unified-path-section-search}).

\subsubsection{Restrictions}

Given a fibration $A \rightarrow B$, and a set $U \subseteq B$, a restriction is defined as the set $\pi^{-1}_{AB}(U)$ on $A$. In this work, we make use of three main types of restrictions. 

\begin{itemize}
    \item \textbf{Fiber}: Given a point $b \in B$, we call $\pi^{-1}(b)$ the \emph{fiber} of $b$ in $A$. The fiber contains all points which, when projected onto $B$, are exactly equivalent to the point $b$.
    \item \textbf{Path restriction}: Given a path $p: [0,1] \rightarrow B$, we call the inverse of the image\footnote{Note that we make an explicit distinction between a \emph{path}, a mapping from $[0,1]$ to a space $B$, and a \emph{path image}, the image of the path, i.e. the set $p([0,1])$.} of $p$ the path restriction on $A$. It contains all points which, when projected onto $B$, are located on the path image of $p$. 
    \item \textbf{Graph restriction}: Let $G = (V, E)$ be a graph on $B$ with vertices being points, and edges being path segments. The union of all fibers and path restrictions associated to the vertices and edges is called the \emph{graph restriction}. It contains all points which, when projected onto $B$, are located on an edge or vertex of $G$. 
\end{itemize}

\subsubsection{Sections}

Let $A \rightarrow B$ be a fibration, and $R = \pi^{-1}_{AB}(U)$ be a restriction on $A$ with associated set $U$ on $B$. We define a section as a mapping $s: U \rightarrow R$ from $U \subseteq B$ to the restriction $R$. We make use of the following two section concepts.

\begin{itemize}
    \item \textbf{Lift}: Given a fiber $R$ over a base point $b$, we define a lift as the mapping $s: {b} \rightarrow R$, which selects a point in the fiber associated to $b$. 
    \item \textbf{Path section}: Given a path restriction $R$ over a base path $p$, we define a path section as $s: [0,1] \rightarrow X$ such that $\pi(s(x)) = x$ for any $b$ in $p[I]$.
\end{itemize}

This concludes our description of fibration trees, their notations, and their internal structures. Next, we describe a dedicated planner which can efficiently exploit fibration trees to improve runtime for high-dimensional multi-robot motion planning problems.
\section{Rapidly-Exploring Random Fibration Trees\label{sec:fibrationrrt}}

\def\Xbest{A}
\def\Xcur{\X_{k}}
\def\Xparent{\X_{p}}

\begin{algorithm}[t]
\algcaption{FibrationRRT}
\begin{algorithmic}[1]
  \Require Planning query $(X, \xi, \xg)$, Fibration tree $T_X$
  \Ensure Solution path on $X$
  \State $\PriorityQueue = \{\Call{GetLeafNodes}{T_X}\}$
\While{$\neg\Call{ptc}{\PriorityQueue, T_X}$}
    \State $\Xbest \gets \Call{SelectNode}{\PriorityQueue}$\Comment{Sec.~\ref{sec:planner-selection}}
  \State $\Xbest \gets \Call{PlanNode}{\Xbest}$\Comment{Sec.~\ref{sec:planning-on-a-node}}
  \State $\PriorityQueue \gets \Call{UpdateActiveNodes}{T_X, \Xbest, \PriorityQueue}$
\EndWhile
  \State \Return $G_X$
\end{algorithmic}
\label{alg:fibrationrrt}
\end{algorithm}

\begin{algorithm}[t]
\algcaption{UpdateActiveNodes}
\begin{algorithmic}[1]
  \Require Fibration Tree $T_X$, $\Xbest$, $\PriorityQueue$
  \Ensure Updated active nodes $\PriorityQueue$
  \If{$\neg\Call{HasParent}{\Xbest}$}
    \State \Return
  \EndIf
  \If{$\neg\Call{SolutionExists}{\Xbest}$}
    \State \Return
  \EndIf
  \State $\Xparent \gets \Call{Parent}{\Xbest, T_X}$
  \If{$\Xparent \not\in \PriorityQueue$}
  \State $\PriorityQueue.\Call{insert}{\Xparent}$
  \EndIf
  \State \Return $\PriorityQueue$
\end{algorithmic}
\label{alg:updateactivenodes}
\end{algorithm}

\begin{algorithm}[t]
\algcaption{SelectNode}
\begin{algorithmic}[1]
  \Require $\PriorityQueue$
  \Ensure Node from $\PriorityQueue$
  \State $\Phi \gets \emptyset$\Comment{Probability Density Function (PDF)}
  \For{$X \in \PriorityQueue$}
    \If{\Call{IsSolved}{X}}
        \If{\Call{HasNonSolvedSiblings}{X}}
            \State \textbf{continue}
        \EndIf
    \EndIf
    \State $n \gets \Call{GetDimension}{X}$
    \State $V \gets \Call{ValidNodes}{X}$
    \State $p \gets \dfrac{1}{\|V\|^{1/n}+1}$
    \State $\Phi.\Call{Insert}{\{X, p\}}$
  \EndFor
  \State \Return $\Call{SampleFromPDF}{\Phi}$
\end{algorithmic}
\label{alg:select-node}
\end{algorithm}

In this section, we describe the Rapidly-Exploring
Random Fibration Trees (\FibrationRRT) algorithm. \FibrationRRT generalizes RRT~\citep{Kuffner2000} by taking as input an additional, user-defined fibration tree as defined in Sec.~\ref{sec:fibration-trees}. Given a motion planning problem and a fibration tree, \FibrationRRT can efficiently exploit the tree to quickly solve even high-dimensional, multi-robot planning problems. We further show that \FibrationRRT is probabilistically
complete given an arbitrary fibration tree (see Appendix~\ref{appendix:proof-completeness}) and we detail how goals are set, which metric we use, and how the algorithm is implemented.

To start, we describe the high-level algorithm in Alg.~\ref{alg:fibrationrrt}. As input we use a planning query consisting of a state space $X$,  start configuration $x_I$ and a goal $X_G$ plus a fibration tree structure $T_X$. 
Our first task is to extract the leaf nodes of $T_X$ (Line 1) and add them to $\PriorityQueue$, the set of active nodes. We then iterate while the planner terminate condition (PTC) is false (Line 2). 
Inside one iteration, we run a three-step process. In step one (Line 3), we select a node from the set of active nodes. This is detailed in Sec.~\ref{sec:planner-selection}. 
In step two (Line 4), we run the planner associated to the selected node
(detailed in Sec.~\ref{sec:planning-on-a-node}).
Finally, in step three (Line 5), we update the active nodes.

The update of the active nodes is further detailed in Alg.~\ref{alg:updateactivenodes}. 
This method adds the parents of an input node A to the active nodes 
when three conditions are true: Node A needs to have a parent (Line 1--3), a solution
has to exist on node A (Line 4--6), and the parent of node A is not yet member of the set
of active nodes (Line 8). 
If those conditions are met, we add the parent to the active nodes (Line 9).

\subsection{Select Node\label{sec:planner-selection}}

Given all nodes in the queue, a selection method is used to decide which node
(and the planner therein) to execute next. 
This selection is based on an exponential importance criterion~\citep{Orthey2024IJRR}, which is defined as
\begin{equation}
    i(X_k) = \dfrac{1}{\|V_k\|^{1/n_k} + 1}\label{eq:exponential-importance-criterion}
\end{equation}
whereby $X_k$ is the $k$-th active node, $V_k$ are the number of valid vertices
in the tree of the planner on $X_k$, and $n_k$ is the dimensionality of $X_k$.
This criterion is motivated by the sampling density---the number of samples per
unit distance---on an $n$-dimensional space which scales with $V^{1/n}$, whereby
$V$ are the number of samples~\citep{Hastie2009}. Intuitively, Eq.~\eqref{eq:exponential-importance-criterion} gives a low number for a high sampling density (i.e. we sampled a node already sufficiently well) and it gives a high number for a low sampling density (i.e. we have not yet sampled a node enough).

The algorithm to select a node is depicted in Alg.~\ref{alg:select-node}. It works by building a probability distribution (Line 1) over nodes which are then sampled. We start by iterating over all active nodes (Line 2) and check if the node should be considered for selection (Line 3-7). This is a special case for parallel fibrations, i.e. whenever we have a solution for one node in a parallel fibration, we first need to solve all sibling nodes---otherwise we cannot move on to the parent node. 

After we verified that the node is selectable, we compute the exponential importance criterion from Eq.~\eqref{eq:exponential-importance-criterion} (Line 8-10) and add it to the probability distribution function (Line 11). Eventually, we sample from the probability distribution to select a node---biased towards the lower sampling density (Line 13).

It is important to note that the planner selection will never stop exploring
active nodes, and it ensures that every node is selected infinitely many times
in the limit. We say that the selection method has the \emph{Uniform Infinite
Recurrence} property. This is a requirement to prove probabilistic completeness,
as we will do in Appendix~\ref{appendix:proof-completeness}.

To obtain a classical prioritized multi-robot planner (sequential fibration)~\citep{VanDenBerg2005}, we can modify the selection by using a \emph{last-node} selection where only the last unsolved node is selected. While this has been implemented in our code, we have not used this in our experiments, since it would violate the Uniform Infinite Recurrence property and thereby remove the probabilistic completeness guarantee of the planner.

\subsection{Planning on a Node\label{sec:planning-on-a-node}}
\def\xrand{a_\text{rand}}
\def\xnear{a_\text{nearest}}
\def\xnew{a_\text{new}}

\begin{algorithm}[t]
\algcaption{PlanNode}
\begin{algorithmic}[1]
  \Require Node $A$ including initial configuration $\xi$, goal region $\xg$,
  fibration $\pi: A \rightarrow B$, and data structure $G$
  \Ensure Updated node $A$
  \If{$\Call{IsEmpty}{G}$}
  \State $G \gets \Call{PathSectionSearch}{\pi, \xi, \xg}$\Comment{Sec.~\ref{sec:unified-path-section-search}}
  \EndIf
  \While{$\neg\Call{ptc}{A}$}
    \State $\xrand \gets \Call{RestrictionSampling}{\pi}$\Comment{Sec.~\ref{sec:unified-restriction-sampling}}
    \State $\xnear \gets \Call{Nearest}{\xrand, G}$
    \State $\xnew \gets \Call{Steer}{\xnear, \xrand, A}$
    \If{$\Call{IsValid}{\xnear, \xnew}$}
    \State $G \gets G \cup \{\xnear, \xnew\}$ 
      \If{$\xnew \in \xg$}
        \State $A.\text{solved} \gets \text{True}$
      \EndIf
    \EndIf
\EndWhile
\State \Return $A$
\end{algorithmic}
\label{alg:plannode}
\end{algorithm}

The core of the \FibrationRRT planner is the \texttt{PlanNode} method. While any motion planning algorithm could be used in this function, we have chosen the RRT~\citep{Kuffner2000} method due to its versatility. 

The adjusted algorithm is shown in Alg.~\ref{alg:plannode}. 
This method is used
to find a solution on a given total space represented by a node $A$. 
From the node, we use
the (projected) initial configuration $x_I$ and goal region $X_G$,
fibration
$\pi: A \rightarrow B$, and
datastructure $G$. 
In our case this includes the tree of the RRT and
additional information about the number of runs.

Planning on a node differs from RRT in two important aspects. First, if there is
no data yet in $G$ (Line 1), we run a path section method (Line 2), where we
compute a path directly in the path restriction of the solution on $B$ (if
possible). This is
further detailed in Sec.~\ref{sec:unified-path-section-search}. Second, the
sampling function is replaced by restriction sampling (Line 5), which takes the
datastructure on the base space $B$ into account. This is detailed in Sec.~\ref{sec:unified-restriction-sampling}.
The remainder of the algorithm
follows the standard RRT method~\citep{Kuffner2000} by executing a loop while the
planner termination condition (PTC) is not fulfilled (Line 4). Inside the loop,
we sample and compute the nearest configuration in the datastructure $G$ (Line
6), steer towards it (Line 7), and add the resulting edge to the datastructure
(Line 8--9). If the new configuration has reached the goal region (Line 10), we
mark the node as solved (Line 11). Finally, we return the updated node (Line
15).

Note that the \texttt{PlanNode} is an extension of the standard
RRT~\citep{Kuffner2000} planner. This can be seen when we look at the special
case where $\pi$ is an empty projection, i.e. there is no base space. In this
case, \FibrationRRT performs exactly equivalent to RRT, whereby path section
performs no operation and restriction
sampling reverts back to uniform sampling. This is similar to
previous works, where RRT was extended to the quotient-space
RRT~\citep{Orthey2019ISRR}. However, this only worked for the special case
of quotient spaces as sequential fibrations, while our method 
represents a general approach which can also handle parallel and partial
fibrations.

\subsection{Unified Restriction Sampling\label{sec:unified-restriction-sampling}}
\begin{algorithm}[t]
\algcaption{RestrictionSampling}
\begin{algorithmic}[1]
\Require Fibration $\pi: A \rightarrow B$
  \Params Path sampling bias (Case 2), path restriction surrounding bias (Case
  2), sampling
  perturbation bias (Case 2)
\Ensure A configuration $a \in A$
\LineComment{\fbox{\textbf{Case 1}: $A$ is a leaf node or $\pi$ is not a fibration.}}
\If{$\neg$\Call{Exists}{$B$}}
    \State \Return \Call{Sample}{A}
\EndIf
\LineComment{\fbox{\textbf{Case 2}: $\pi$ is a sequential- or partial fibration.}}
\If{\Call{HasFiber}{$\pi$}}
    \State $\xb \gets \Call{SampleDatastructure}{B}$ 
    \State $x \gets \Call{SampleFiber}{\xb, \pi}$
    \State \Return $x$
\EndIf
\LineComment{\fbox{\textbf{Case 3}: $\pi$ is a parallel fibration.}}
\State $x \gets \Call{IdentityElement}{A}$
\For{$B_m \in B_{1:M}$}
    \State $\xb \gets \Call{SampleDatastructure}{B_m}$
    \State $x \gets \Call{InsertionMap}{\xb, \pi, B_m}$
\EndFor
\State \Return $x$
\end{algorithmic}
\label{alg:restrictionsampling}
\end{algorithm}

Given a fibration $\pi : A \rightarrow B$, one of the most important planning primitives is restriction sampling. 
Restriction sampling (Alg.~\ref{alg:restrictionsampling}) returns a
configuration which
lies in the restriction of the current datastructure on the base space $B$. 

Depending on the type of projection, we have to distinguish three cases. 
In the first case (Line 1--3), $\pi$ is an empty projection (i.e. the space $A$ belongs to a leaf node of the fibration tree). In this case, we uniformly sample the space $A$ and return.

In the second case (Line 4--8), if the projection has a fiber, there is an
explicit way to sample the fiber. We therefore first sample a configuration
$\xbase$ on the base space $B$ (Line 5) and then compute a fiber element to
obtain a configuration $x$ on the total space $A$. Note that partial fibrations
can have no solutions, in which case a default invalid element is returned. Note
that this case has three associated parameters. First, the path sampling bias.
This is a percentage of samples drawn from the solution path on $B$ instead of
sampling from the tree. Second, the path restriction surrounding bias, which
provides a small margin around the solution path on $B$ in which the path
sampling takes place. Finally, the sampling perturbation bias, which
additionally perturbates a sampled configuration to sample effectively in a
neighborhood of the graph restriction~\citep{Orthey2021ICRA}. See Sec.~\ref{sec:experimental-setup} for implementation values.

In the last case (Line 9--14), $\pi$ is a parallel projection. In this case we iterate over all children spaces, sample a base space element (Line 11) and apply the insertion map to map the element into the total space (Line 12). Once all states are mapped, the total space element (Line 14) is returned.

\subsection{Unified Path Section Search\label{sec:unified-path-section-search}}
\begin{algorithm}[t]
\algcaption{PathSectionSearch}
\begin{algorithmic}[1]
\Require Fibration $\pi: A \rightarrow B$, initial configuration $\xi$
  \Params Maximum branching (Case 2), maximum depth (Case 2), Restriction accuracy (Case 3), propagation attempts (Case 3), maximum
  number of permutations (Case 4)
\Ensure A (solution) path on the path restriction or a failure
\LineComment{\fbox{\textbf{Case 1}: $A$ is a leaf node or has no solution.}}
\If{$\neg$\Call{Exists}{$B$}\OR $\neg$\Call{HasSolution}{$B$}}
    \State \Return
\EndIf
\State $\pb \gets \Call{GetSolutionPath}{B}$ 
\LineComment{\fbox{\textbf{Case 2}: $\pi$ is a sequential fibration}}
\If{\Call{IsSequential}{$\pi$}}
\State $F \gets \Call{GetFiber}{\pi}$
\State $p \gets \Call{InterpolateSection}{\pb, \pi, \xi}$
\State $p \gets \Call{RecursiveSideStep}{\pb, \pi}$
\State \Return $p$
\EndIf

\LineComment{\fbox{\textbf{Case 3}: $\pi$ is a partial fibration}}
\If{\Call{IsPartial}{$\pi$}}
\State $n \gets 1$
\State $p \gets \{\xi\}$
\While{$\Call{NotEmpty}{\pb(n)}$}
  \State $a' \gets \Call{SampleFiber}{\pb(n), \pi}$
  \State $a \gets p(n-1)$
  \State $a' \gets \Call{Steer}{a, a'}$
  \If{$\neg$\Call{IsValid}{$a, a'$}}
    \State \BREAK
  \EndIf
  \State $p \gets p \cup \{a, a'\}$
\EndWhile
\State \Return $p$
\EndIf

\LineComment{\fbox{\textbf{Case 4}: $\pi$ is a parallel fibration}}

\State $\Sigma \gets \Call{GetBoundedUniquePermutations}{\pi}$
\For{$\sigma \in \Sigma$}
  \State $p \gets \Call{InterpolatePermutatedPath}{\pb, \sigma, \xi}$
  \If{\Call{IsValid}{$p$}}
    \State \Return $p$
  \EndIf
\EndFor
\State \Return

\end{algorithmic}
\label{alg:pathsectionsearch}
\end{algorithm}

Given a fibration $\pi : A \rightarrow B$, and a solution path on the base space $B$, a valuable primitive for efficient planning is the path section search~\citep{Orthey2021TRO}. 
A path section search is a dedicated local planner which searches exclusively in the path restriction for a valid path. 
In detail, given a solution path $\pb: [0,1] \rightarrow B$, the path
restriction on $A$ is defined as the set $\{x \in A\mid \pi(x) \in
\pb([0,1])\}$, i.e. all elements in $A$, which, when projected onto $B$, lie on
the path image $\pb([0,1])$ in $B$. 
The start and goal configurations for the path section search are elements of
the fibers of $\pb(0)$ and $\pb(1)$ on $A$, respectively. 
Those are either well-defined from a top-down goal, or sampleable using a bottom-up goal. 

The inner workings of the path section search are shown in
Alg.~\ref{alg:pathsectionsearch}. There are four cases we have to address.
First, when $A$ is a leaf node or no solution exists on $B$, there does not
exist a path section. In this case, we can terminate immediately (Line 1--3).
Once this case is exhausted, we can get the solution path from $B$ (Line 4) and
continue.

In the second case, $\pi$ is a sequential fibration. In this case, we can employ
a path section search like section patterns~\citep{Orthey2021TRO}. However, due
to the possible high-dimensionality of the space $A$, especially when using multiple robots, we
used the recursive side step method which we 
introduced in an earlier paper~\citep{Orthey2024IJRR}. 
This method interpolates
an L1-metric path on the path restriction (Line 6--7), and then follows it until
a constraint violation occurs. Once this happens, we compute a new L1-metric
path starting from the violated configuration and continue until no side step is
valid or we reach the goal (Line 8). The resulting path is
then returned (Line 9). This function uses two parameters, the maximum
branching (which determines how many L1 paths are created once a violated
configuration is found) and the maximum depth (which determines how many times a
branching is allowed to happen until we stop).

In the third case, $\pi$ is a partial fibration. In this case, a path
interpolation method is not directly available. Instead, we use the steering
function to move from one configuration to the next along the restriction of the
solution path. We start by initializing a path from the initial configuration
(Line 13). We then steer until we reach the last base path configuration (Line
14). For each base path configuration, we sample a random fiber element (Line
15) and steer towards it (Line 17). If this segment is not valid or we do not
reach the lifted element (Line 18), we break (Line 19) and return the path so far (Line
23). Otherwise, we add the last segment to the path and continue (Line 21). This
method uses two parameters, the restriction accuracy (which determines how close
we need to steer to a fiber element) and the number of propagation attemps (which determines
how often we repeat the steering when unsuccessful).

Finally, in the last case, $\pi$ is a parallel fibration. We found that in this
case, the most efficient way to find a path section is similar to a side step
function~\citep{Orthey2024IJRR}, but instead of interpolating an L1-metric path,
we interpolate each factor separately. For example, when there are three
factors, we would go through every permutation and find a path section for each
without moving the degrees of freedom of the other two factors. 
This method can
be generalized by first finding a set of permutations (Line 25). We cap this
number because the number of permutations grows exponentially with the number of
factors. We then iterate over all permutations (Line 26) and interpolate the
section based on the permutation (Line 27). Once this path section is found, we
check its validity and return the path in case of success (Line 28--29). If no
valid section is found, we return without a path (Line 32). This method uses one
parameter, the maximum number of permutations.

Note that this method is similar to the local connector method~\cite{Solovey2016, VanDenBerg2005Prioritized}, but uses multiple ordering permutations to increase the likelihood of finding a feasible path.

\subsection{Goal Setting\label{sec:goalsetting}}

Fibration trees support two modes of goal setting, a top-down goal and a
bottom-up goal. Those are two ways of specifying problems: In a top-down goal, you specify a single configuration as a goal which is valid for all robots. In a bottom-up goal, you only specify the goals for leaf nodes, which is useful for task-space goals (e.g. the end-effector position) or for simplified robot geometries (e.g. the base of a robot). We explain both types below.

\noindent\textbf{Top-Down Goal} For the top-down goal, a single goal
configuration is set on the root node. In this case, the goals for all nodes are defined by recursively projecting this goal configuration downwards. This provides a unique, well-defined goal configuration for each of the nodes in the fibration tree. Note that a top-down goal, if available, is the recommended choice for fibration tree planning. However, a top-down goal might not always be available or the goal has different regions, where picking a single configuration might lead to making the problem infeasible.

\noindent\textbf{Bottom-Up Goal} A bottom-up goal is defined by specifying a configuration for each leaf node in the fibration tree. This is not well-defined, since multiple root configurations can project on the same leaf-node configurations. This ambiguity can be resolved in one of two ways. Either, before planning, all the leaf-nodes are lifted to the root node state and a validity check is performed. If the state is valid it is taken as a top-down goal. Otherwise, the procedure continues until a maximum number of iterations is reached, or a valid root state is found. This method is called \emph{valid leaf-state lifting}. 

Or, as a second approach, leaf-node states can be used as sampleable goals for their parent nodes. Every goal, including the root goal, is defined implicitly by sampling leaf-node states and then lifting them upwards. Note that this differs from valid leaf-state lifting by postponing the lifting for each node towards a time when a goal state is already available for the corresponding child node. This method is called a \emph{liftable goal region}.

\noindent\textbf{Synthesis of Bottom-Up Approaches} Both bottom-up approaches have disadvantages. One disadvantage is early termination. If a problem is infeasible, a liftable goal region will need to go through all nodes and perform unneccessary computations, even if valid leaf-state lifting could quickly terminate. Another disadvantage is completeness. If valid leaf-state lifting produces a configuration, this configuration might be unreachable, and therefore make the problem infeasible---even if a solution exists. 

In practice, we therefore use a combination of both approaches to resolve their respective disadvantages. First, we use valid leaf-state lifting to guarantee that a valid goal configuration exists. Second, we start planning with liftable goal regions. This makes the planner keep the probabilistic completeness properties, while ensuring that a goal configuration theoretically exists on the root node.

\subsection{L-Infinity Metric}

In this work, \FibrationRRT and all planners in our benchmark use exclusively the L-infinity metric~\citep{atias2018effective}. In this section, we explain our reasoning for this choice.
While earlier work~\citep{atias2018effective} has not found a significant difference in the choice of metric, we have opted to use the L-infinity metric for \FibrationRRT based upon the following reasons.

\begin{itemize}
    \item \textbf{Interpretation} Thresholds for goal regions in
      high-dimensional spaces using Lp-metrics are difficult to interpret.
      However, when we use the L-infinity metric, a threshold can be interpreted as the maximum joint displacement over all DoFs of all robots. This is straightforward to define and intuitive for users of the planner.
    \item \textbf{Avoidance of Clustering} Lp-metrics consider the sum of many
      independent variables, which, by the law of large numbers, often have
      similar expected values. Most of the outcome therefore clusters around certain values~\citep{aggarwal2001surprising}. In contrast, the L-infinity metric provides the largest DoF deviation, which is independent of the number of DoFs.
    \item \textbf{More Meaningful Nearest Neighbor}. When using Lp-metrics, we could get a low value for the distance, even if one robot is having a large distance, since we average over all dimensions. However, when using the L-infinity metric, we ensure that robots are really nearby.    
\end{itemize}

Due to those reasons, we have opted to use the L-infinity metric for \FibrationRRT. To ensure that the metric is not influencing benchmarking, we use the same metric for all planners. We believe future work should investigate the influence of the metric when the dimensionality is increased significantly~\citep{atias2018effective, aggarwal2001surprising}.

\subsection{Software Implementations\label{sec:fibrationrrt:software}}

All code is open source and available as an extension of the OMPL library\footnote{https://github.com/aorthey/ompl/tree/FibrationTrees}. We plan to merge this as a multi-robot extension into the OMPL main github repository.
For evaluations and implementation of partial (task-space) projections, we wrote a framework based upon this OMPL extension. This framework is based upon the Dynamic Animation and Robotics Toolkit (DART)~\citep{lee2018dart} and is also available on github\footnote{https://github.com/aorthey/FibrationTrees}.

\setlength{\subfigheight}{0.3\linewidth}
\setlength{\scenarioheight}{0.4\linewidth}
\def\vdistance{0.7em}
\begin{figure*}[htbp]
    \centering
    
    \begin{subfigure}[b]{0.24\textwidth}
        \centering
        \includegraphics[width=\linewidth,height=\scenarioheight]{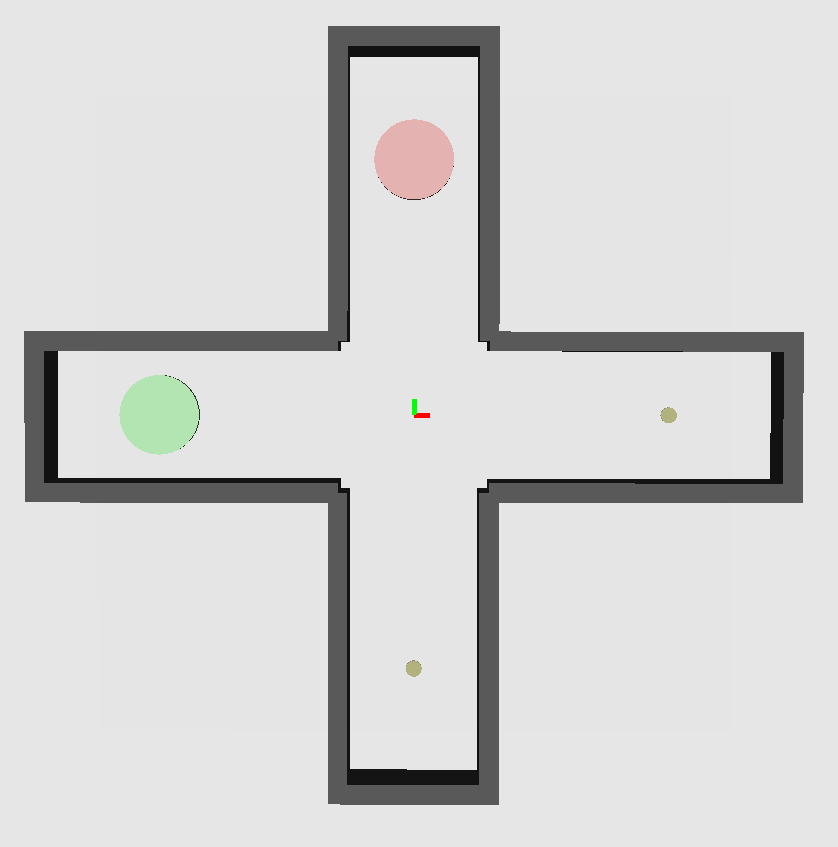}
        
        \vspace{\vdistance}
        \includegraphics[width=\linewidth,height=\subfigheight]{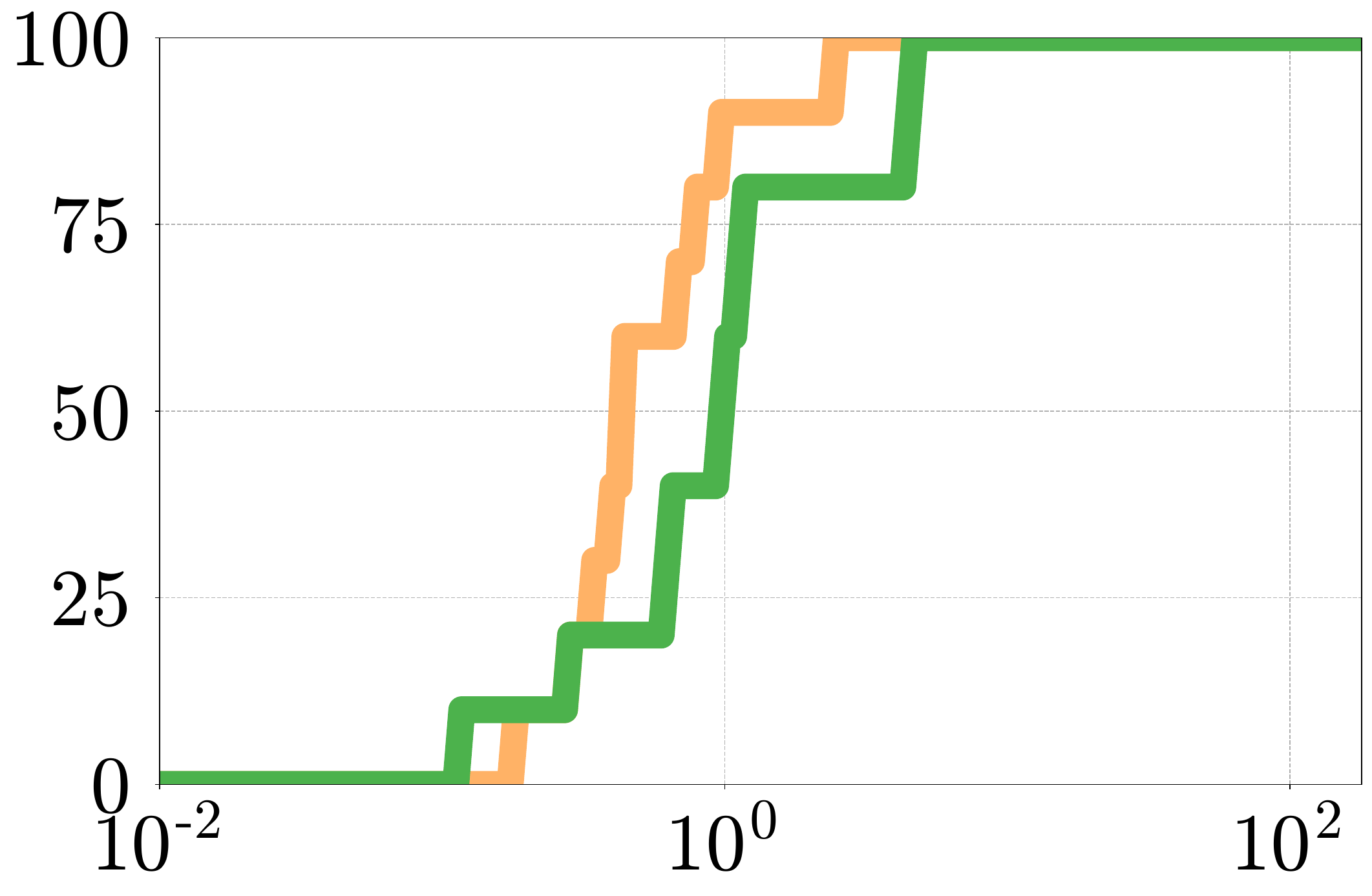}
        \caption{Disks In Crossing}
    \end{subfigure}
    \begin{subfigure}[b]{0.24\textwidth}
        \centering
        \includegraphics[width=\linewidth,height=\scenarioheight]{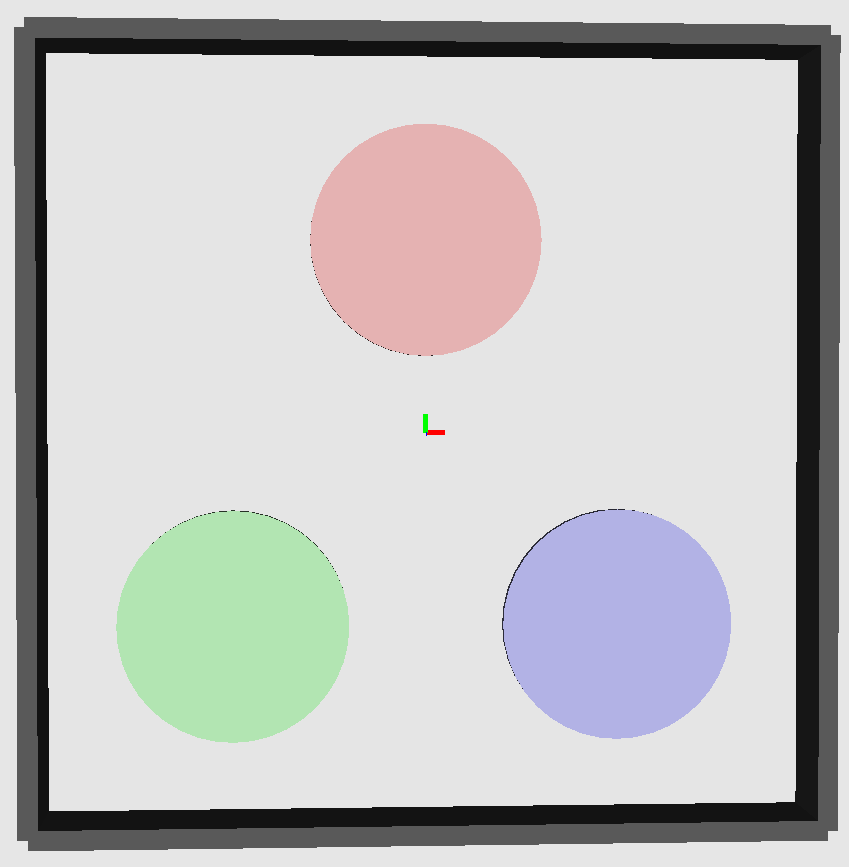}
        
        \vspace{\vdistance}
        \includegraphics[width=\linewidth,height=\subfigheight]{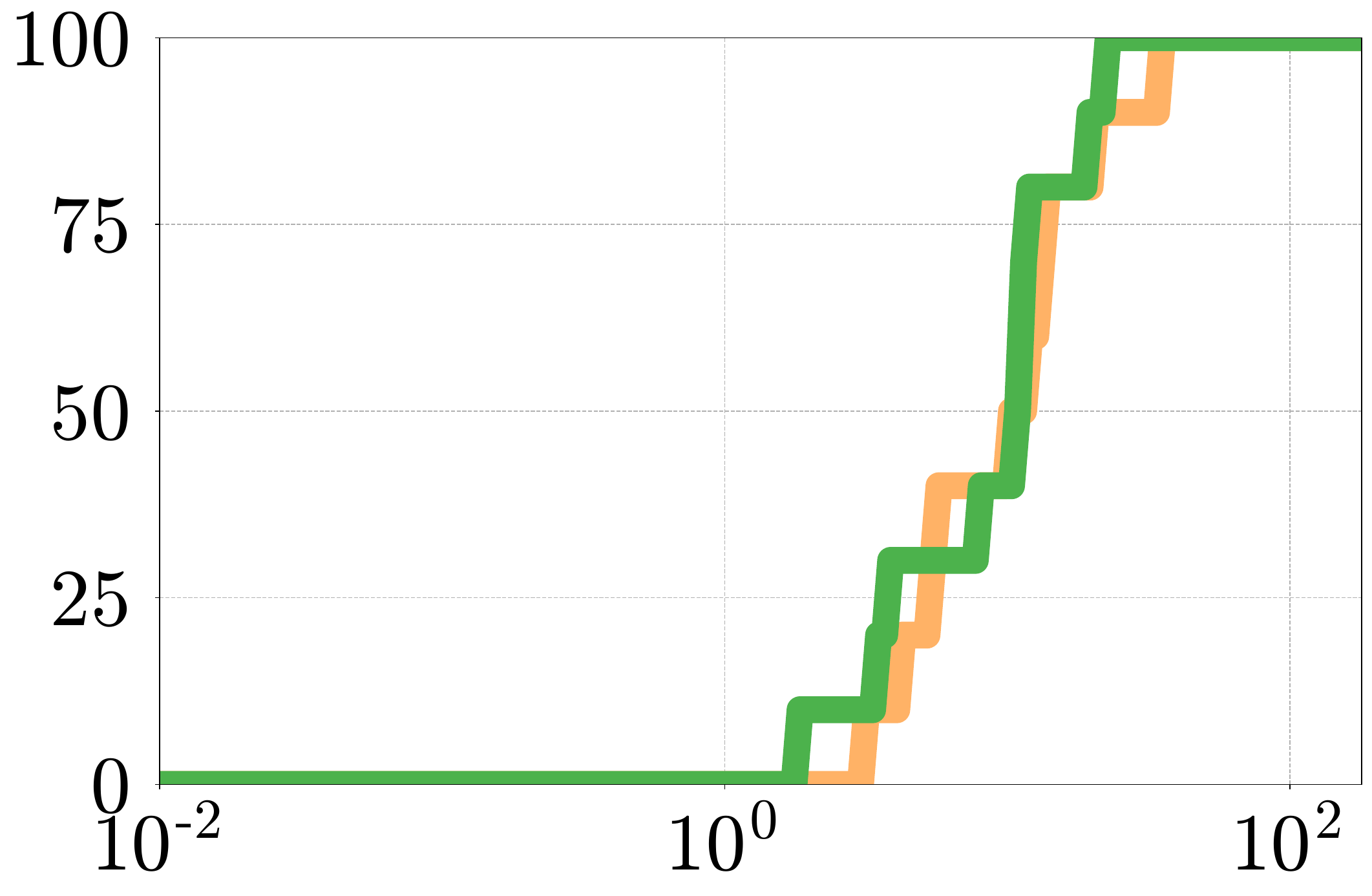}
        \caption{Disks On Square}
    \end{subfigure}
    \begin{subfigure}[b]{0.24\textwidth}
        \centering
        \includegraphics[width=\linewidth,height=\scenarioheight]{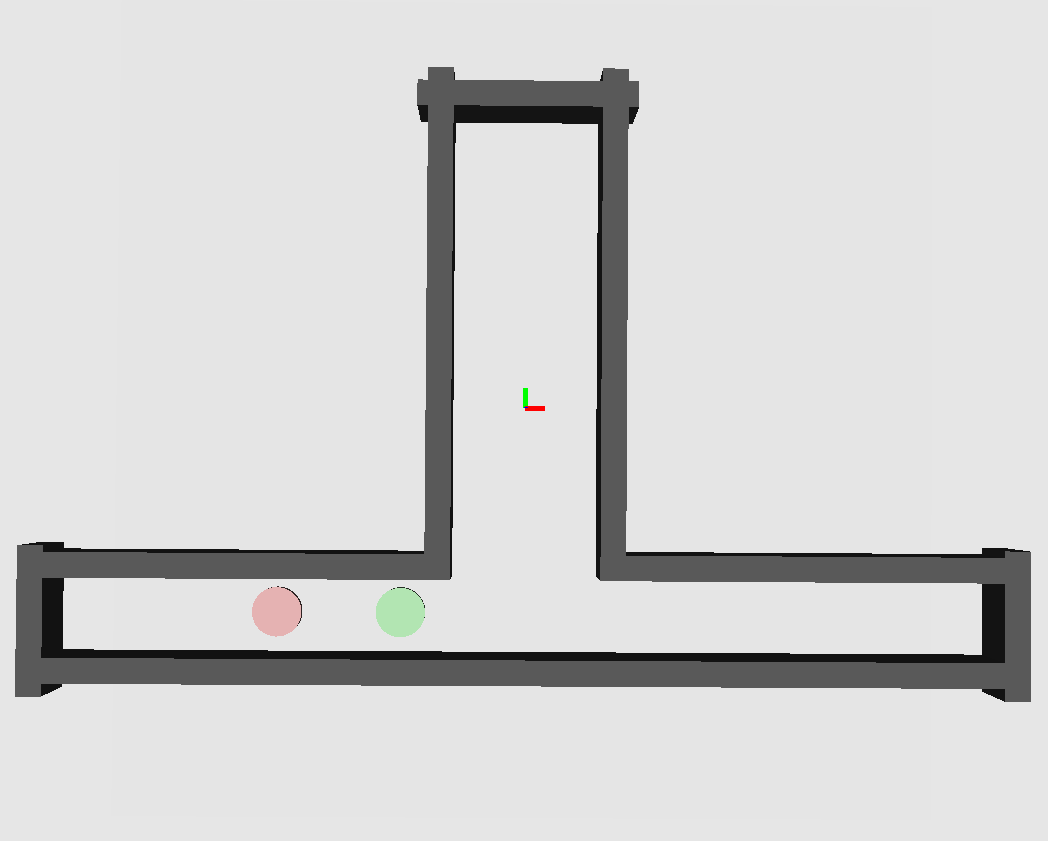}
        
        \vspace{\vdistance}
        \includegraphics[width=\linewidth,height=\subfigheight]{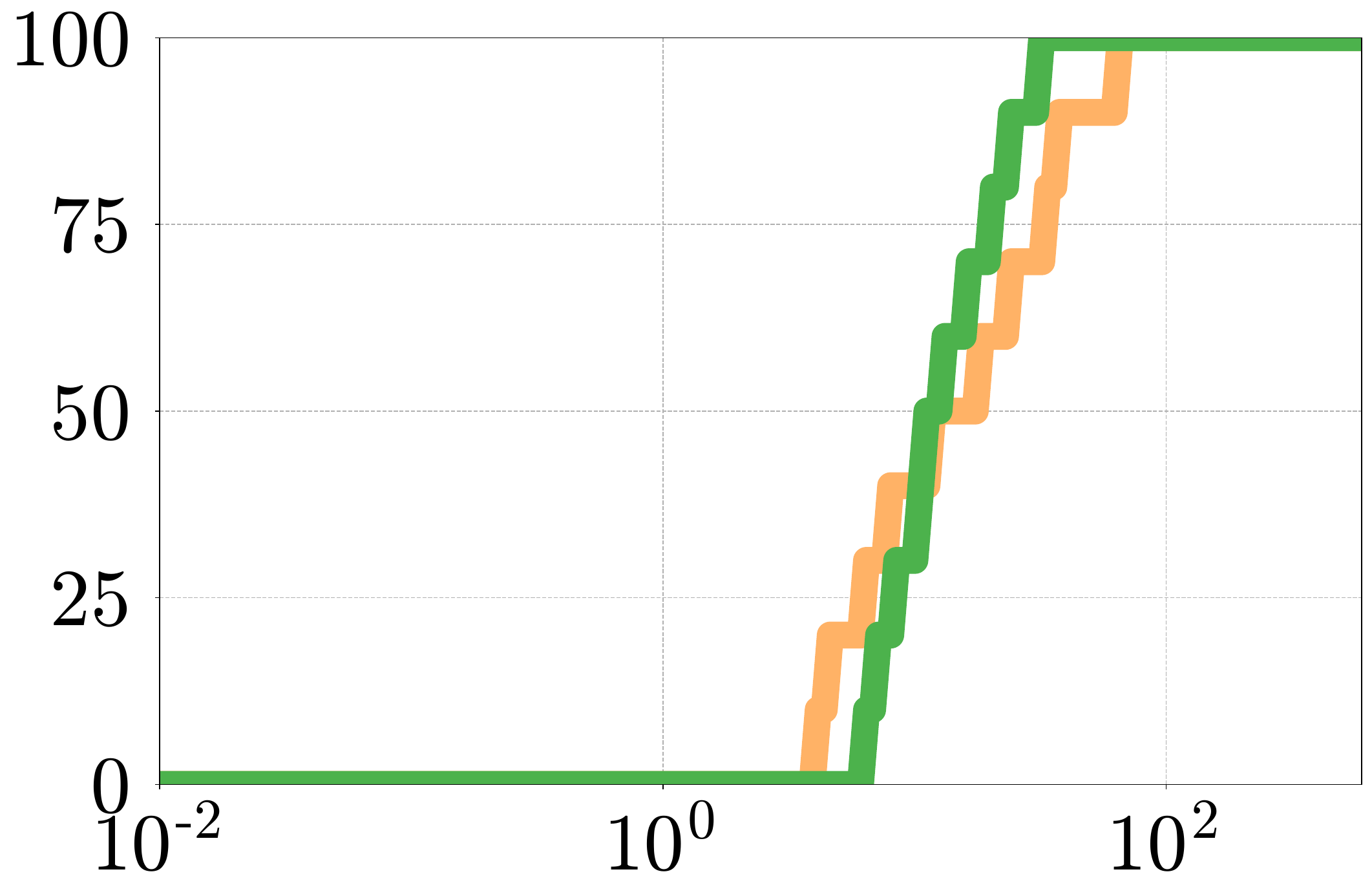}
        \caption{Disks In Tee}
    \end{subfigure}
    \begin{subfigure}[b]{0.24\textwidth}
        \centering
        \includegraphics[width=\linewidth,height=\scenarioheight]{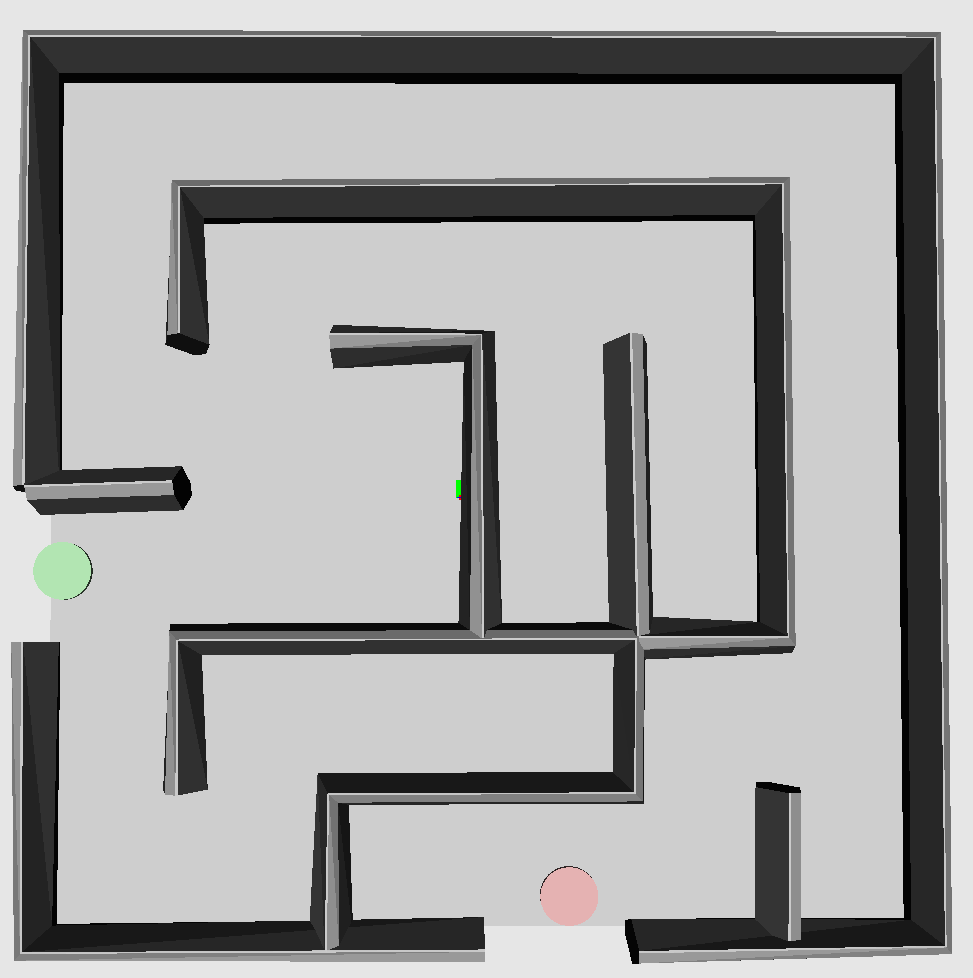}
        
        \vspace{\vdistance}
        \includegraphics[width=\linewidth,height=\subfigheight]{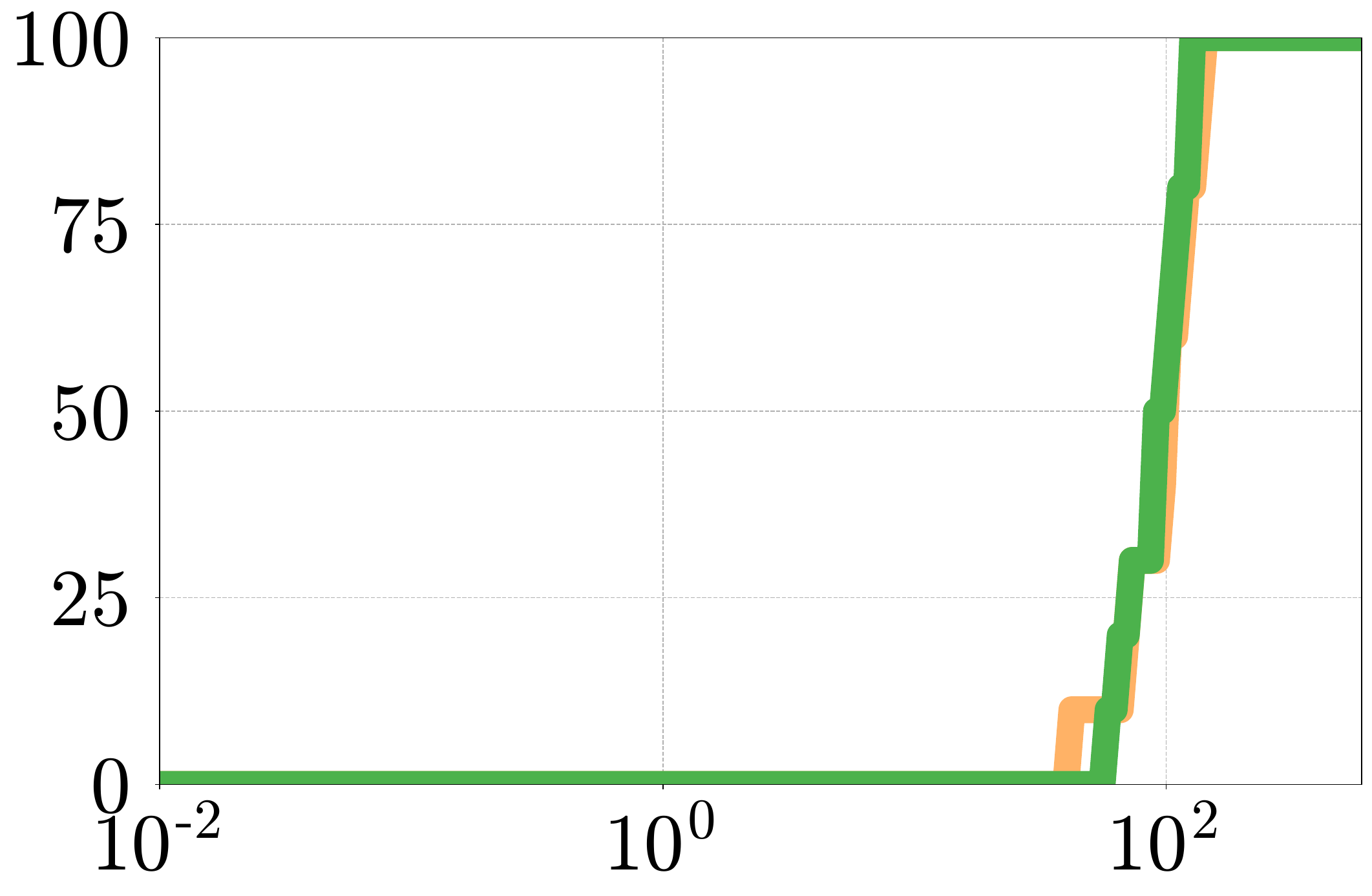}
        \caption{Disks In Maze}
    \end{subfigure}
    
    \vspace{0.8em}

    \centering
    \fbox{\parbox{\dimexpr\linewidth-4\fboxsep-2\fboxrule}{\centering\small
        \vspace{0.15cm}
        \sqbox{colorFTD} \textbf{\FibrationRRT [Ours]} \quad
        \sqbox{colorRRT} RRT \quad
        \vspace{0.15cm}
    }}
    
    \caption{Overhead Benchmarks. Success graphs of the benchmarks for the Multi-Disks scenarios where we plan in component space (Fibration tree has a single root node which represents the component state space).
    Top row shows the scenarios, bottom row shows the success graphs, whereby the $x$-axis shows time in log-space, while the $y$-axis shows the success rate from $0$ to $100$ percent. 
    Colors indicate the planner as shown in the legend above.}
    \label{fig:benchmark_overhead}
\end{figure*}

\section{Benchmarks\label{sec:benchmarks}}

To evaluate \FibrationRRT on different fibration trees, we utilize the DART library~\citep{Lee2018} in combination with OMPL for benchmarking~\citep{Moll2015}. Our evaluations involve five benchmarks:

\begin{itemize}
    \item \textbf{Overhead Benchmark} Comparison of \FibrationRRT with the classical RRT algorithm to find any computational overhead (Sec.~\ref{sec:benchmarks_overhead}).
    \item \textbf{Structural Benchmark} Comparison of decomposition-based and prioritization-based fibration trees using \FibrationRRT and comparing them to single-tree, single-query planners (Sec.~\ref{sec:benchmarks_structure}).
    \item \textbf{Prioritization Benchmark} Comparison of \FibrationRRT using a prioritization-based fibration tree against a prioritization-based planner, namely the rapidly exploring random quotient-space tree (QRRT) planner~\citep{Orthey2024IJRR} (Sec.~\ref{sec:benchmarks_prioritization}).
    \item \textbf{Decomposition Benchmark} Comparison of \FibrationRRT using a decomposition-based fibration tree against a decomposition-based planner, namely multi-robot discrete RRT (dRRT) planner~\citep{Solovey2016} (Sec.~\ref{sec:benchmarks_decomposition}).
    \item \textbf{Task-Space Benchmark} Comparison of \FibrationRRT using task-space constraints with Task-RRT~\citep{berenson2011task} (Sec.~\ref{sec:benchmarks_task_space}).
\end{itemize}

We report in the following sections on the results of those benchmarks. Please see Sec.~\ref{sec:discussion} on the interpretation of those results.

\subsection{Experimental Setup\label{sec:experimental-setup}}

We evaluate our benchmarks on a $4$-core, $8$GB RAM laptop running Ubuntu $24.04$. 
The following parameters are specified for \FibrationRRT, together with the default values set for the experiments.
\begin{itemize}
    \item \textbf{Goal bias} A value in $[0,1]$ indicating the percentage of samples drawn from the goal. The default value is $0.05$ (same as RRT).
    \item \textbf{Range} A value in $\Rnonneg$ indicating the maximum extension distance of a single iteration. The default value is computed from the volume of the configuration space (same as RRT). 
    \item \textbf{Selector function type} A type indicating that either uniform selection, exponential selection (Sec.~\ref{sec:planner-selection}), or the last active level is used. The default is exponential selection.
\end{itemize}

Additionally, we need to specify parameters for the two primitives of
restriction sampling and path section search.
For unified restriction sampling, we use a path restriction sampling bias of
$0.5$, a path restriction surrounding bias of $0.1$, and a sampling perturbation value of $0.05$. 
For unified path section search, we use a maximum branching of $2$, a maximum
depth of $5$, a restriction accuracy of $0.1$, a number of propagation attempts
of $5$, and a maximum number of permutations of $10$.

We run each planner in each benchmark for $10$ runs and collect the time when a first feasible solution has been found. 
The results are visualized using a success graph, which shows the success percentage over runtime. Please see Tab.~\ref{tab:scenario-properties} for the properties of each scenario.
\subsection{Overhead Benchmark\label{sec:benchmarks_overhead}}
In the first four scenarios, we like to establish a baseline to show that \FibrationRRT performs similar to RRT when no fibration trees are used. This should show that the overhead of using of fibration trees is minimal. We call those experiments the overhead experiments.

The scenarios are shown in Fig.~\ref{fig:benchmark_overhead} (top row). Each
scenario is a multi-disk planning problem, inspired by problems in the
multi-robot literature~\citep{Solovey2016,Orthey2020WAFR}. In scenario Disks in Crossing (Fig.~\ref{fig:benchmark_overhead} (a)), two disks need to traverse a crossing to the opposite side while avoiding collision with themselves. In scenario Disks on Square (Fig.~\ref{fig:benchmark_overhead} (b)), three larger disks are located inside a square and have to move to the opposite side of their start location. The green disks has to move to the upper right corner, the blue disks to the upper left corner, and the right disk to the bottom center position. In scenario Disks in Tee (Fig.~\ref{fig:benchmark_overhead} (c)), two disks have to move from the left to the right side of a Tee-shaped structure. However, both disks have to exchange their positions to reach their respective goal configurations. Finally, in scenario Disks in Maze (Fig.~\ref{fig:benchmark_overhead} (d)), two disks are located on opposite ends of a maze and have to exchange positions by navigating the maze while avoiding collisions with each other. 

\subsubsection{Results}

The results for \FibrationRRT and RRT are shown in Fig.~\ref{fig:benchmark_overhead} (bottom row). The scenarios are ordered in ascending difficulty, meaning it takes, on average, longer to solve (d) than it takes to solve (a). It can be seen that there is a small, visible difference for low runtimes in (a), but both algorithms perform relatively similar in (b), (c), and (d). 

The results show that there seems to be a small, measurable overhead for low runtimes, but that it vanishes for higher runtimes and does not constitute a bottleneck. 

\subsection{Structural Benchmark\label{sec:benchmarks_structure}}

In the structural benchmark, we compare \FibrationRRT with different fibration trees as input to a suite of single-tree planners. This benchmark uses eight multi-robot scenarios as depicted in Fig.~\ref{fig:scenarios}. We compare \FibrationRRT with a decomposition-based
fibration tree (\FibrationRRT-Decomposition) and \FibrationRRT with a
prioritization-based fibration tree (\FibrationRRT-Prioritization). Additionally, we compare those two planners to the single-query, single-tree planners rapidly-exploring random tree (RRT)~\citep{Lavalle1998}, expansive space-trees (EST)~\citep{Hsu1999EST}, fast marching trees (FMT)~\citep{Janson2015FMT}, and lower bound tree-RRT (LBT-RRT)~\citep{Salzman2016LBTRRT}.

\subsubsection{Scenarios}

The eight scenarios consists of the following problems. In Multi Disks
(Fig.~\ref{fig:scenarios} (a)), we have a
scenario of eight disk robots (2-dof) with four on the left and four on the right of
a plane. Both sets of disks have to change their position with the opposite
disk. A single cylindrical obstacle in the middle of the plane has to be
avoided. This scenario has a total of 16-dof.

In Airship Coordination (Fig.~\ref{fig:scenarios} (b)), eight flying airship robots are
given which can move freely in space while rotating around its Z-axis (4-dof).
Those airships have to fly from one side of the plane to the other while
avoiding collisions with themselves and with multiple objects in the
environment. This scenario has a total of 32-dof.

In Mobile Navigation (Fig.~\ref{fig:scenarios} (c)), a total of six mobile
manipulator robots are given. Each robot consists of a base (3-dof) and a
manipulator arm (7-dof) for a total of 10-dof. Two robots are tasked to move
from the right side to the left, while the other robots have to move from the
left side to the right. The goal positions for the robots are indicated by the transparent yellow robot (showing only base position for visibility). This scenario has a total of 60-dof.

In Reeds Shepp (Fig.~\ref{fig:scenarios} (d)), four cars are given, which are
modelled as Reeds Shepp cars~\citep{reeds1990optimal} (3-dof). The tasks consists
of a parking lot, where two cars have to leave while two cars have to enter and
reach the previous occupied parking spots.
This scenario has a total of 12-dof.

In Cube Robots (Fig.~\ref{fig:scenarios} (e)), a total of nine cubes are given,
which can translate and rotate on the plane (3-dof each). The cubes have
different sizes, and need to move each from one side of the plane to the
opposite side while avoiding collisions. This scenario has a total of 24-dof.

In Drones in Pipe (Fig.~\ref{fig:scenarios} (f)), we have a scenario of sixteen
drone robots which can freely translate and rotate in space (6-dof). The task
requires the drones to fly through a narrow pipe with metal lattices. This
scenario has a total of 96-dof.

In Forest (Fig.~\ref{fig:scenarios} (g)), we use four mobile manipulator robots
(as in Mobile Navigation) which have to trade places with each other in a
diagonal fashion. The difficulty here consists of a narrow passage environment
where multiple tall cuboids are blocking the robots. This scenario has a total
of 40-dof.

Finally, in Warehouse (Fig.~\ref{fig:scenarios} (h)), we 
use two forklifts (3-dof) and three lifting trucks (3-dof) which have to reach
specific positions inside a warehouse. This scenario has a total of 15-dof.

{
\newcolumntype{Y}{>{\centering\arraybackslash}m{.002\linewidth}}
\newcolumntype{L}{>{\raggedright\let\newline\\\arraybackslash\hspace{0pt}}m{.53\linewidth}}

\def\customIndent{\hspace{3mm}}

\begin{table}[t]
    \vspace{1em}
    \caption{Overview of the properties of each scenario used for benchmarking. 
    The properties include if a scenario is homogeneous (all robots are the
    same), heterogeneous (robots differ), task-constrained (IK solver required),
    space-time (planning includes a time dimensions and dynamic obstacles), and
    if the fibration trees used includes parallel, sequential, or partial fibrations.
    For each scenario, we note if a certain property exists (green), or does not
    exists (grey).\label{tab:scenario-properties}}
    \footnotesize
    \center
    \begin{tabularx}{\linewidth}{| X | *{7}{Y} |}
        \hline
        Scenario & 
        \rotatebox{90}{Homogeneous} &  
        \rotatebox{90}{Heterogeneous} &  
        \rotatebox{90}{Task-Constraint} &  
        \rotatebox{90}{Space-Time} &
        \rotatebox{90}{Parallel} &
        \rotatebox{90}{Sequential} &
        \rotatebox{90}{Partial} \\
        \hline
        Multi Disks (16-dof) & \markYes & \markNo & \markNo & \markNo & \markYes & \markYes & \markNo \\
        Zeppelin Control (32-dof) & \markYes & \markNo & \markNo & \markNo & \markYes & \markYes & \markNo \\
        Mobile Navigation (60-dof) & \markYes & \markNo & \markNo & \markNo & \markYes & \markYes & \markNo \\
        Reeds Sheep (12-dof) & \markYes & \markNo & \markNo & \markNo & \markYes & \markYes & \markNo \\
        \hline
        Cube Robots (24-dof) & \markNo & \markYes & \markNo & \markNo & \markYes & \markYes & \markNo \\
        Drones In Pipe (96-dof) & \markYes & \markNo & \markNo & \markNo & \markYes & \markYes & \markNo \\
        Forest (40-dof) & \markYes & \markNo & \markNo & \markNo & \markYes & \markYes & \markNo \\
        Warehouse (25-dof) & \markNo & \markYes & \markNo & \markNo & \markYes & \markYes & \markNo \\
        \hline
        Fixed Manipulator (7-dof) & \markYes & \markNo & \markYes & \markNo & \markNo & \markNo & \markYes \\
        Fixed on Wall (14-dof) & \markYes & \markNo & \markYes & \markNo & \markYes & \markNo & \markYes \\
        Mobile Manipulators (30-dof) & \markYes & \markNo & \markYes & \markNo & \markYes & \markNo & \markYes \\
        Path Velocity Decomposition (11-dof) & \markYes & \markNo & \markYes & \markYes & \markNo & \markYes & \markYes \\
        \hline
    \end{tabularx}
\end{table}
}

\subsubsection{Results}

Fig.~\ref{fig:benchmark_structure} shows the results for this benchmark. It can be seen that \FibrationRRT achieves a success rate of over $80$ percent
in $12$ out of $12$ cases and reaches a $100$\% success rate in $7$
cases---given the time limits of each scenario. The
planners RRT, EST, FMT, and LBTRRT achieve a success rate of over $80$ percent
on $3$ out of $8$ cases and reach a $100$\% success rate in $1$ case, namely FMT for the Cube
Robots (Fig.~\ref{fig:benchmark_structure}e).

In the cases where \FibrationRRT reaches $100$\% success rate, it outperforms the
remaining planners in $5$ cases by at least one order of magnitude in reaching the
$100$\% success rate. 
This is the case for Multi Disks (Fig.~\ref{fig:benchmark_structure}a), Airship Coordination (Fig.~\ref{fig:benchmark_structure}c), Reeds-Shepp Simple (Fig.~\ref{fig:benchmark_structure}d), Drones in Pipe (Fig.~\ref{fig:benchmark_structure}f), and Warehouse (Fig.~\ref{fig:benchmark_structure}h).

\subsection{Prioritization Benchmark\label{sec:benchmarks_prioritization}}

In this benchmark, we compare \FibrationRRT to the quotient-space RRT (QRRT)~\citep{Orthey2019ISRR, Orthey2024IJRR}, by comparing a prioritization-based multi-robot motion planner to \FibrationRRT using a prioritized fibration tree. We compare the two planners on the same set of multi-robot scenarios as for the structural benchmark (see Fig.~\ref{fig:scenarios}). Ideally, the benchmarks should show that \FibrationRRT performs equivalent to QRRT, while having more flexibility in the input. To compare both algorithms, we ensured that QRRT uses the same path section method, same selection method (exponential importance criterion), and same restriction sampler with equivalent parameters. Note that QRRT is just one prioritization-based algorithm which we use to showcase that \FibrationRRT can handle such scenarios in a similar fashion.

\subsubsection{Results}

The results are shown in Fig.~\ref{fig:benchmark_prioritization}. Both algorithms, \FibrationRRT and QRRT, perform relatively similar on the first six scenarios (a--f), where there is no clear advantage for either one of them. However, in the last two scenarios (g, h), \FibrationRRT outperforms QRRT slightly by reaching a success rate of 100\% inside the time budget, while QRRT fails to do this. 

\subsection{Decomposition Benchmark\label{sec:benchmarks_decomposition}}

In the decomposition benchmark, we compare \FibrationRRT using a decomposition-based fibration tree (a single parallel fibration onto the individual robot spaces) with the decomposition-based multi-robot discrete RRT (dRRT) planner~\citep{solovey_2016}. The scenarios are the same as for the structural and prioritization benchmark (see Fig.~\ref{fig:scenarios}).

To guarantee a fair comparison, we created four different versions of dRRT because of one important distinction between \FibrationRRT and dRRT. While \FibrationRRT uses a selection function (Sec.~\ref{sec:planner-selection}) to decide with which planner to continue, dRRT has a fixed selection sequence. This sequence first computes roadmaps on the individual spaces until a fixed number of vertices has been generated or a time budget has been reached~\citep{solovey_2016}. After those roadmaps have been generated, the planner invokes the dRRT coordination method which explores the (implicit) tensor graph of the individual roadmaps. It is therefore important to choose a good time budget for the roadmap generation phase of dRRT. To have a fair comparison, we opted to include four different values, namely time budgets of $1$s, $5$s, $10$s, and $50$s. We call the corresponding planners dRRT-1, dRRT-5, dRRT-10, and dRRT-50, respectively. 

\subsubsection{Results}

The results are shown in Fig.~\ref{fig:benchmark_decomposition}. 
\FibrationRRT is able to solve six scenarios with a success rate of over 90\% while only having a success rate of 20\% for the remaining two scenarios, namely Drones in Pipe (Fig.~\ref{fig:benchmark_decomposition}f) and Mobile Robots Forest (Fig.~\ref{fig:benchmark_decomposition}g). dRRT has a success rate of over 90\% for three scenarios (a, d, e), reaches 50\% for one scenario (c), but fails for the remaining four scenarios. 

\subsection{Task-Space Benchmark\label{sec:benchmarks_task_space}}

In this benchmark, we evaluate \FibrationRRT on scenarios containing task-space constraints on the end-effector of the robots. To compare this planner to
classical planners, we implement a task-space version of RRT (Task-RRT). Task-RRT
extends RRT in two aspects. First, a task-space sampler is implemented, which
generates samples which are constrained on the respective task-space. For
example, for the Vertical Maze scenario (see below), samples are generated via an inverse kinematics (IK) sampler who constraints the end-effector to lie close to the surface, consistent with the task constraint used by \FibrationRRT. 
Second, we use a forward propagation function, which moves the robot towards a
target sample in such a way that its end-effector is guaranteed to stay on the
task-space. This is accomplished using a Jacobian-based method~\citep{Craig2005}
to move the robot. Thus, the planner is similar to the task-space region
RRT~\citep{berenson2011task}, but uses an explicit representation of the task
manifold for sampling~\citep{kingston2019}.
\setlength{\subfigheight}{0.3\linewidth}

\begin{figure*}[htbp]
    \centering
    \begin{subfigure}[b]{0.24\textwidth}
        \centering
        \includegraphics[width=\linewidth,height=\subfigheight]{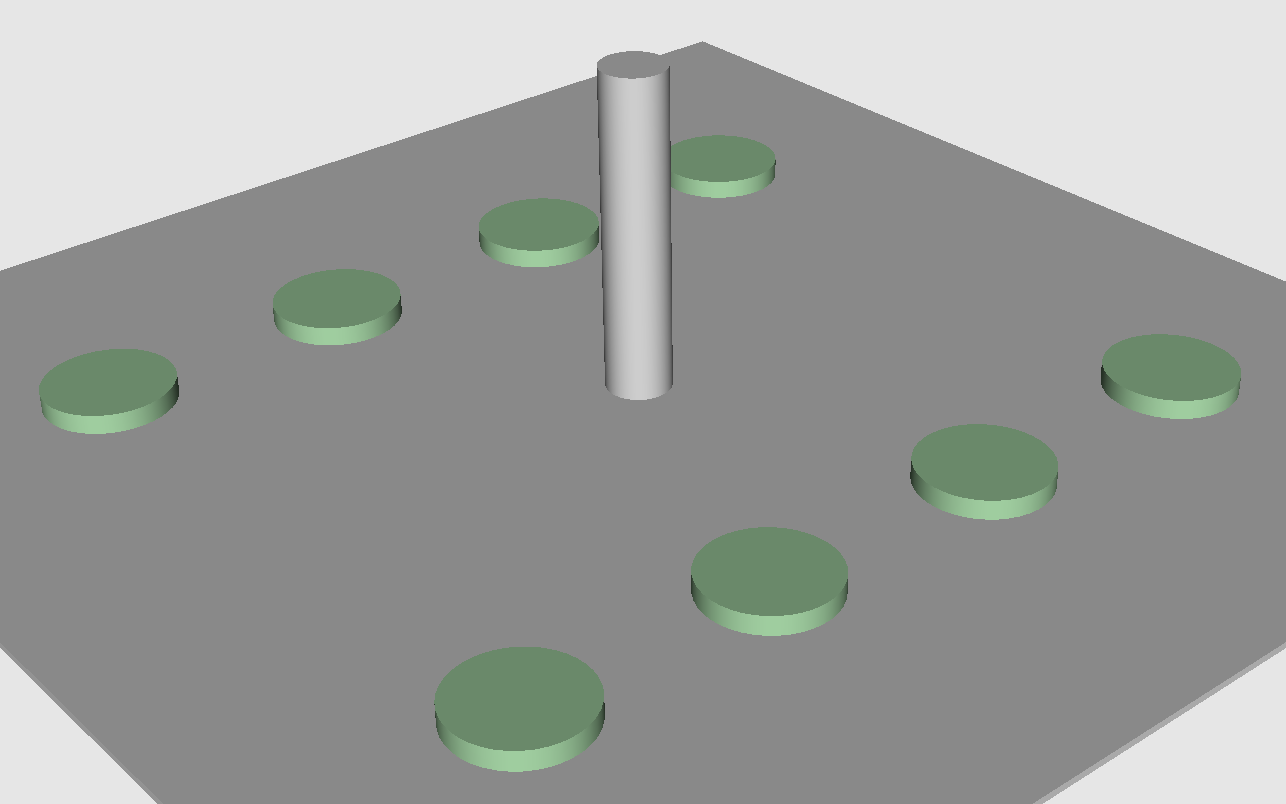}
        \caption{Multi Disks (16-dof)}
    \end{subfigure}
    \begin{subfigure}[b]{0.24\textwidth}
        \centering
        \includegraphics[width=\linewidth,height=\subfigheight]{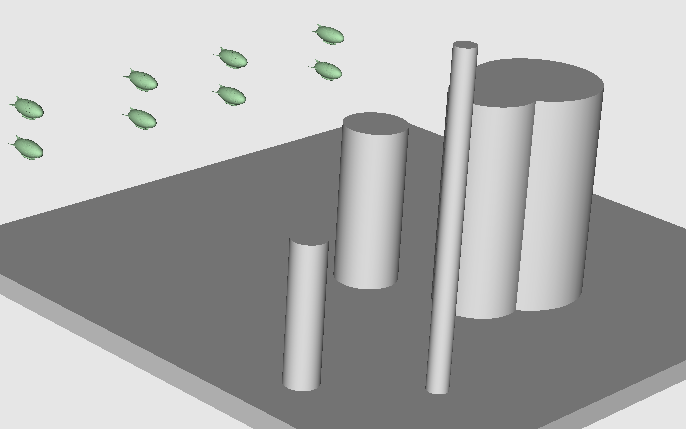}
        \caption{Airship Coordination (32-dof)}
    \end{subfigure}
    \begin{subfigure}[b]{0.24\textwidth}
        \centering
        \includegraphics[width=\linewidth,height=\subfigheight]{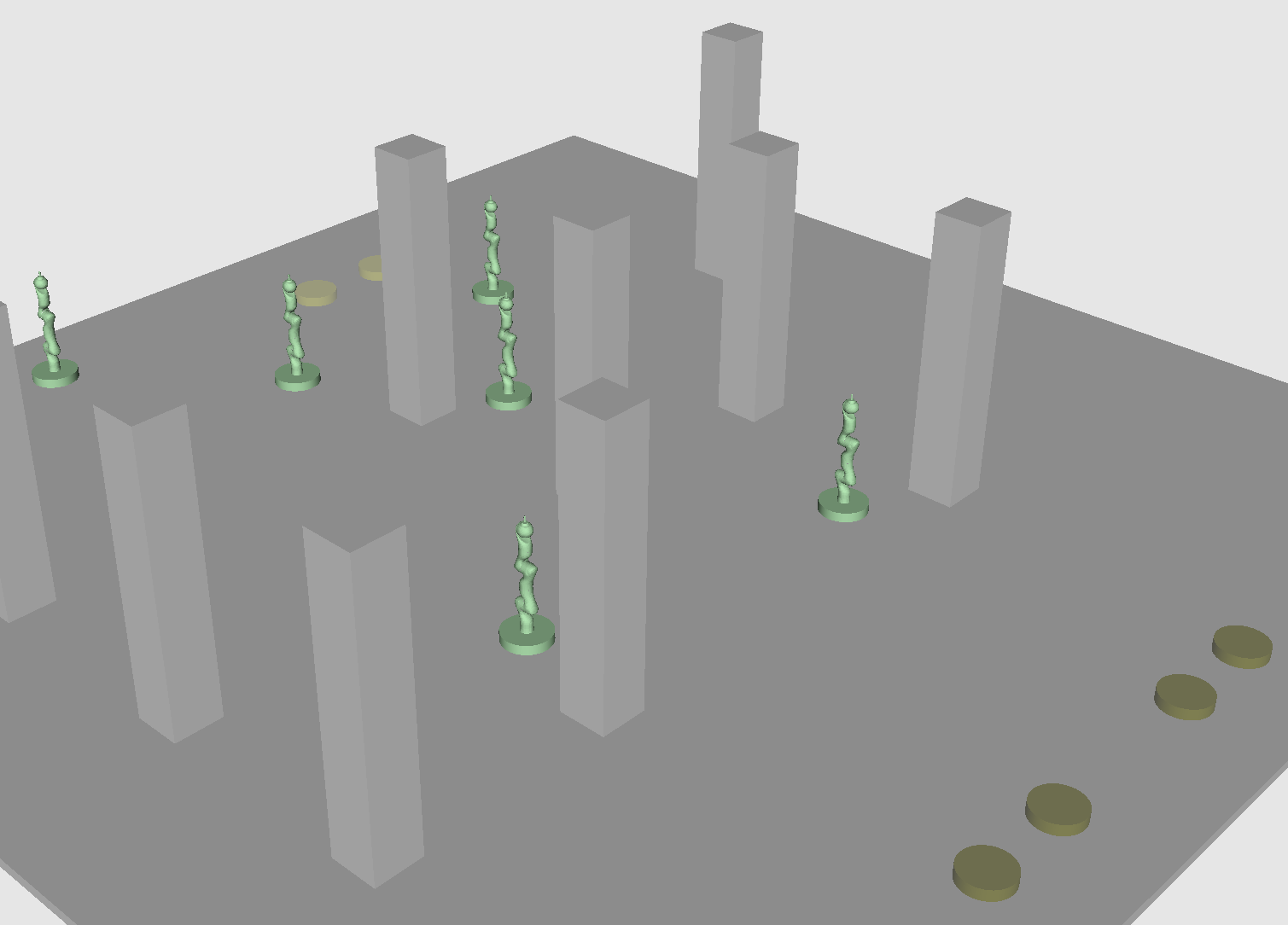}
        \caption{Mobile Navigation (60-dof)}
    \end{subfigure}
    \begin{subfigure}[b]{0.24\textwidth}
        \centering
        \includegraphics[width=\linewidth,height=\subfigheight]{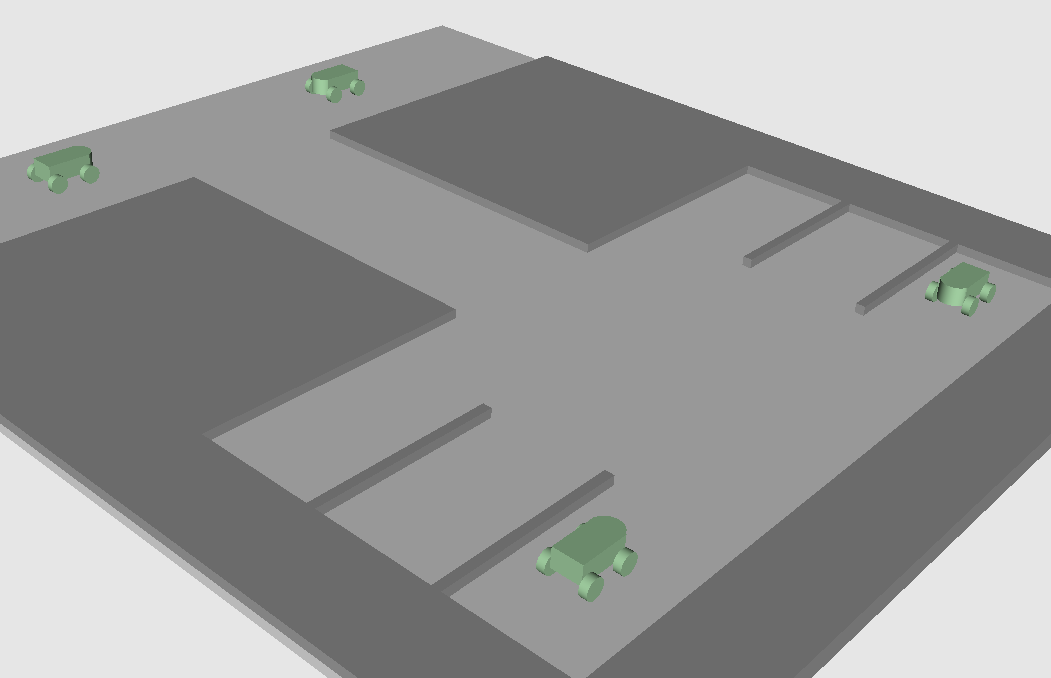}
        \caption{Reeds Shepp (12-dof)}
    \end{subfigure}

    \vspace{0.5em}

    \begin{subfigure}[b]{0.24\textwidth}
        \centering
        \includegraphics[width=\linewidth,height=\subfigheight]{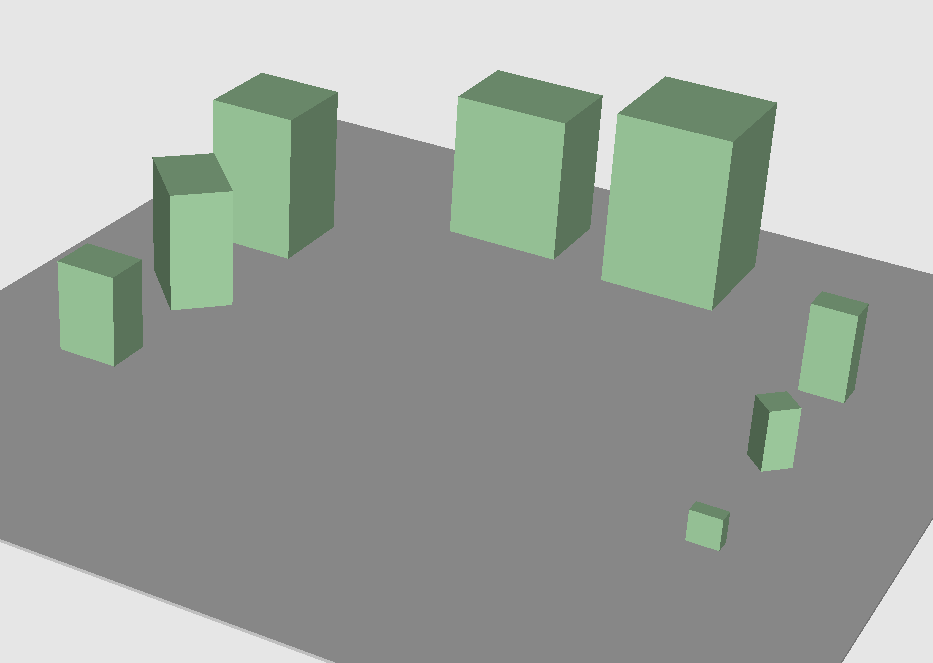}
        \caption{Cube Robots (24-dof)}
    \end{subfigure}
    \begin{subfigure}[b]{0.24\textwidth}
        \centering
        \includegraphics[width=\linewidth,height=\subfigheight]{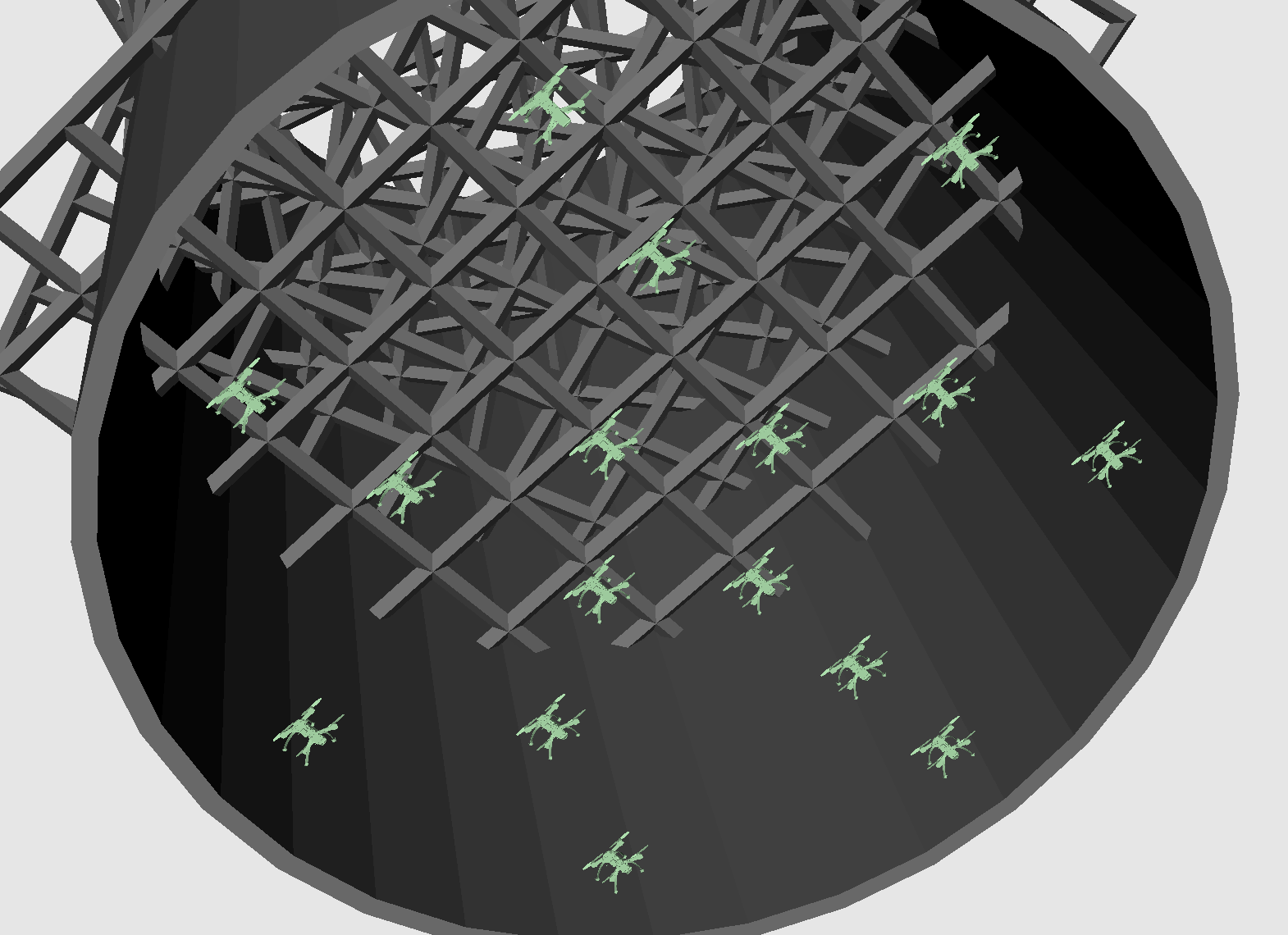}
        \caption{Drones In Pipe (96-dof)}
    \end{subfigure}
    \begin{subfigure}[b]{0.24\textwidth}
        \centering
        \includegraphics[width=\linewidth,height=\subfigheight]{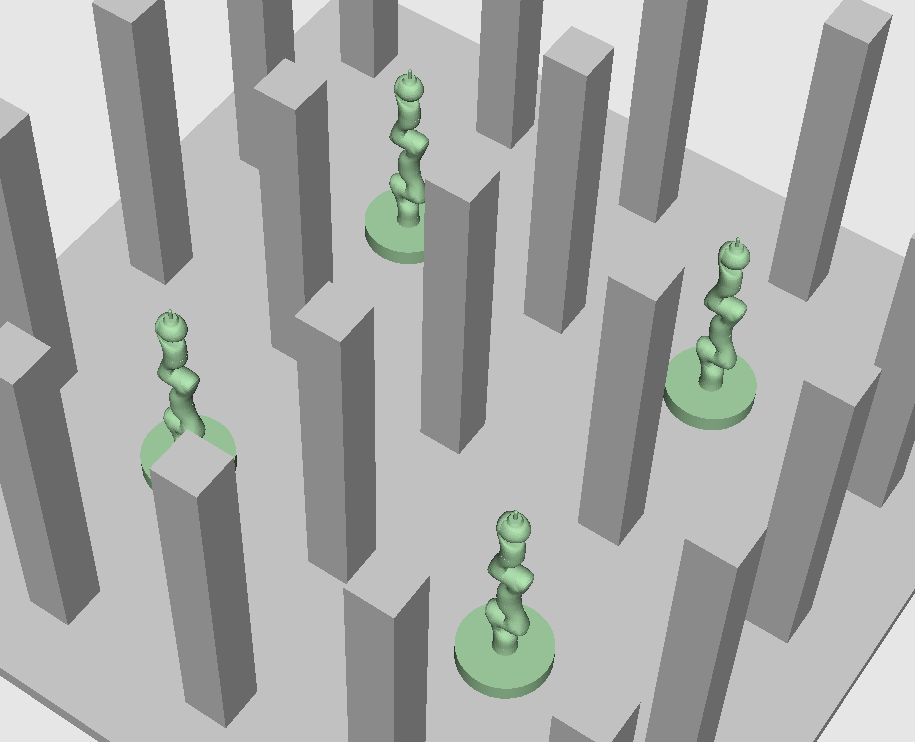}
        \caption{Forest (40-dof)}
    \end{subfigure}
    \begin{subfigure}[b]{0.24\textwidth}
        \centering
        \includegraphics[width=\linewidth,height=\subfigheight]{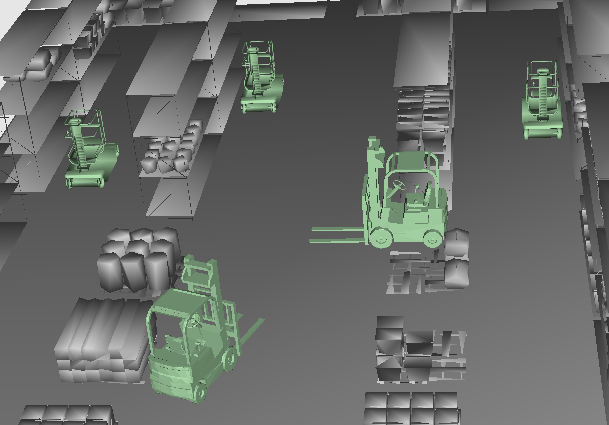}
        \caption{Warehouse (15-dof)}
    \end{subfigure}

    \vspace{0.5em}


    \caption{Experiment used for benchmarking. Movable robots are depicted in
    light green, obstacles in gray, and movable obstacles in light red. First
    and second row are multi-robot navigation scenarios, while the third row
    uses task constraints on the endeffector and dynamic obstacles in space-time.
    \label{fig:scenarios}}
\end{figure*}

\setlength{\subfigheight}{0.3\linewidth}

\begin{figure*}[htbp]
    \centering
    \begin{subfigure}[b]{0.24\textwidth}
        \centering
        \includegraphics[width=\linewidth,height=\subfigheight]{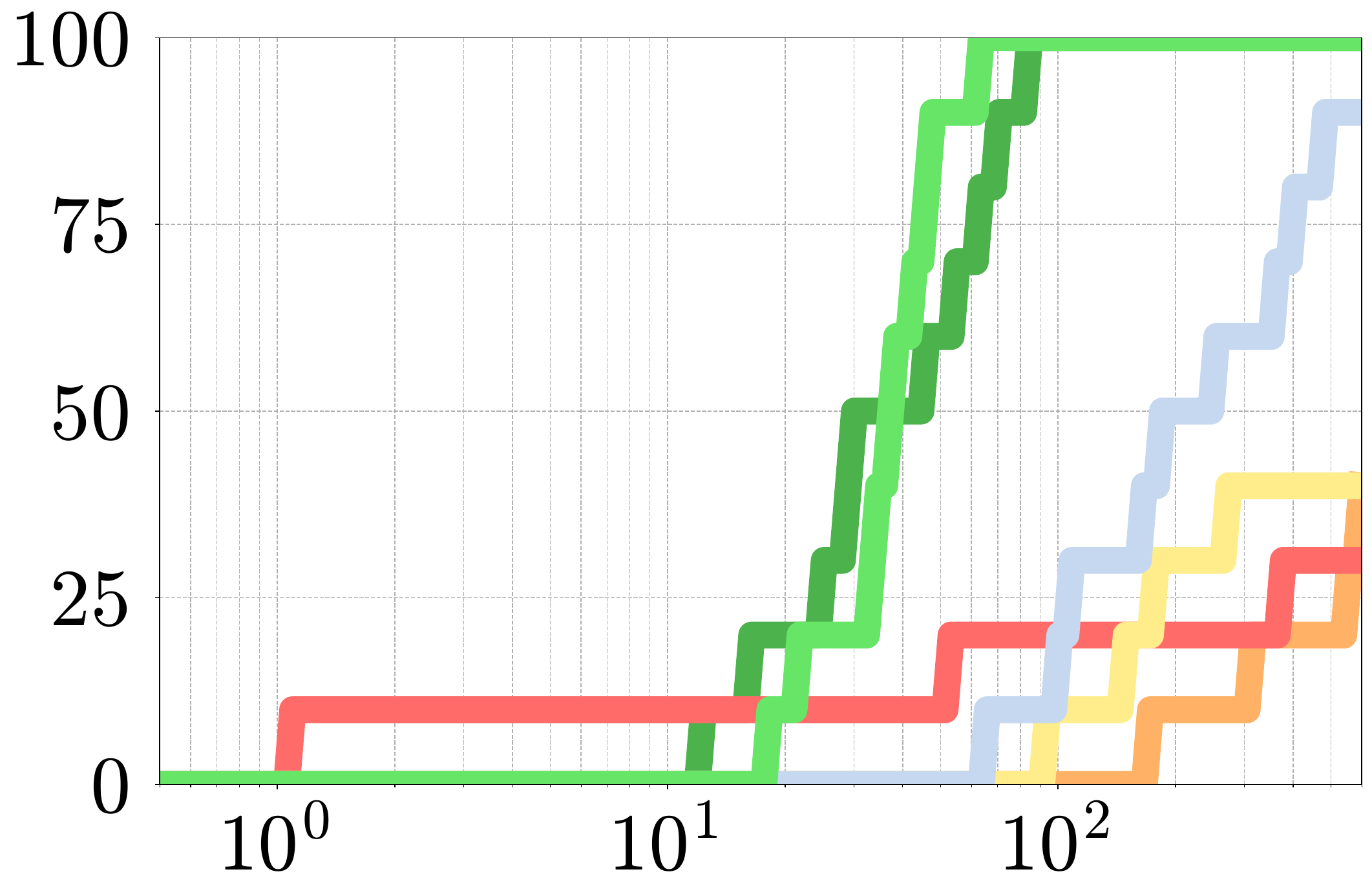}
        \caption{Multi Disks}
    \end{subfigure}
    \begin{subfigure}[b]{0.24\textwidth}
        \centering
        \includegraphics[width=\linewidth,height=\subfigheight]{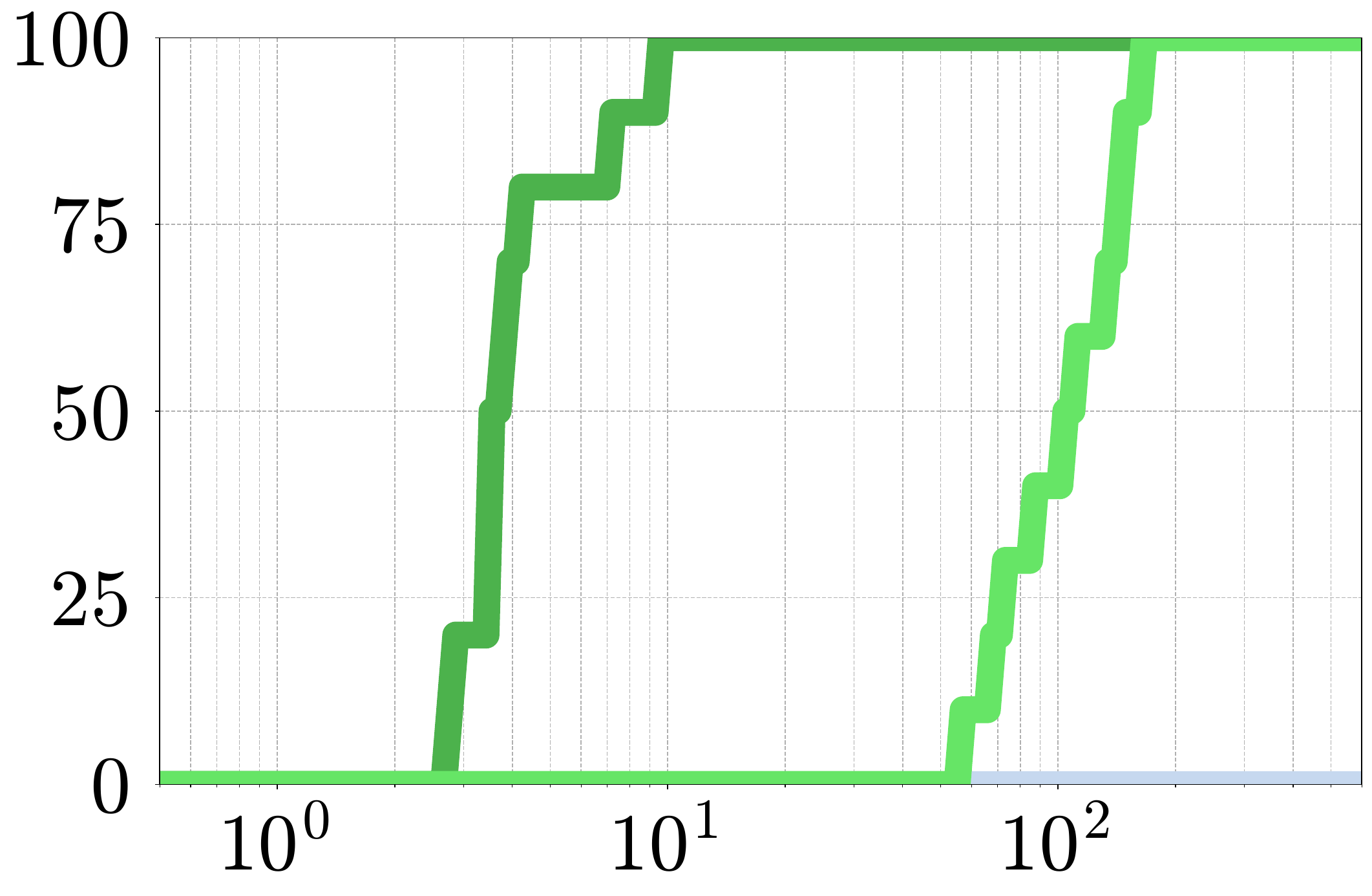}
        \caption{Airship Coordination}
    \end{subfigure}
    \begin{subfigure}[b]{0.24\textwidth}
        \centering
        \includegraphics[width=\linewidth,height=\subfigheight]{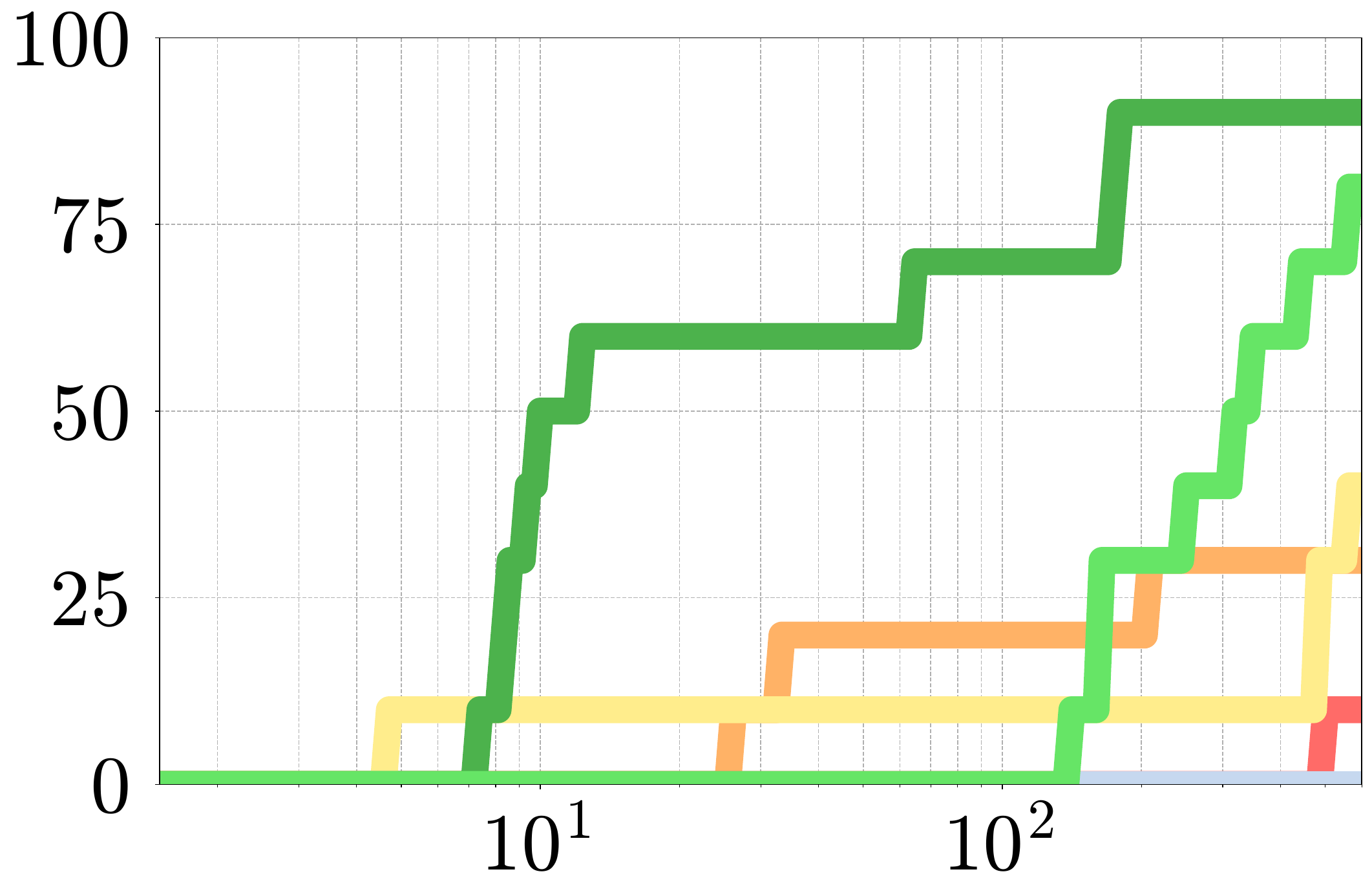}
        \caption{Mobile Navigation}
    \end{subfigure}
    \begin{subfigure}[b]{0.24\textwidth}
        \centering
\includegraphics[width=\linewidth,height=\subfigheight]{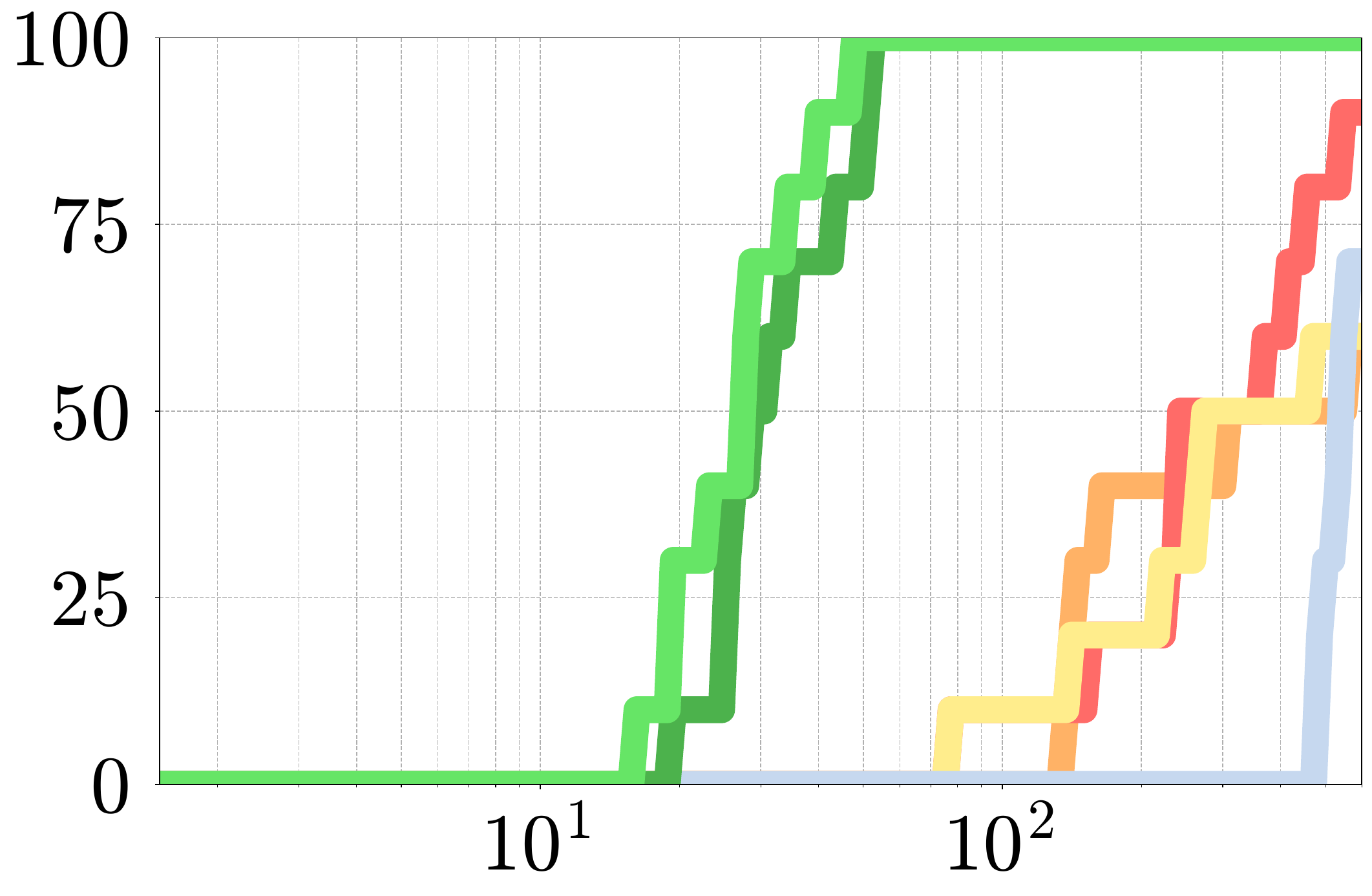}
        \caption{Reeds Shepp Simple}
    \end{subfigure}

    \vspace{0.5em}

    \begin{subfigure}[b]{0.24\textwidth}
        \centering
        \includegraphics[width=\linewidth,height=\subfigheight]{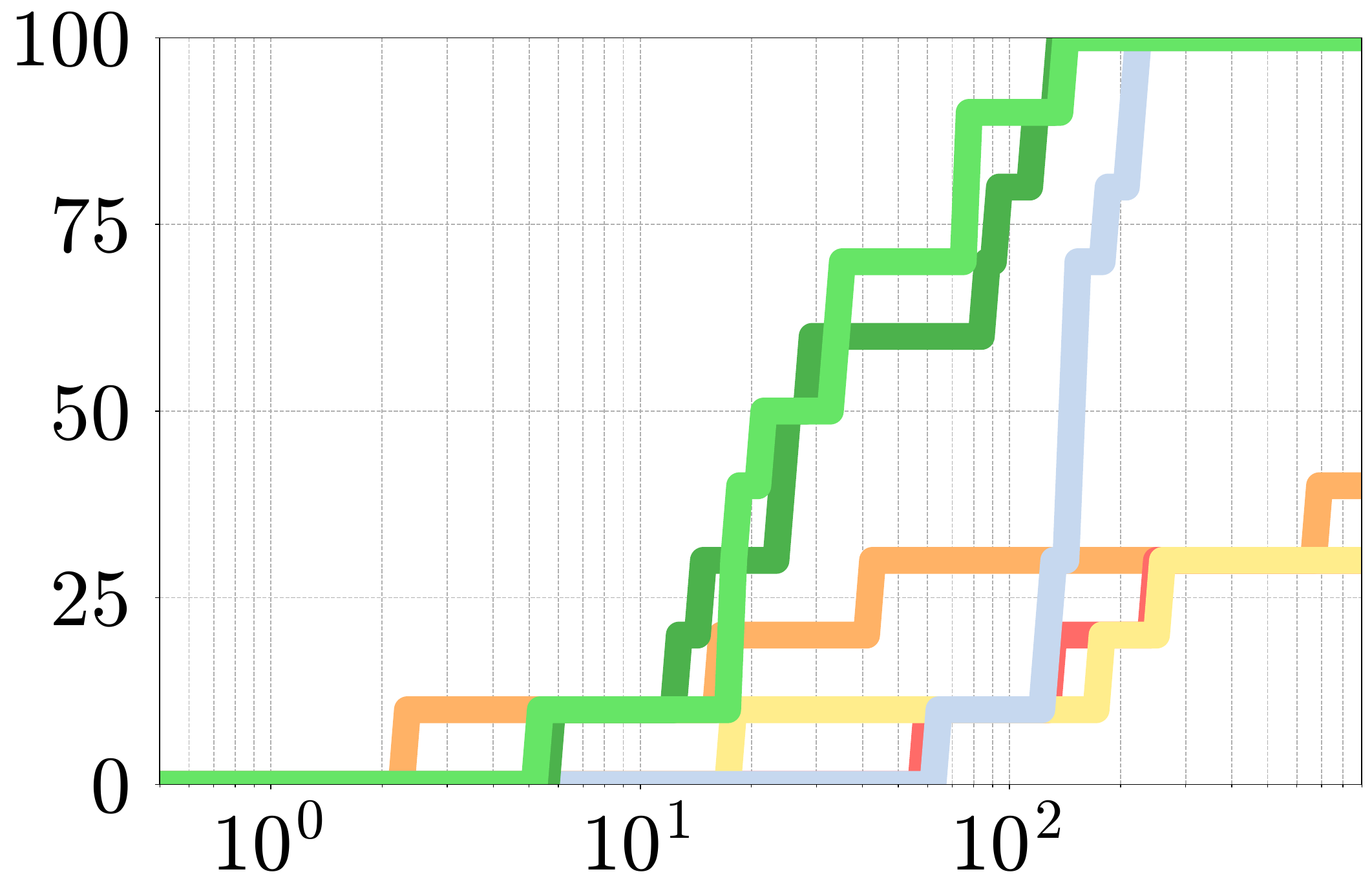}
        \caption{Cube Robots}
    \end{subfigure}
    \begin{subfigure}[b]{0.24\textwidth}
        \centering
        \includegraphics[width=\linewidth,height=\subfigheight]{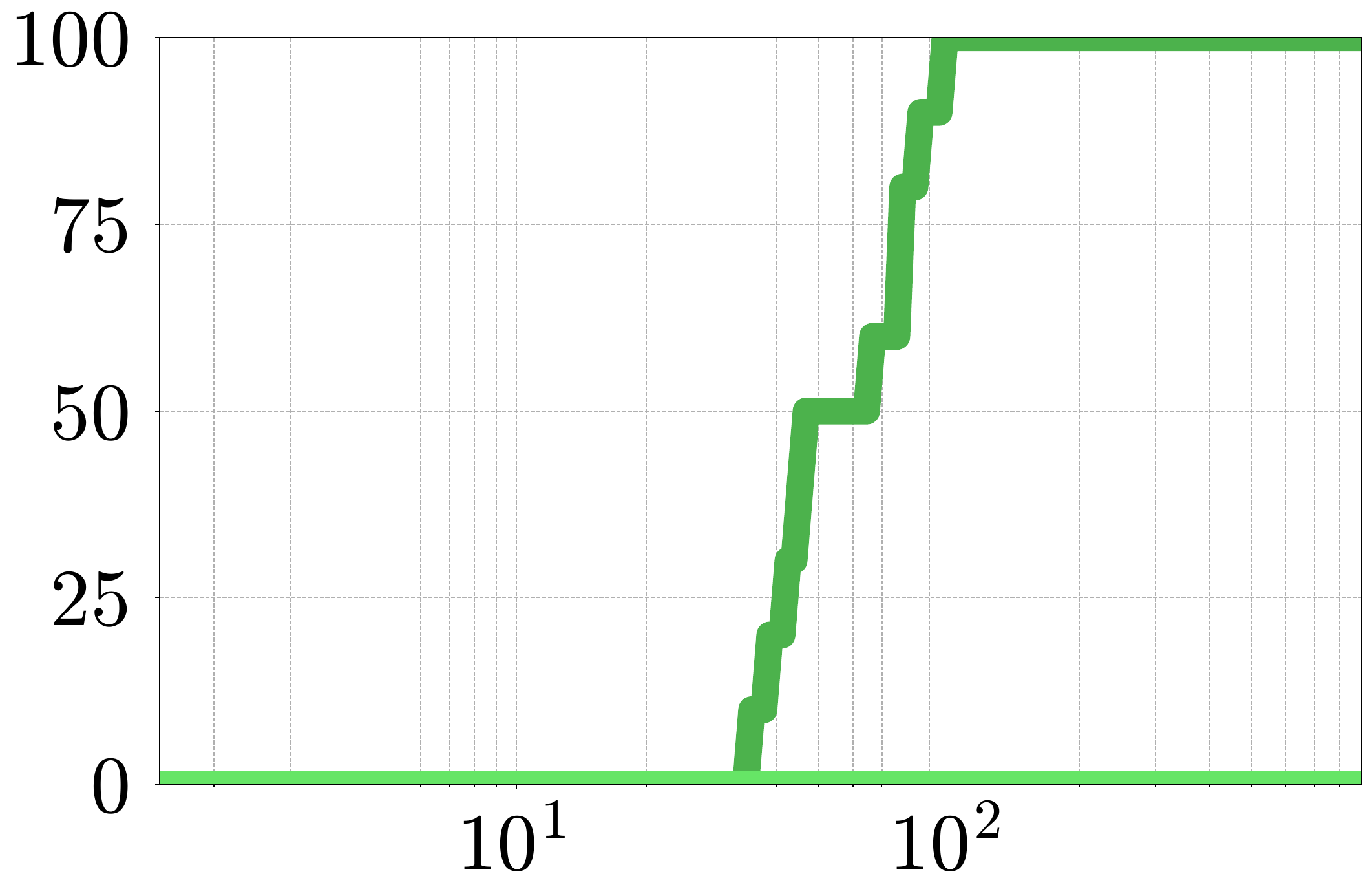}
        \caption{Drones In Pipe}
    \end{subfigure}
    \begin{subfigure}[b]{0.24\textwidth}
        \centering
        \includegraphics[width=\linewidth,height=\subfigheight]{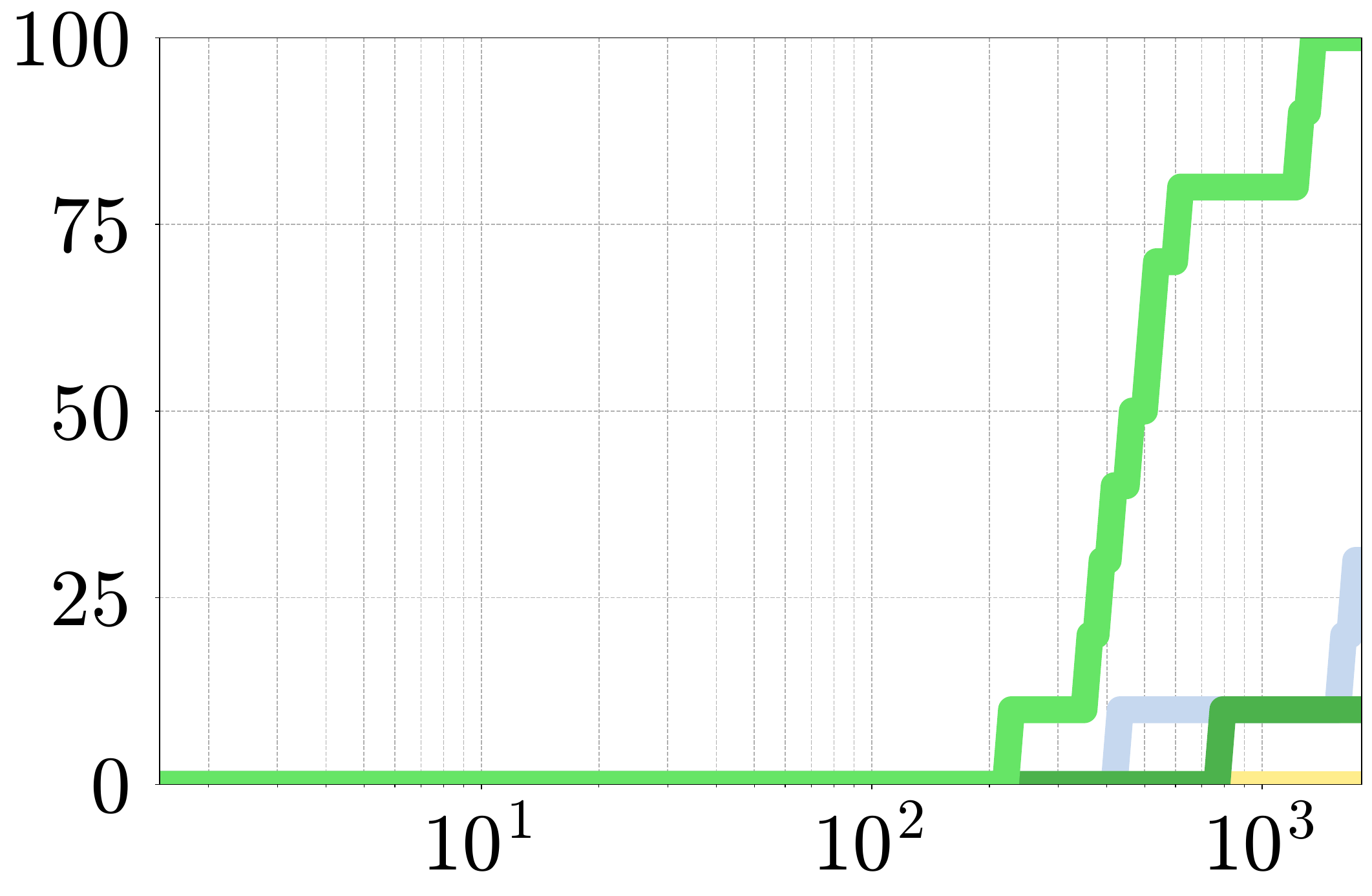}
        \caption{Mobile Robots Forest}
    \end{subfigure}
    \begin{subfigure}[b]{0.24\textwidth}
        \centering
        \includegraphics[width=\linewidth,height=\subfigheight]{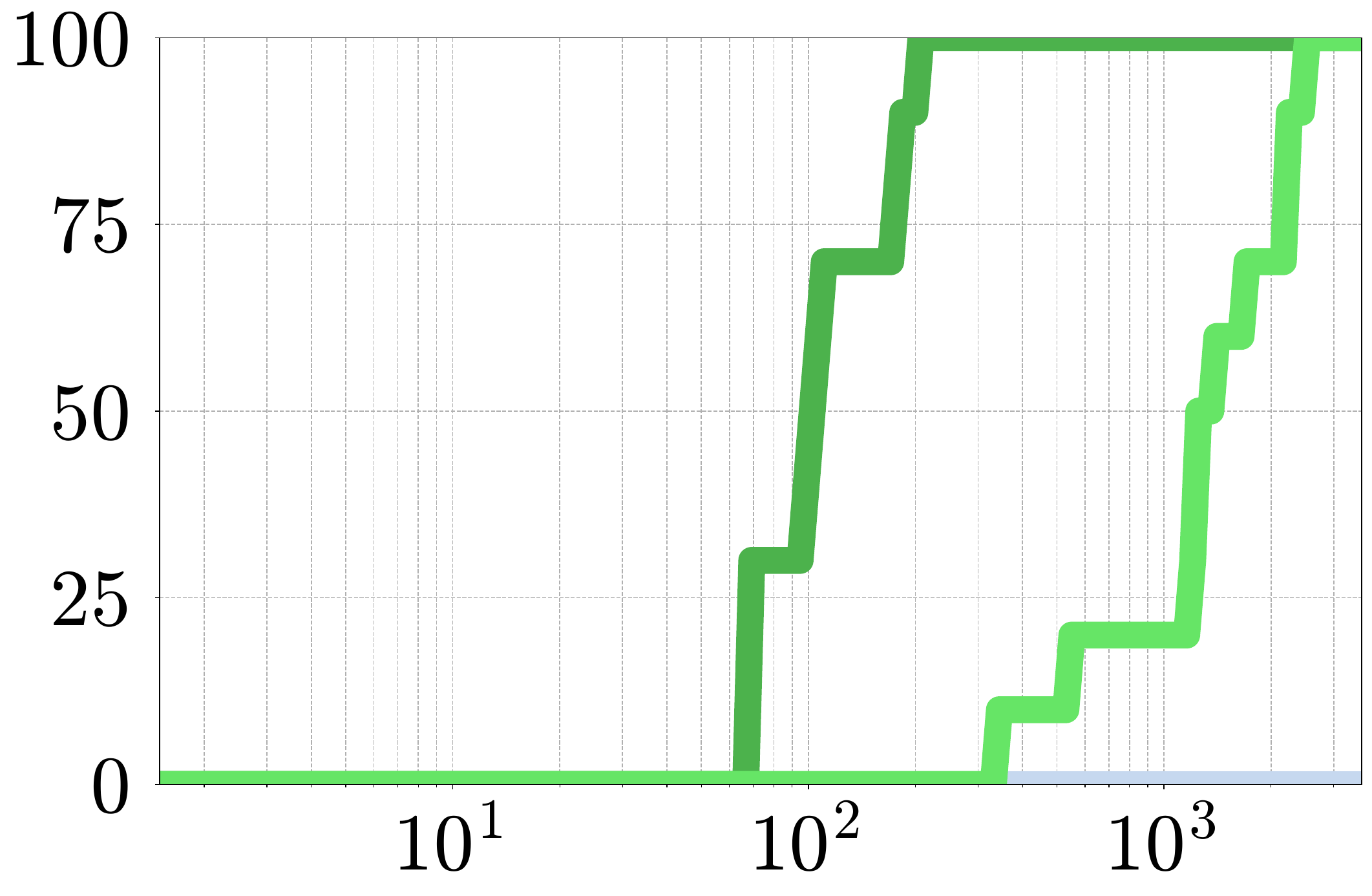}
        \caption{Warehouse}
    \end{subfigure}

    \vspace{0.5em}
\centering

\fbox{\parbox{\dimexpr\linewidth-4\fboxsep-2\fboxrule}{\centering\small
\vspace{0.15cm}
\sqbox{colorFTD} \textbf{\FibrationRRT-Decomposition [Ours]} \quad
\sqbox{colorFTP} \textbf{\FibrationRRT-Prioritization [Ours]}  \quad
\sqbox{colorRRT} RRT \quad
\sqbox{colorLBT} LBTRRT \quad
\sqbox{colorEST} EST \quad
\sqbox{colorFMT} FMT
\vspace{0.15cm}
}}
    \caption{Success graphs of the benchmarks for scenarios from Fig.~\ref{fig:scenarios}. The $x$-axis shows time in log-space, while the $y$-axis shows the success rate from $0$ to $100$ percent. Colors indicate the planner as shown above.}
    \label{fig:benchmark_structure}
\end{figure*}

\setlength{\subfigheight}{0.3\linewidth}

\begin{figure*}[htbp]
    \centering
    \begin{subfigure}[b]{0.24\textwidth}
        \centering
        \includegraphics[width=\linewidth,height=\subfigheight]{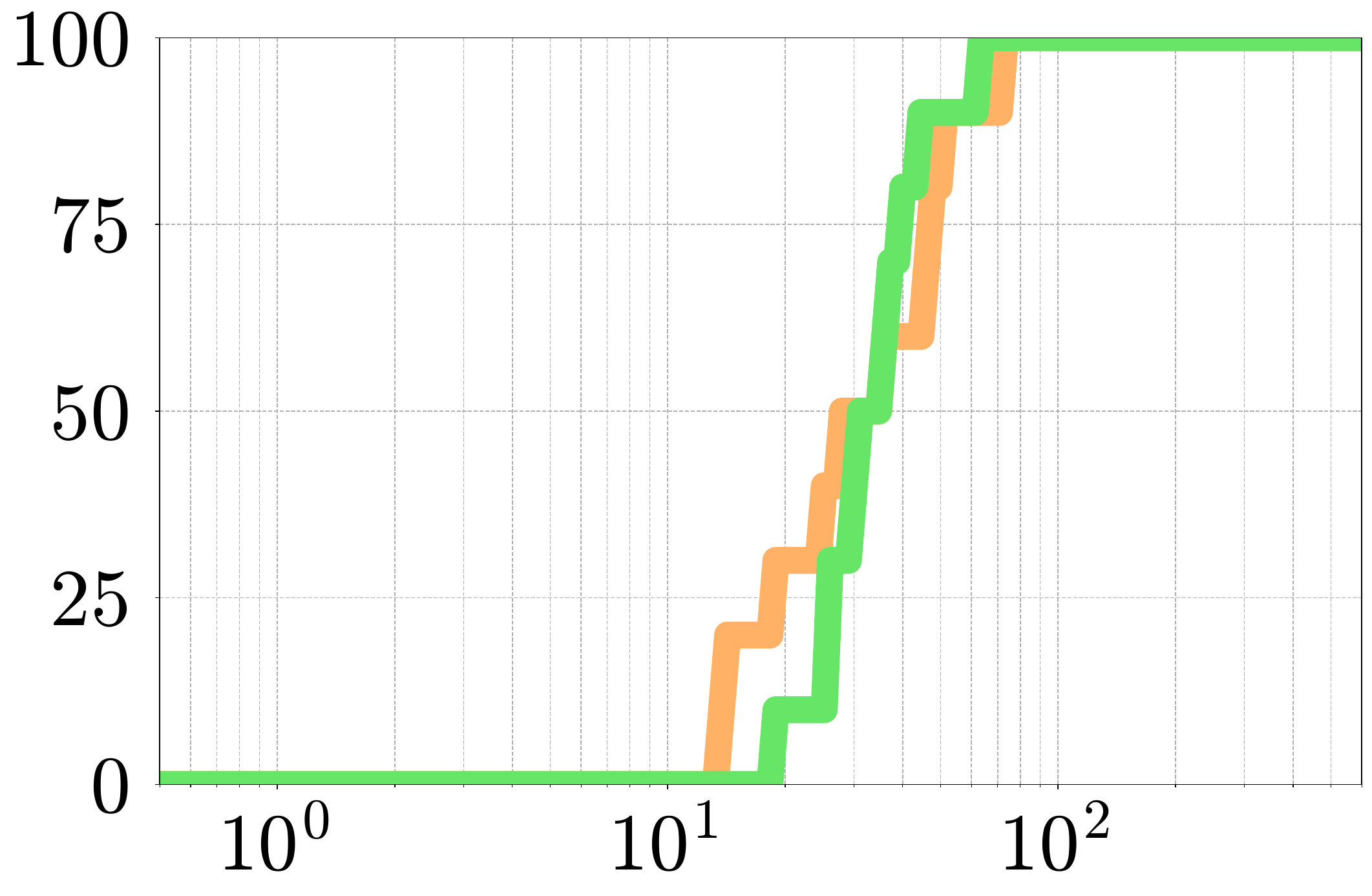}
        \caption{Multi Disks}
    \end{subfigure}
    \begin{subfigure}[b]{0.24\textwidth}
        \centering
        \includegraphics[width=\linewidth,height=\subfigheight]{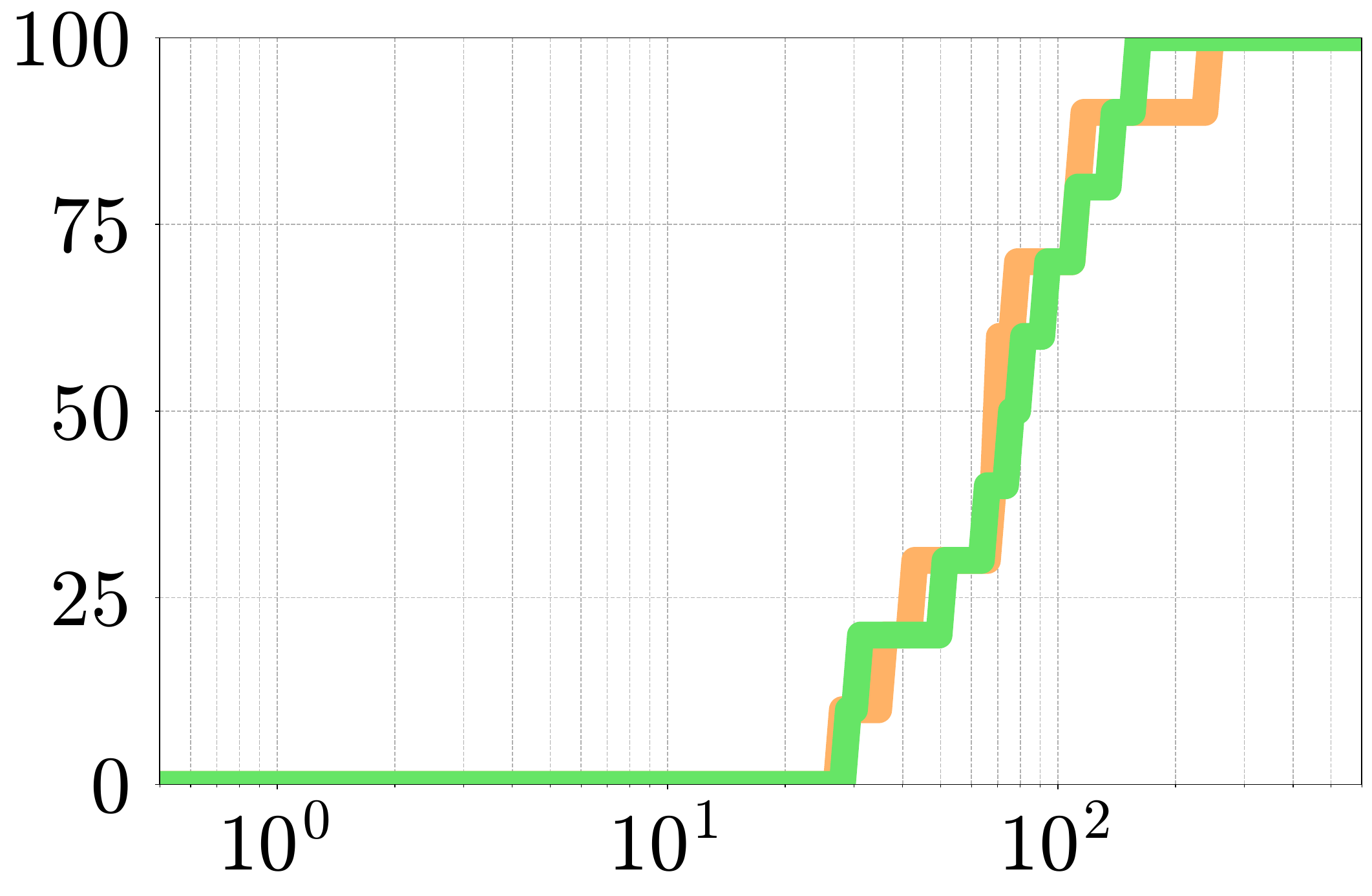}
        \caption{Airship Coordination}
    \end{subfigure}
    \begin{subfigure}[b]{0.24\textwidth}
        \centering
        \includegraphics[width=\linewidth,height=\subfigheight]{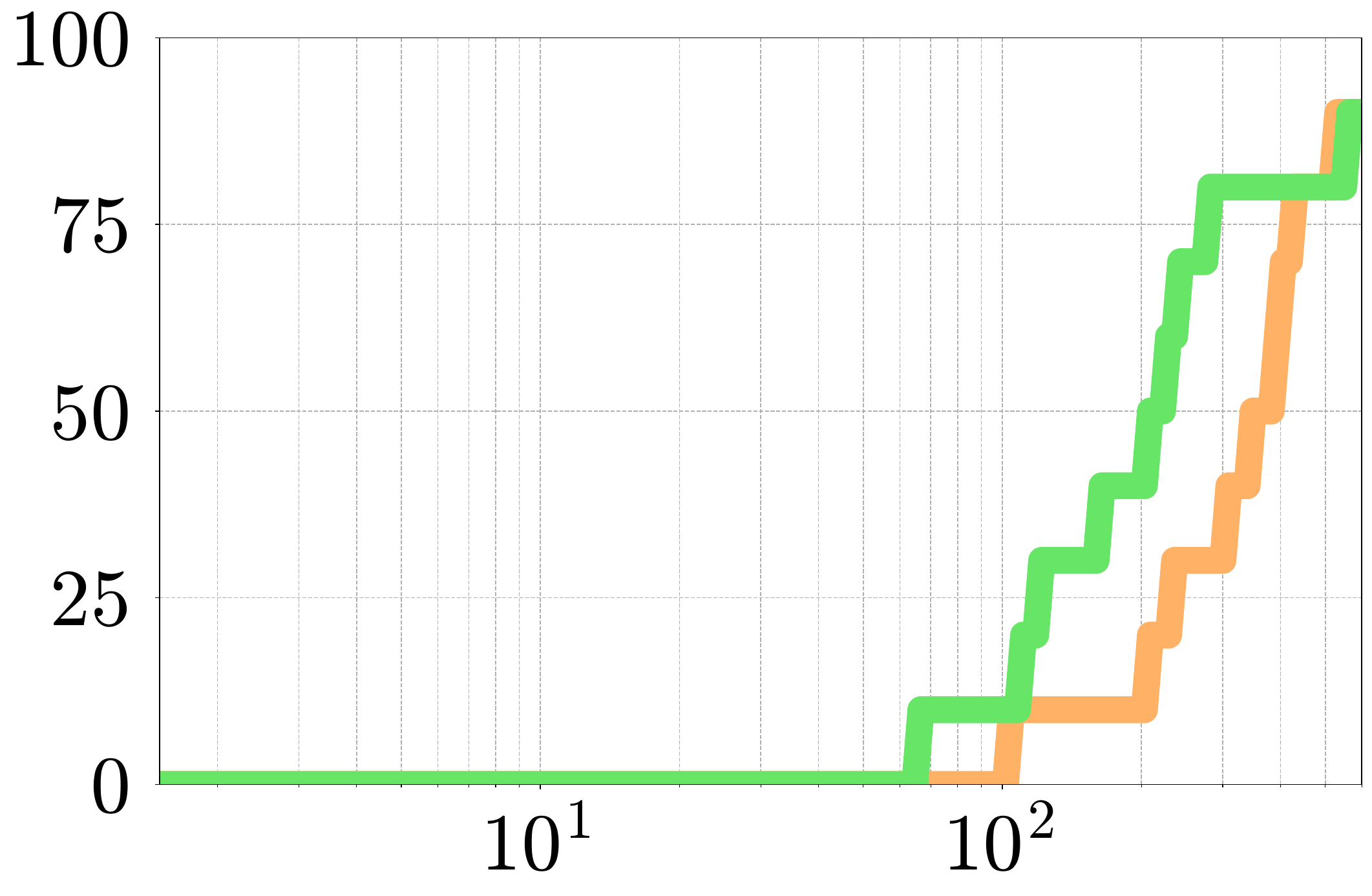}
        \caption{Mobile Navigation}
    \end{subfigure}
    \begin{subfigure}[b]{0.24\textwidth}
        \centering
        \includegraphics[width=\linewidth,height=\subfigheight]{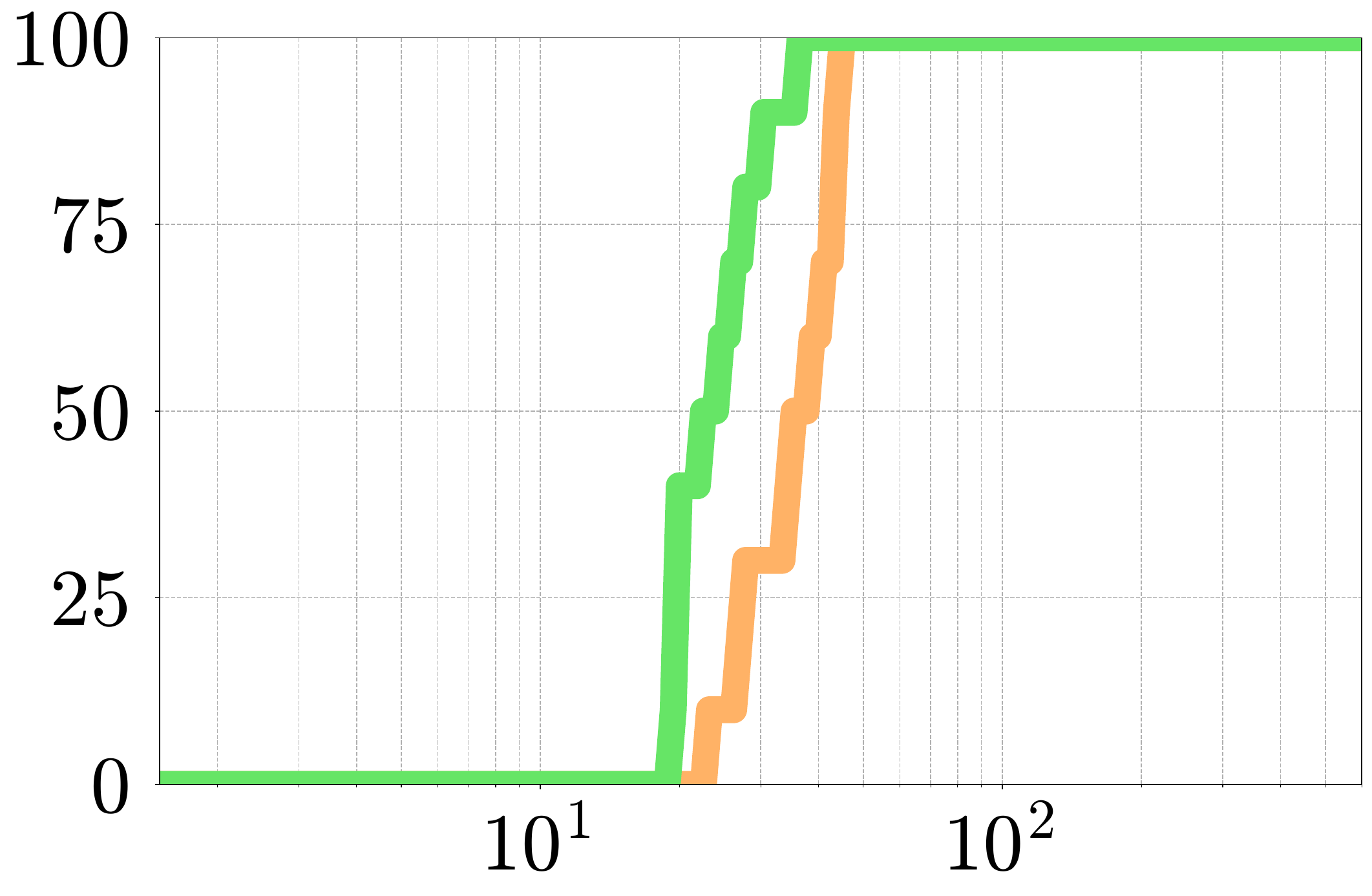}
        \caption{Reeds Shepp Simple}
    \end{subfigure}

    \vspace{0.5em}

    \begin{subfigure}[b]{0.24\textwidth}
        \centering
        \includegraphics[width=\linewidth,height=\subfigheight]{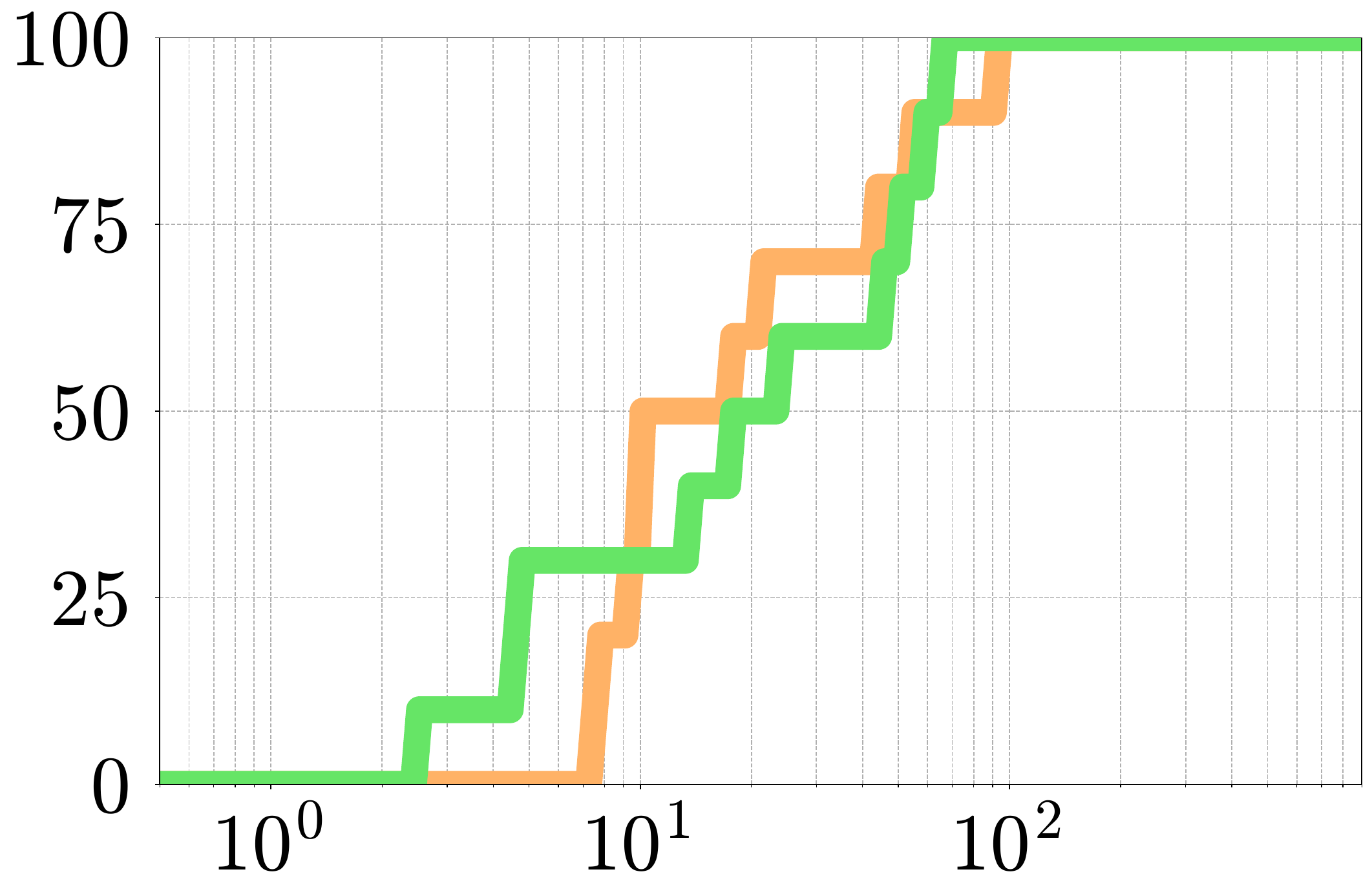}
        \caption{Cube Robots}
    \end{subfigure}
    \begin{subfigure}[b]{0.24\textwidth}
        \centering
        \includegraphics[width=\linewidth,height=\subfigheight]{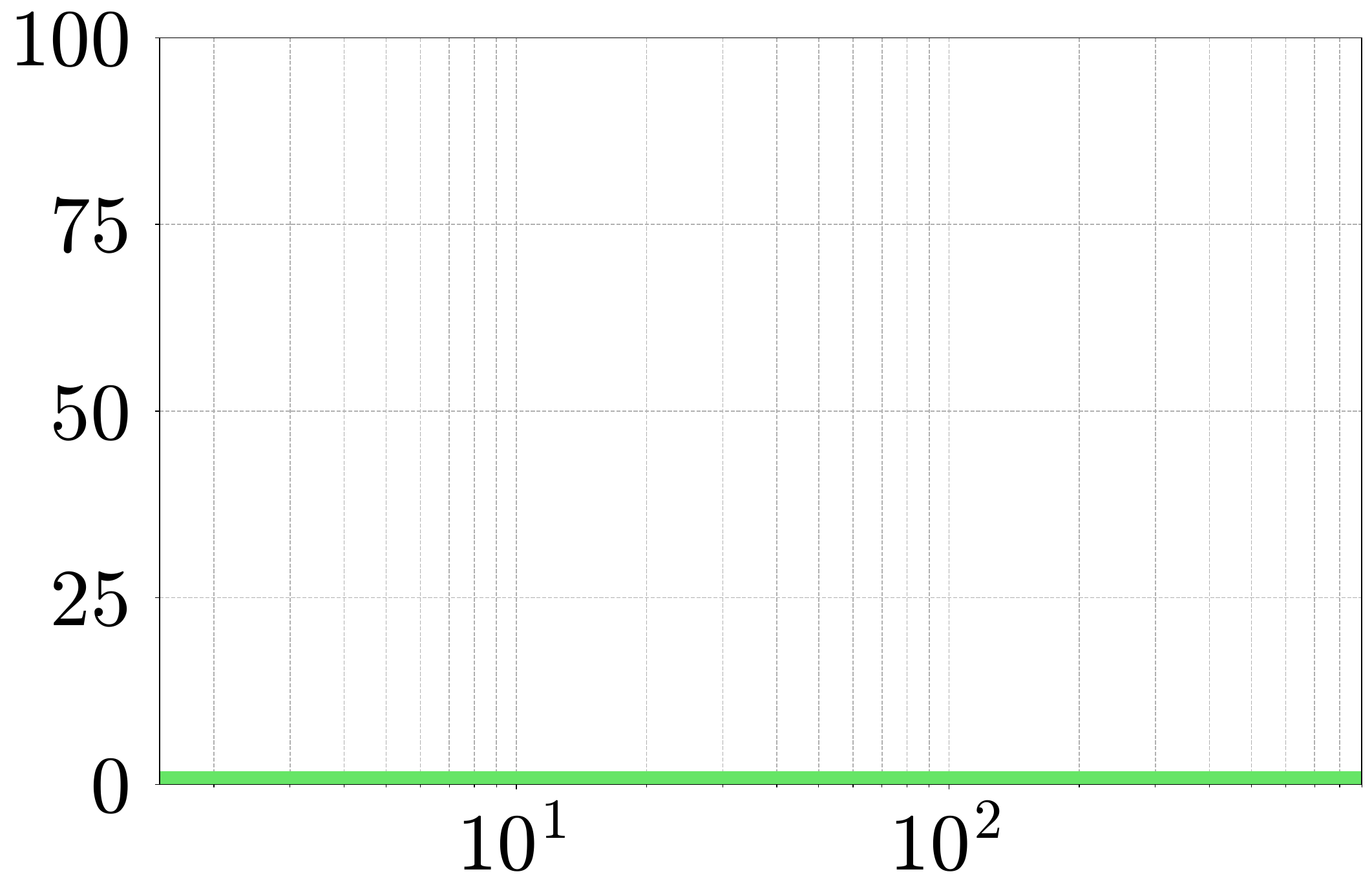}
        \caption{Drones In Pipe}
    \end{subfigure}
    \begin{subfigure}[b]{0.24\textwidth}
        \centering
        \includegraphics[width=\linewidth,height=\subfigheight]{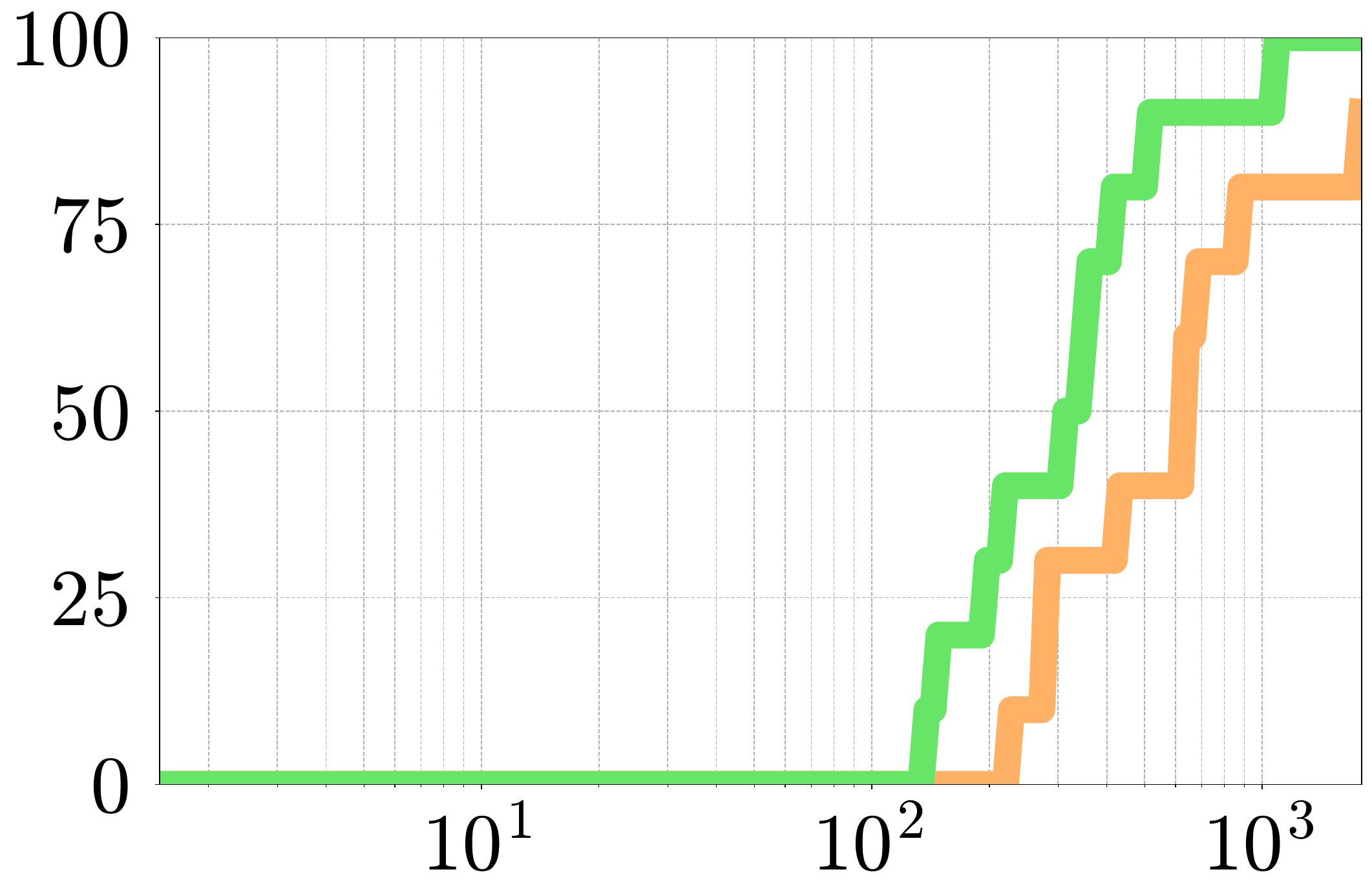}
        \caption{Mobile Robots Forest}
    \end{subfigure}
    \begin{subfigure}[b]{0.24\textwidth}
        \centering
        \includegraphics[width=\linewidth,height=\subfigheight]{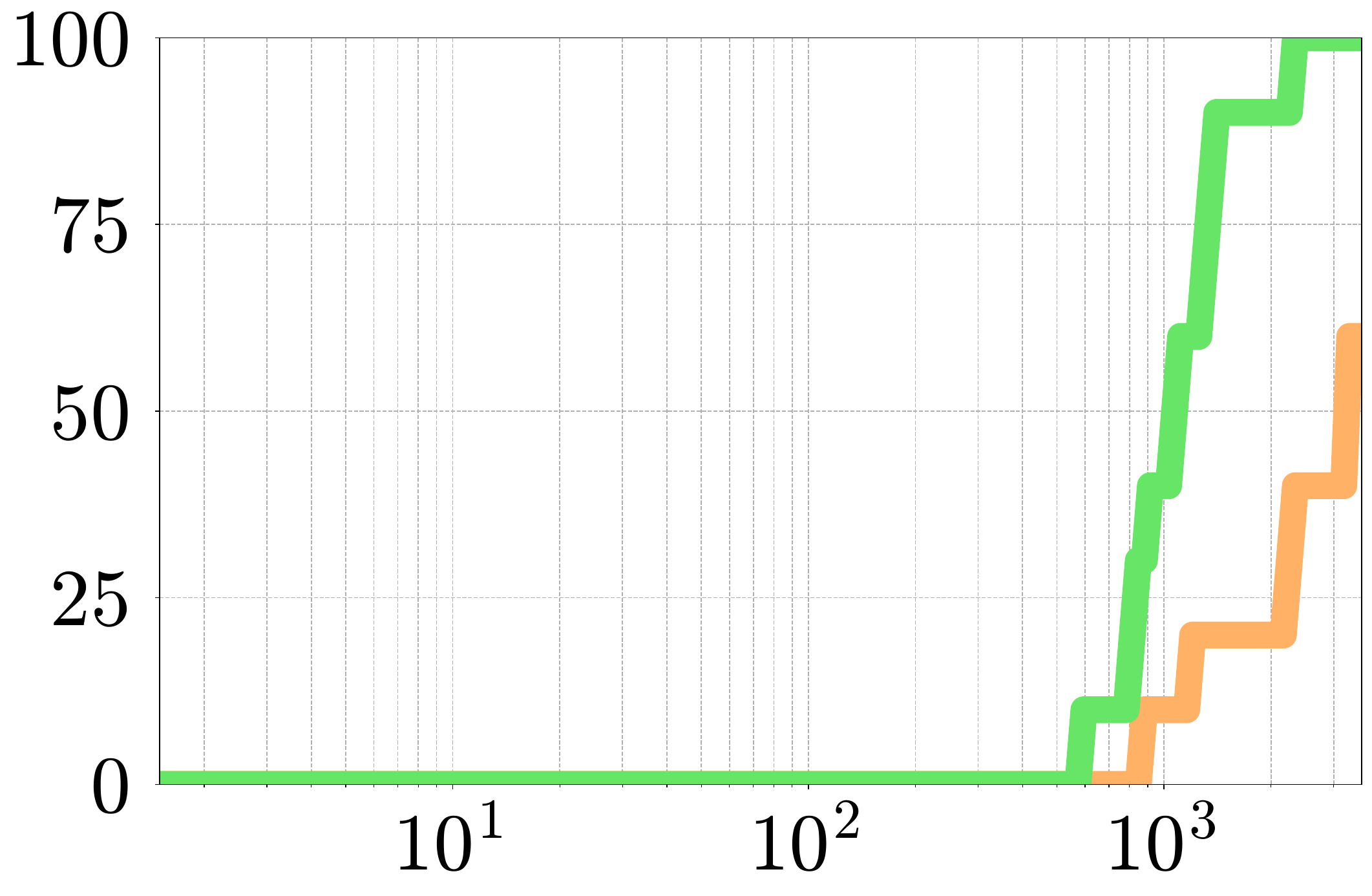}
        \caption{Warehouse}
    \end{subfigure}

    \vspace{0.5em}
\centering

\fbox{\parbox{\dimexpr\linewidth-4\fboxsep-2\fboxrule}{\centering\small
\vspace{0.15cm}
\sqbox{colorFTP} \textbf{\FibrationRRT-Prioritization [Ours]}  \quad
\sqbox{colorQRRT} QRRT
\vspace{0.15cm}
}}
    \caption{Success graphs for the prioritization benchmarks.}
    \label{fig:benchmark_prioritization}
\end{figure*}

\setlength{\subfigheight}{0.3\linewidth}

\begin{figure*}[htbp]
    \centering
    \begin{subfigure}[b]{0.24\textwidth}
        \centering
        \includegraphics[width=\linewidth,height=\subfigheight]{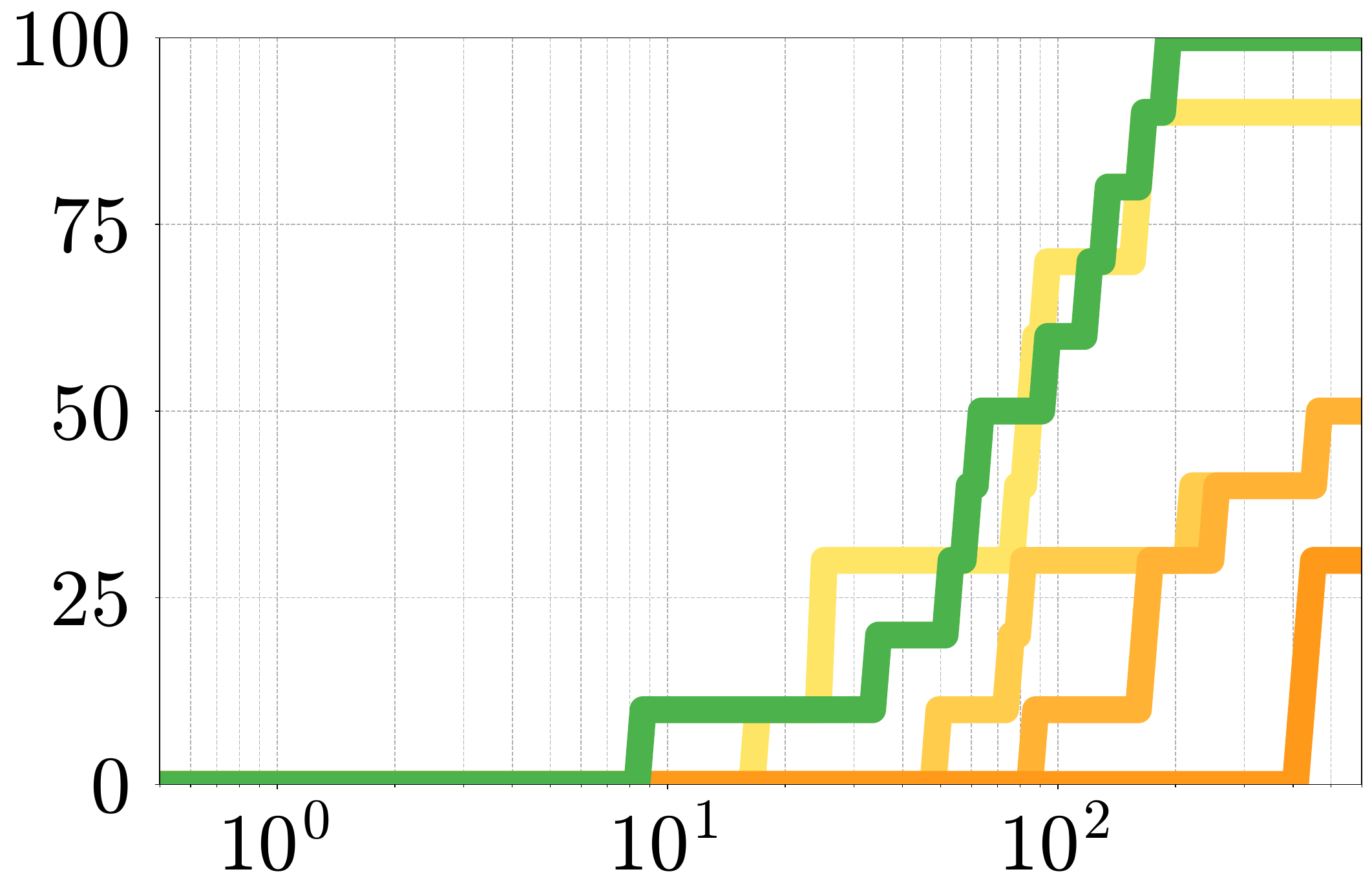}
        \caption{Multi Disks}
    \end{subfigure}
    \begin{subfigure}[b]{0.24\textwidth}
        \centering
        \includegraphics[width=\linewidth,height=\subfigheight]{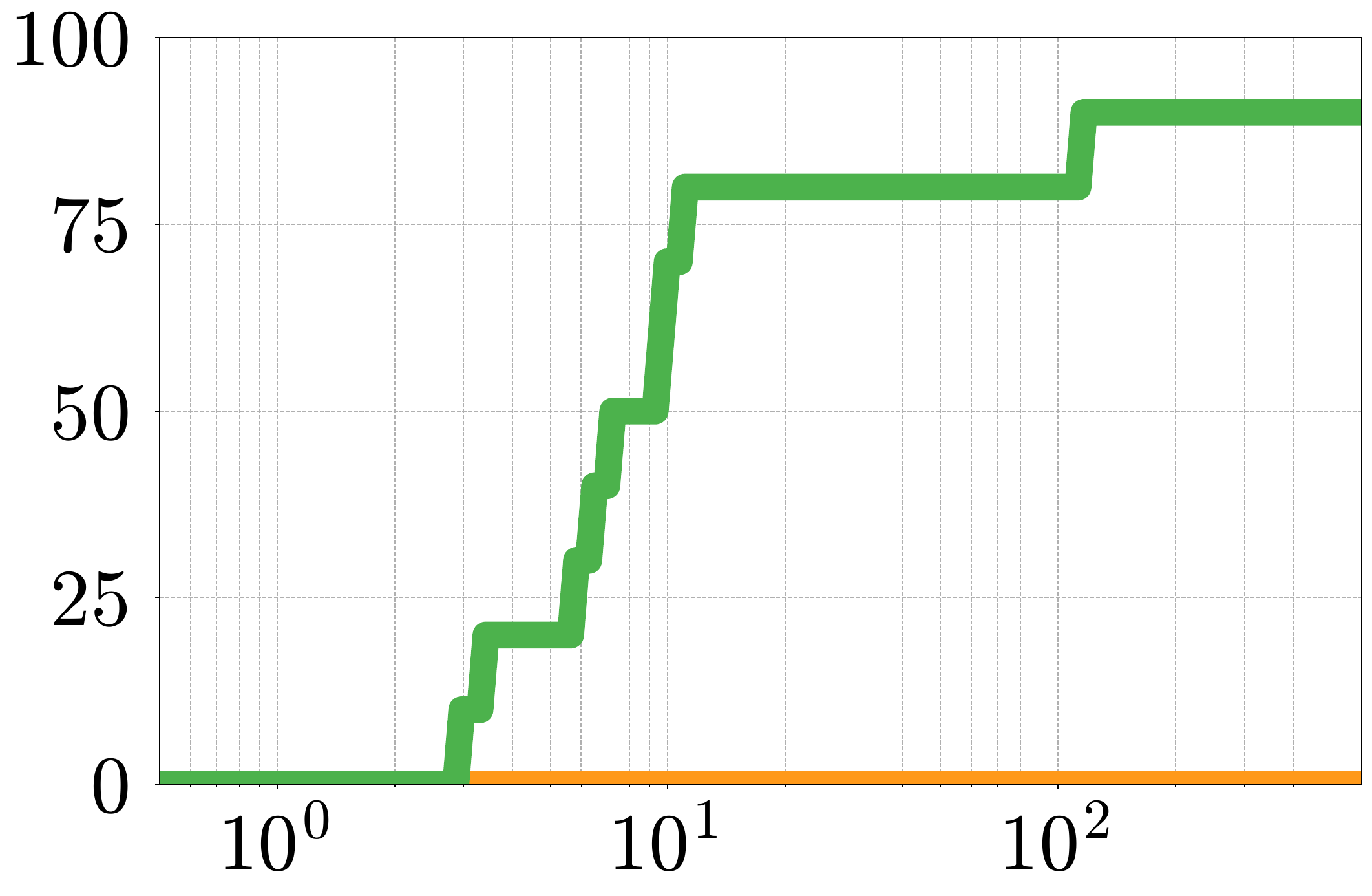}
        \caption{Airship Coordination}
    \end{subfigure}
    \begin{subfigure}[b]{0.24\textwidth}
        \centering
        \includegraphics[width=\linewidth,height=\subfigheight]{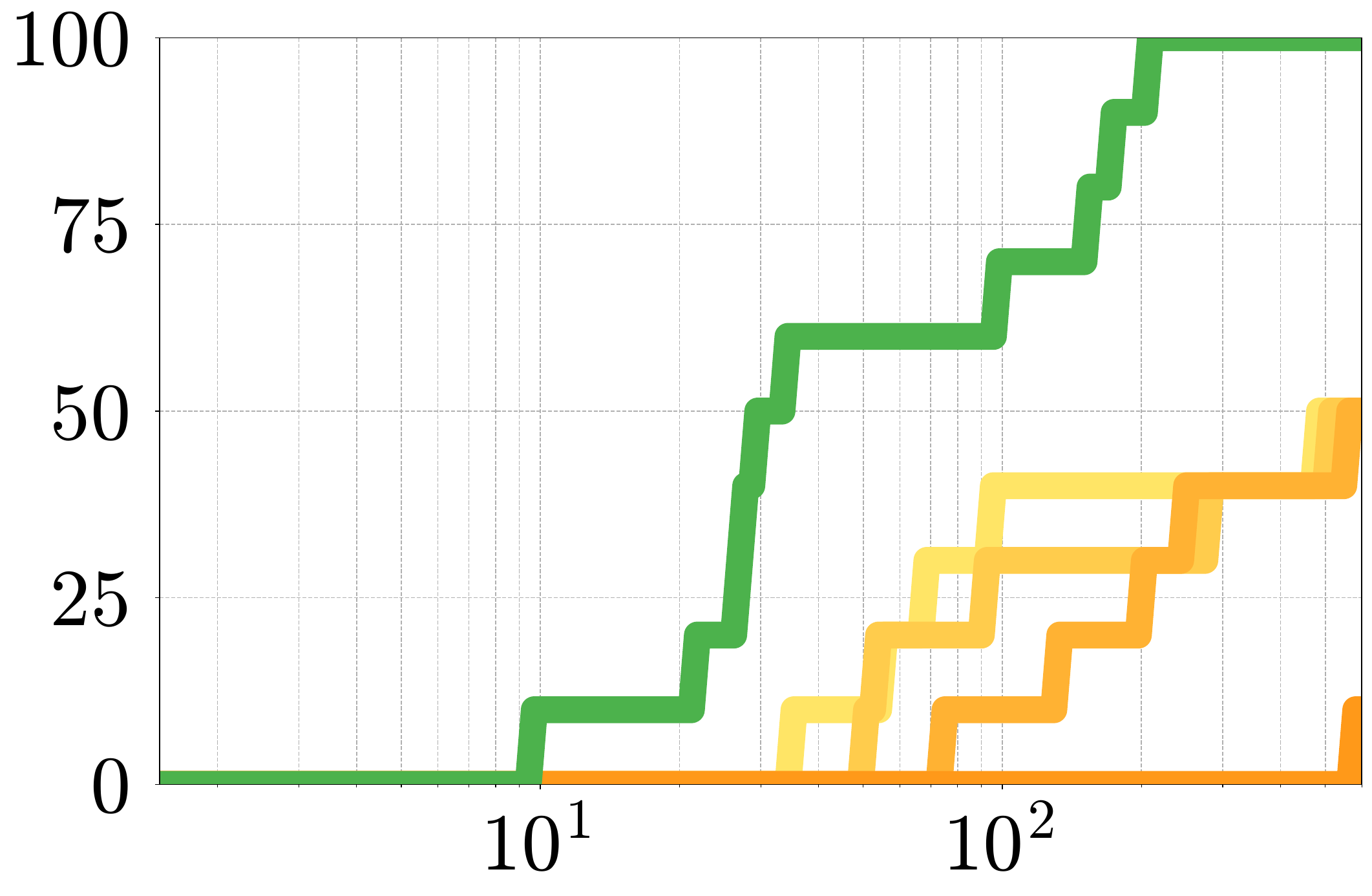}
        \caption{Mobile Navigation}
    \end{subfigure}
    \begin{subfigure}[b]{0.24\textwidth}
        \centering
        \includegraphics[width=\linewidth,height=\subfigheight]{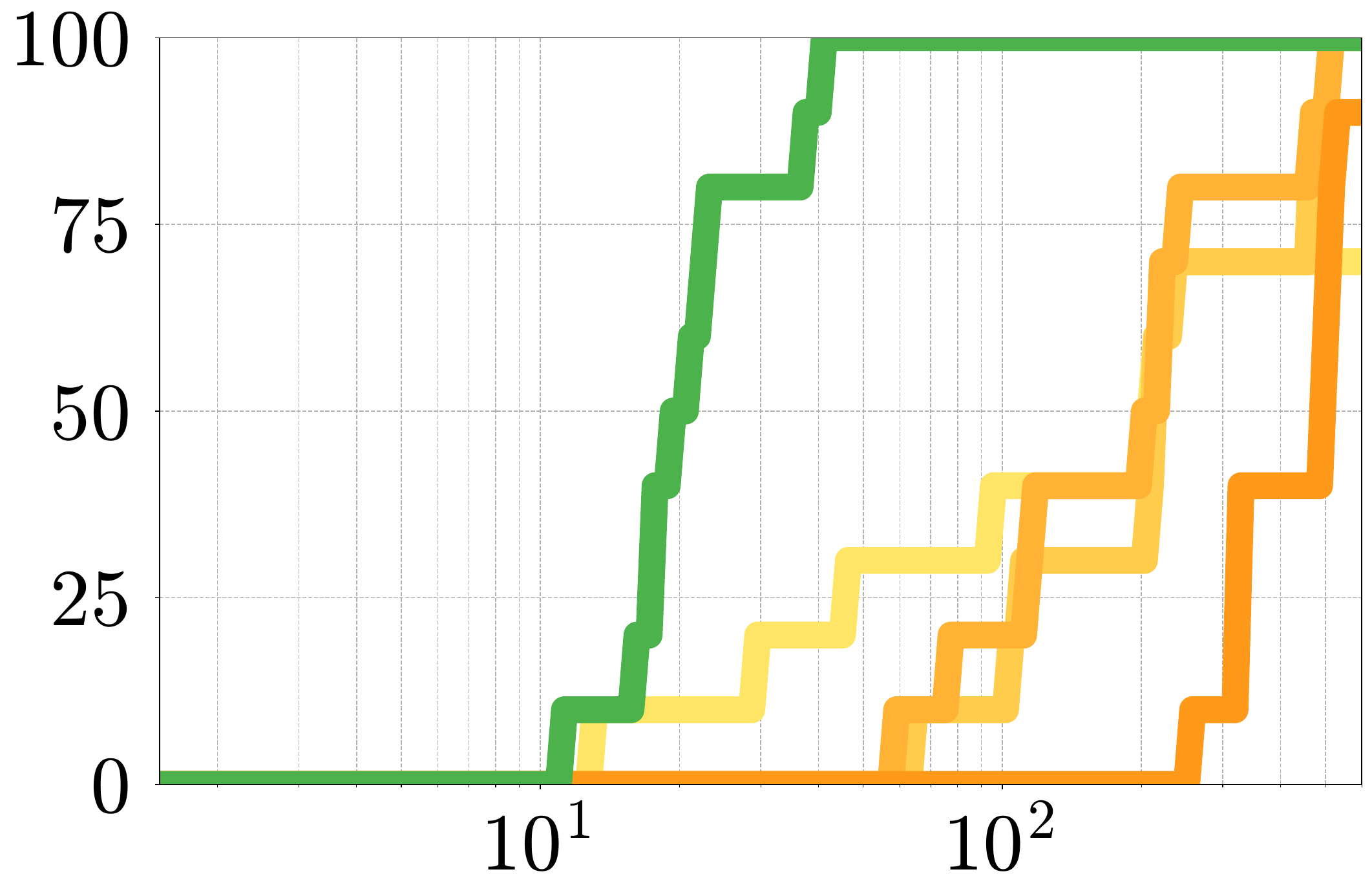}
        \caption{Reeds Shepp Simple}
    \end{subfigure}

    \vspace{0.5em}

    \begin{subfigure}[b]{0.24\textwidth}
        \centering
        \includegraphics[width=\linewidth,height=\subfigheight]{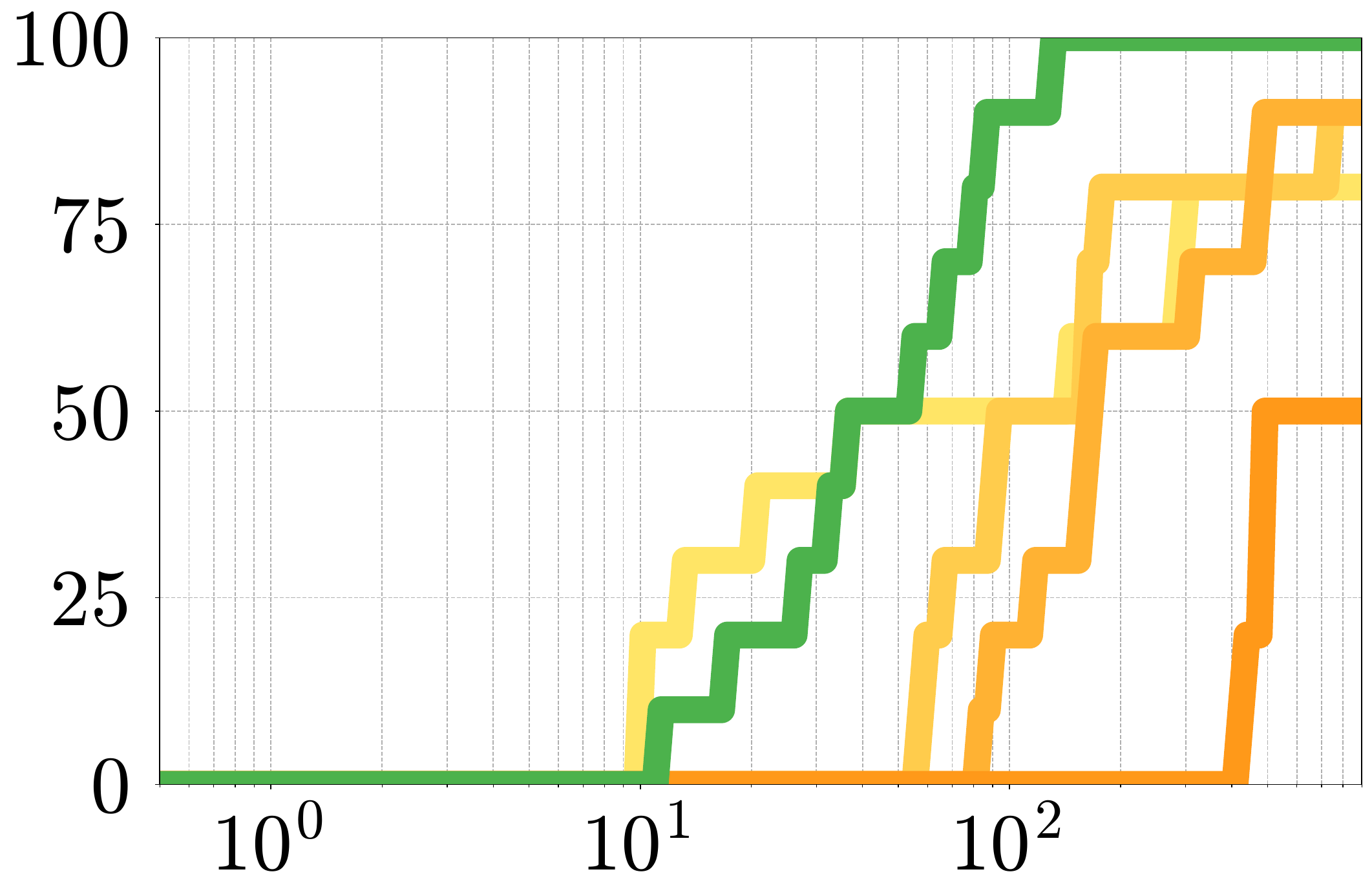}
        \caption{Cube Robots}
    \end{subfigure}
    \begin{subfigure}[b]{0.24\textwidth}
        \centering
        \includegraphics[width=\linewidth,height=\subfigheight]{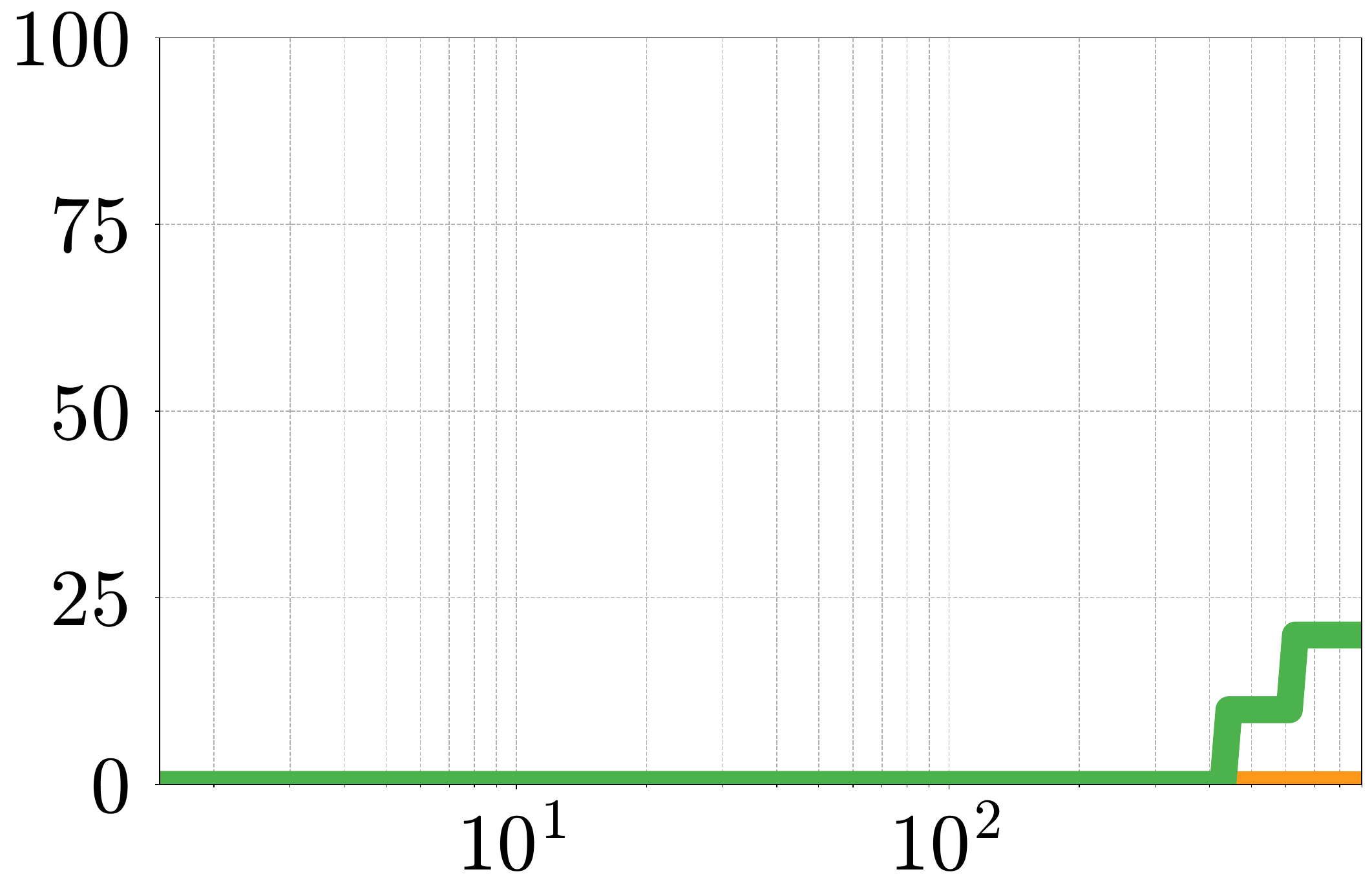}
        \caption{Drones In Pipe}
    \end{subfigure}
    \begin{subfigure}[b]{0.24\textwidth}
        \centering
        \includegraphics[width=\linewidth,height=\subfigheight]{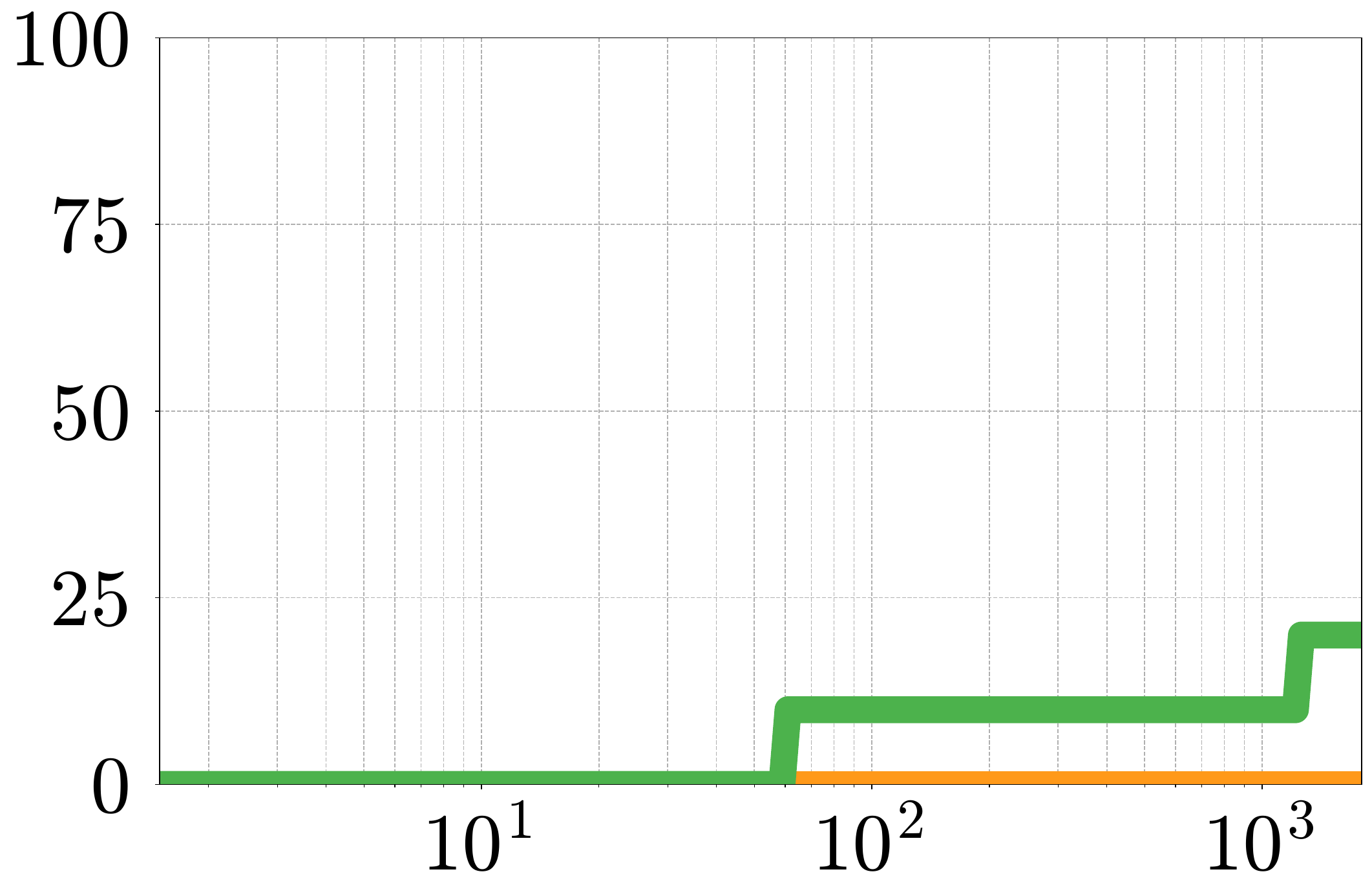}
        \caption{Mobile Robots Forest}
    \end{subfigure}
    \begin{subfigure}[b]{0.24\textwidth}
        \centering
        \includegraphics[width=\linewidth,height=\subfigheight]{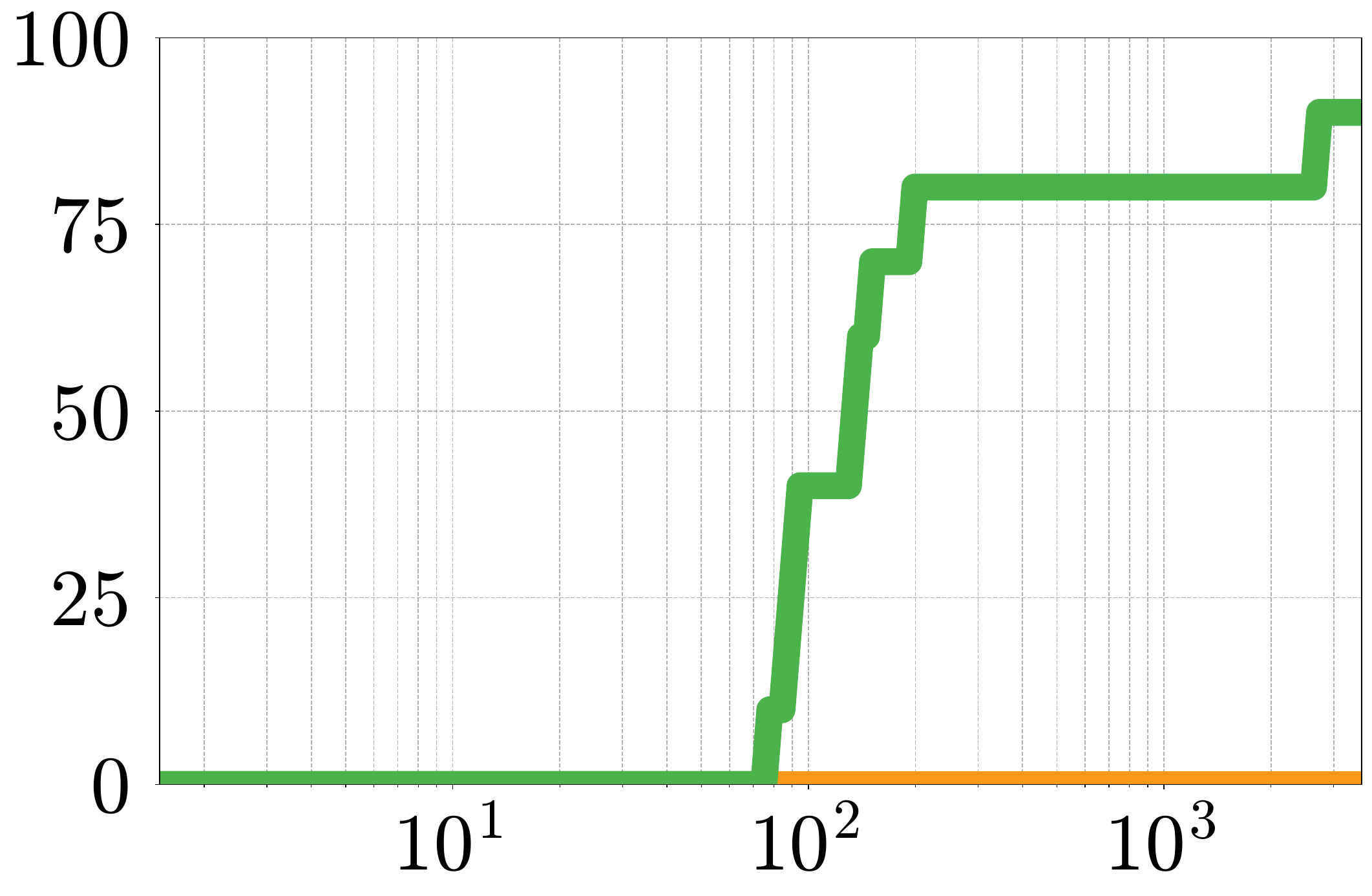}
        \caption{Warehouse}
    \end{subfigure}

    \vspace{0.5em}
\centering

\fbox{\parbox{\dimexpr\linewidth-4\fboxsep-2\fboxrule}{\centering\small
\vspace{0.15cm}
\sqbox{colorFTD} \textbf{\FibrationRRT-Decomposition [Ours]}  \quad
\sqbox{colorDiscreteRRT1} dRRT-1 \quad
\sqbox{colorDiscreteRRT5} dRRT-5 \quad
\sqbox{colorDiscreteRRT10} dRRT-10 \quad
\sqbox{colorDiscreteRRT50} dRRT-50
\vspace{0.15cm}
}}
    \caption{Success graphs for the decomposition benchmarks.}
    \label{fig:benchmark_decomposition}
\end{figure*}

\setlength{\subfigheight}{0.3\linewidth}
\setlength{\scenarioheight}{0.4\linewidth}

\def\vdistance{0.7em}

\begin{figure*}[htbp]
 \centering
    \begin{subfigure}[b]{0.24\textwidth}
        \centering
        \includegraphics[width=\linewidth,height=\subfigheight]{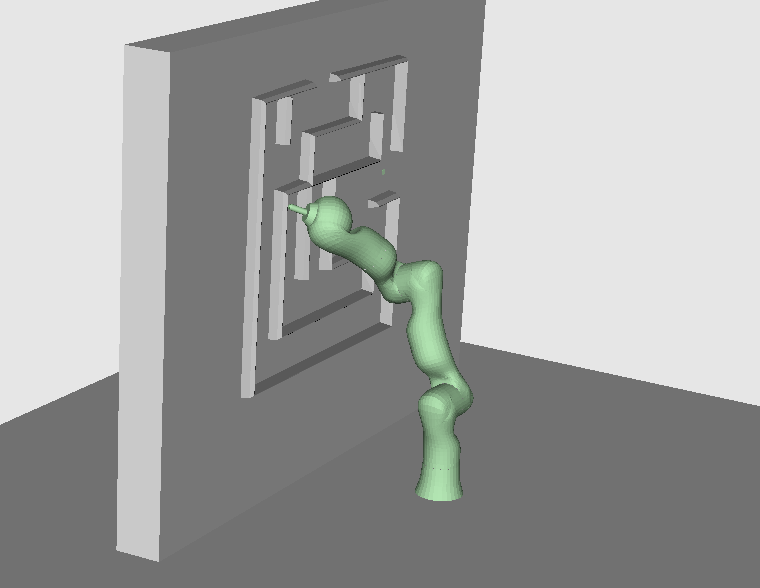}

        \vspace{\vdistance}
        \includegraphics[width=\linewidth,height=\subfigheight]{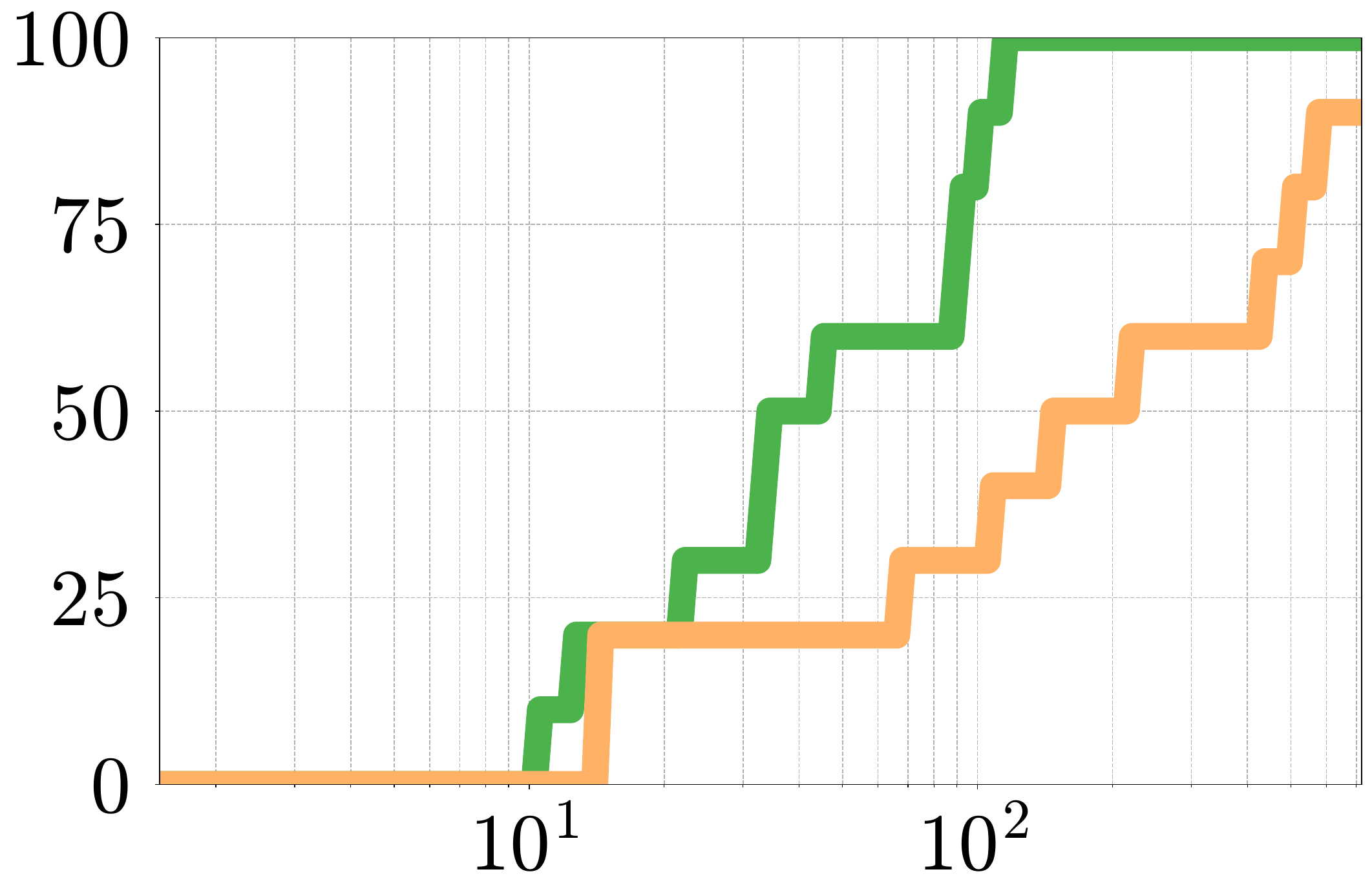}
        \caption{Vertical Maze}
    \end{subfigure}
    \begin{subfigure}[b]{0.24\textwidth}
        \centering
        \includegraphics[width=\linewidth,height=\subfigheight]{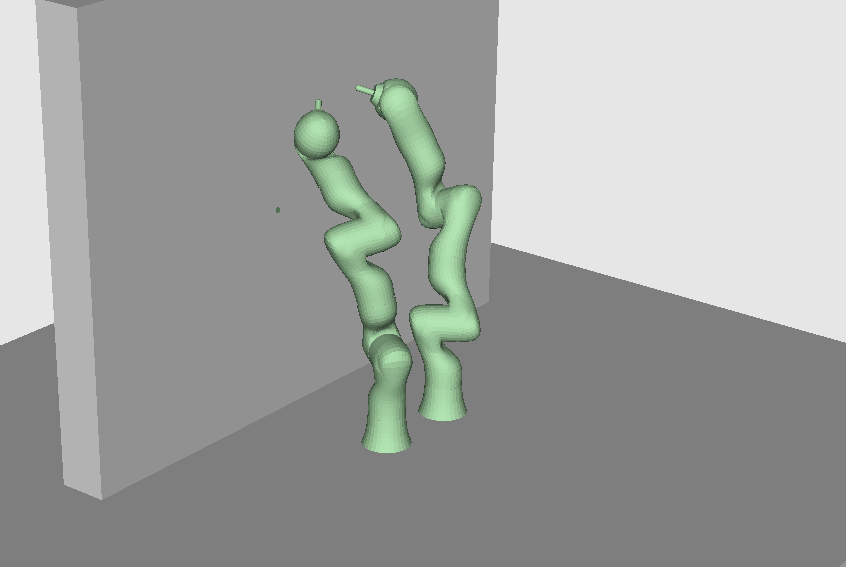}

        \vspace{\vdistance}
        \includegraphics[width=\linewidth,height=\subfigheight]{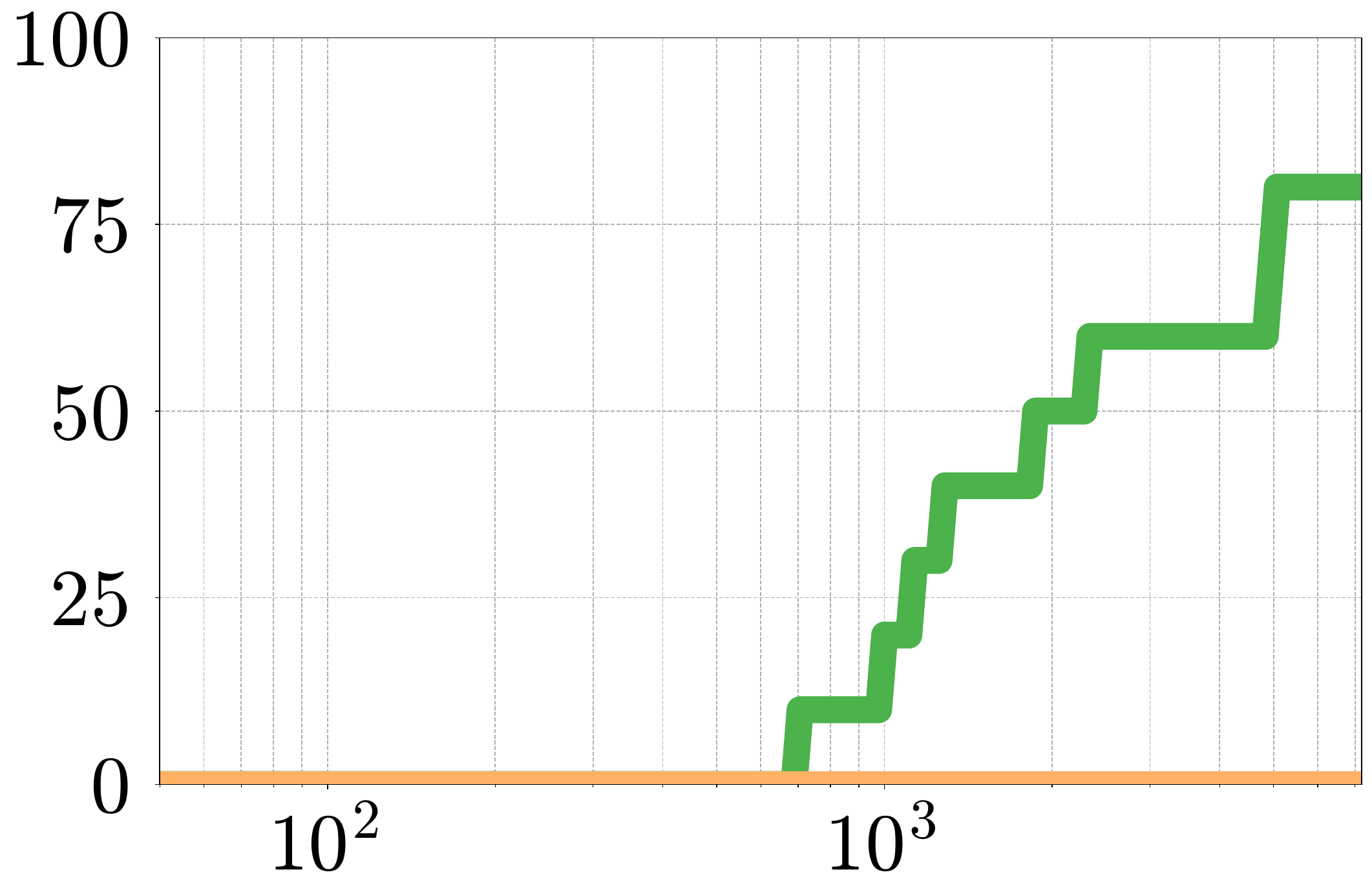}
        \caption{Fixed on Wall}
    \end{subfigure}
    \begin{subfigure}[b]{0.24\textwidth}
        \centering
        \includegraphics[width=\linewidth,height=\subfigheight]{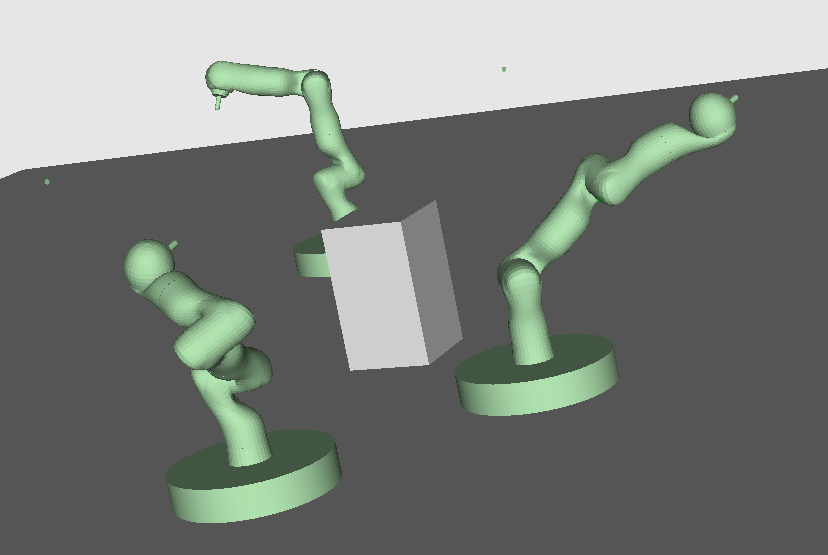}

        \vspace{\vdistance}
        \includegraphics[width=\linewidth,height=\subfigheight]{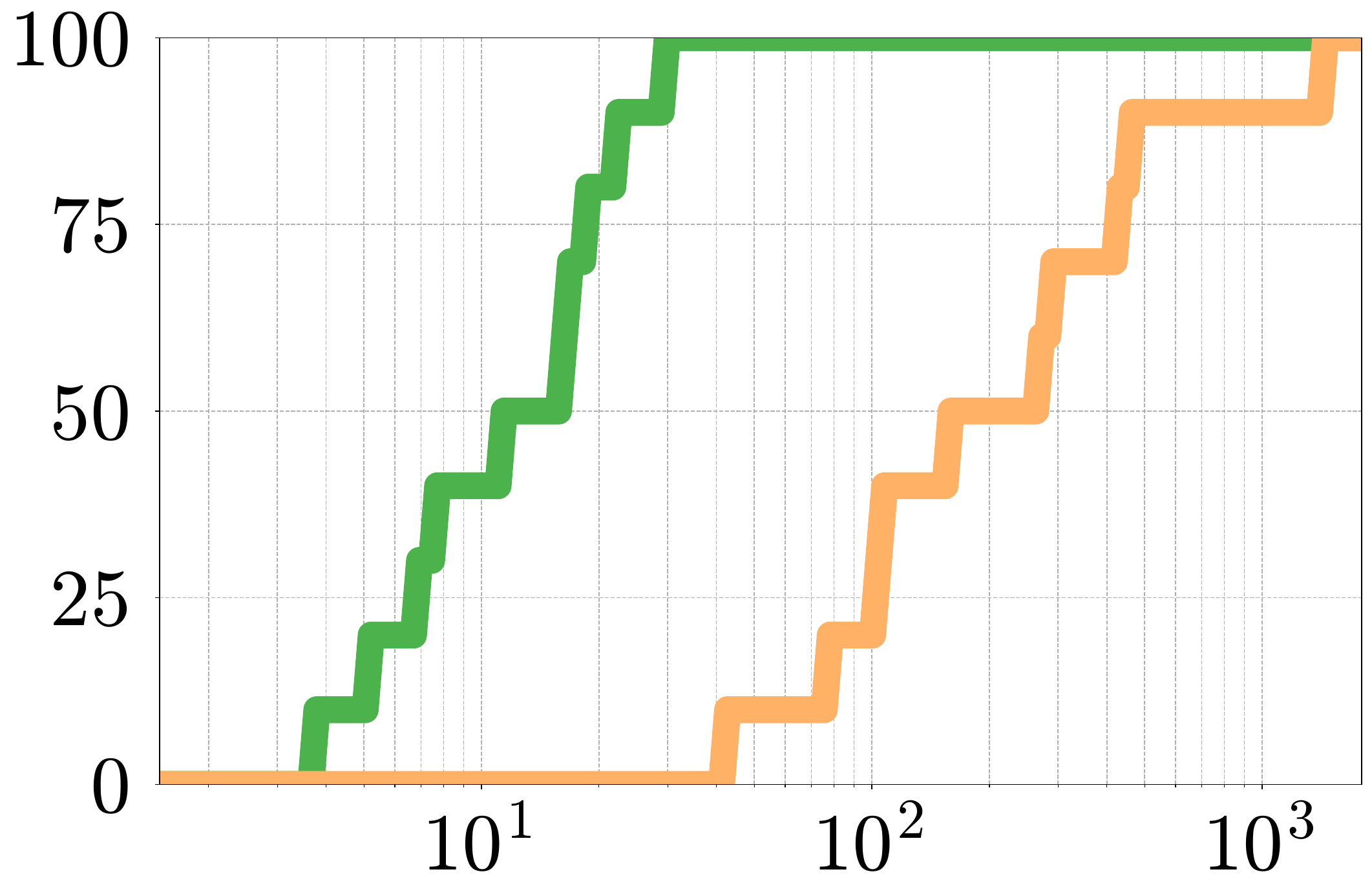}
        \caption{Mobile Manipulators}
    \end{subfigure}
    \begin{subfigure}[b]{0.24\textwidth}
        \centering
        \includegraphics[width=\linewidth,height=\subfigheight]{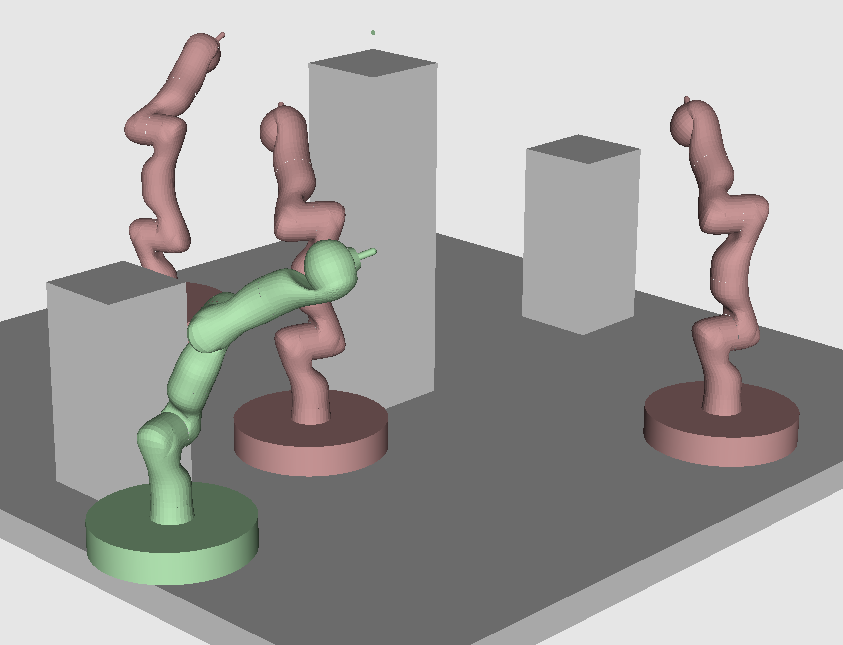}

        \vspace{\vdistance}
        \includegraphics[width=\linewidth,height=\subfigheight]{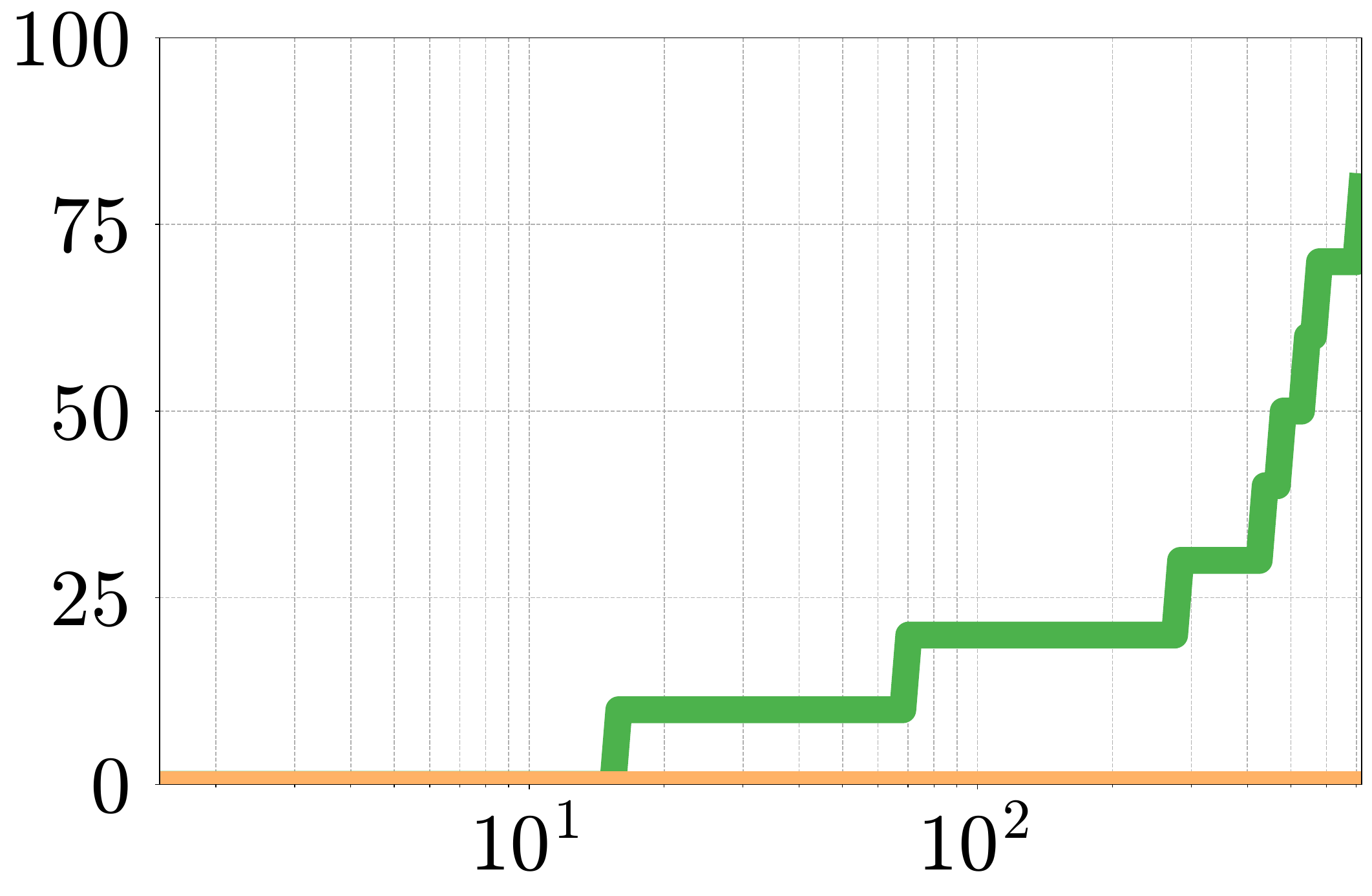}
        \caption{Path Velocity Decomposition}
    \end{subfigure}
    \centering

\fbox{\parbox{\dimexpr\linewidth-4\fboxsep-2\fboxrule}{\centering\small
\vspace{0.15cm}
\sqbox{colorFTD} \textbf{\FibrationRRT [Ours]} \quad
\sqbox{colorRRT} Task-RRT \quad
\vspace{0.15cm}
}}
    \caption{Success graphs of the benchmarks for scenarios from Fig.~\ref{fig:scenarios}. The $x$-axis shows time in log-space, while the $y$-axis shows the success rate from $0$ to $100$ percent. Colors indicate the planner as shown above.}
    \label{fig:benchmark_taskspace}
\end{figure*}
\begin{figure*}[t]
\centering
\captionsetup[subfigure]{justification=centering}

    \begin{subfigure}[b]{0.15\textwidth}
        \vskip 0pt
        \centering
        \begin{tikzpicture}
  \node[root] {$\R^{7}$}
    child[partial] { node[partial] {$\R^{3}$}
    };
\end{tikzpicture}
        \caption{Scenario 1}
    \end{subfigure}

    \begin{subfigure}[b]{0.15\textwidth}
        \vskip 0pt
        \centering
        \begin{tikzpicture}
  \node[root] {$\R^{14}$}
    child[parallel] { node[parallel] {$\R^{7}$}
        child[partial] { node[partial] {$\R^{3}$}
        }
    }
    child[parallel] { node[parallel] {$\R^{7}$}
        child[partial] { node[partial] {$\R^{3}$}
        }
    };
\end{tikzpicture}
        \caption{Scenario 2}
    \end{subfigure}

    \begin{subfigure}[b]{0.5\textwidth}
        \vskip 0pt
        \centering
        \begin{tikzpicture}[
    sibling distance=7.0em
    ]
  \node[root] {$\left(\SE{2}\times \R^{7}\right)^{3}$}
    child[parallel] { node[parallel] {$\SE{2} \times \R^{7}$}
        child[partial] { node[partial] {$\R^{1}$}
        }
    }
    child[parallel] { node[parallel] {$\SE{2} \times \R^{7}$}
        child[partial] { node[partial] {$\R^{1}$}
        }
    }
    child[parallel] { node[parallel] {$\SE{2} \times \R^{7}$}
        child[partial] { node[partial] {$\R^{2}$}
        }
    };
\end{tikzpicture}
        \caption{Scenario 3}
    \end{subfigure}

    \begin{subfigure}[b]{0.15\textwidth}
        \vskip 0pt
        \centering
        \begin{tikzpicture}
  \node[root] {$\SE{2}\times \R^{7} \times T$}
    child[sequential] { node[sequential] {$\SE{2} \times \R^{7}$}
        child[partial] { node[partial] {$\R^{3}$}
        }
    };
\end{tikzpicture}
        \caption{Scenario 4}
    \end{subfigure}

\caption{Fibration trees used for task-space evaluations. Colors indicate type of projection: Parallel~\parallelColorBox, Sequential~\sequentialColorBox, and Partial~\partialColorBox.\label{fig:fibration-trees-task-space-scenarios}}
\end{figure*}
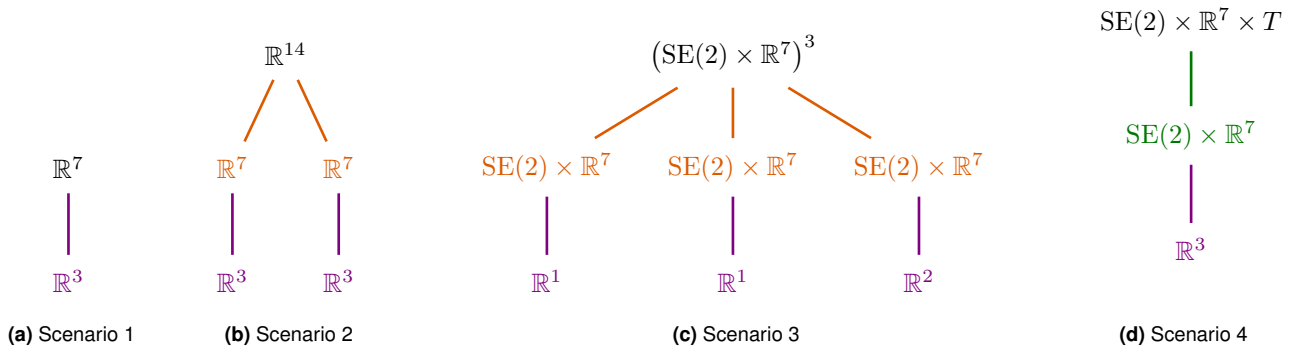

\subsubsection{Scenarios}
The task-space constraints in those four scenarios are limitations on the
movement of the Tool-center-point (Tcp).
In scenario Vertical Maze, a fixed manipulator robot is tasked to
move its Tcp through a maze while keeping the Tcp close to the surface of the wall as shown in Fig.~\ref{fig:benchmark_taskspace} (a).
This is modeled using a single partial fibration as shown in Fig.~\ref{fig:fibration-trees-task-space-scenarios}. The projection is defined as $\pi: X \rightarrow T$ with $X$ being the joint space, and $T$ being the position of the Tcp of the manipulator. The projection mapping itself is given by the forward kinematics from joint space to Tcp space. The lift of the projection is given by solving an inverse kinematics problem with a random seed. This is an admissible projection, since a valid joint space configuration has by definition a valid Tcp position. 

In scenario Fixed on Wall (Fig.~\ref{fig:benchmark_taskspace} (b))
, we combine parallel and partial fibrations by having
two fixed manipulator robots in close proximity, which both need to cross to
reach goal points along the wall. Both robots are constrained to keep the Tcp at all times close to the wall surface. This scenario is modeled using a fibration tree consisting of a parallel fibration and two partial fibrations (see Fig.~\ref{fig:fibration-trees-task-space-scenarios}). 

The third scenario is Mobile Manipulators
(Fig.~\ref{fig:benchmark_taskspace} (c)), which consists of three mobile manipulators
which have different Tcp constraints, whereby two manipulators need to follow a straight line along the coordinate
axes at a height of $0.5$m, while the third robot needs to keep its Tcp at a height
of $0.5$m, but not necessarily along a straight line. We model this using a parallel, and three partial fibration as shown in Fig.~\ref{fig:fibration-trees-task-space-scenarios}.  

In the last scenario, Path Velocity Decomposition (Fig.~\ref{fig:benchmark_taskspace} (d)), we use a time-based state space with a mobile manipulator which has to cross over a floor segment, where it faces both static and dynamic obstacles. We also constrain the manipulator to keep its Tcp at a height of $0.7 \pm 0.05$m above the floor. This scenario is modeled using a fibration tree with a single sequential fibration from the time-based state space to a time-less state space, plus a partial fibration onto the task-constraint (see Fig.~\ref{fig:fibration-trees-task-space-scenarios}).

To find motions which are consistent with time, we require two modifications to the steering method~\citep{kant1986toward, Grothe2022ICRA}. 
First, we need to ensure monotonicity, 
i.e. a connection between two elements $(s_1, t_1)$ and $(s_2, t_2)$ is only possible if $t_2 > t_1$. 
Second, we cannot move arbitrarily fast, i.e. we have to ensure that a maximum velocity is respected. 
This is achieved by constraining the slope $\left|\frac{ds}{dt}\right| < V_{max}$. For this scenario, the mobile robot has $V_{max} = 1.0$.

\subsubsection{Results}

The results are shown in Fig.~\ref{fig:benchmark_taskspace}.
They show that \FibrationRRT achieves a success rate of over $80$ percent
in $4$ out of $4$ cases and reaches a $100$\% success rate in $2$
cases. TaskRRT achieves a success rate of over $80$ percent
on $2$ out of $4$ cases and reach a $100$\% success rate in only one case for the Mobile Manipulators (Fig.~\ref{fig:benchmark_taskspace}c).

In the cases where \FibrationRRT reaches $100$\% success rate, it outperforms Task-RRT in $2$ cases by at least one order of magnitude in reaching the
$100$\% success rate. 
This is the case for Vertical Maze (Fig.~\ref{fig:benchmark_taskspace}a) and the Mobile Manipulators (Fig.~\ref{fig:benchmark_taskspace}c).

\section{Discussion and Limitations\label{sec:discussion}}

Let us discuss the results from the benchmarks in Sec.~\ref{sec:benchmarks}. We compared \FibrationRRT on different scenarios to showcase the flexibility and unified nature of fibration trees. 

The results indicate that both prioritization-based and decomposition-based
approaches are valuable
structures for multi-robot motion planning. While approaches which plan purely
in the composite state space have problems finding solutions, we were able to 
show that \FibrationRRT
can exploit both structures and reach results which improve runtime
significantly (Sec.~\ref{sec:benchmarks_structure}). This is in line with previous results on multi-robot motion
planning, where decomposition-based
approaches~\citep{Wagner2015,Shome2020,Dobson2017scalable,guo2026stac} and
prioritization-based
approaches~\citep{erdmann_1987,VanDenBerg2005Prioritized,Orthey2024IJRR} have
shown similar results and significant improvement on runtime. This again
underlines that \FibrationRRT can efficiently exploit a given fibration tree
structure.

A rather surprising result is the advantage of the decomposition-based approach
compared to a prioritization-based approach. While both approaches perform
similar in the structure benchmark (Sec.~\ref{sec:benchmarks_structure}) on scenarios (a,d,e), there is a significant difference between them
on the remaining scenarios with the decomposition-based approach outperforming
the prioritization-based approach on (b,c,f,h), while the opposite is true for
scenario (g). We believe this could be caused by some scenarios having fewer
inter-robot collisions which have to be resolved (which would be an advantage
for decomposition-based approaches because they would not need to resolve them). However, scenario (g) consists of a
difficult narrow-passage environment where most single-robot movements would
results in inter-robot collisions. In those types of scenarios, resolving robot
paths one by one seems to be advantageous to find a solution quicker.

\subsection{Limitations}

While we showed \FibrationRRT to perform well on all scenarios, there are still several improvements possible. 

\noindent\textbf{Handling of variable-volumetric fibers}
In restriction sampling, we handle all fibers as if they would have the same volume. Often, however, the volume of fibers might differ significantly and could lead to a highly non-uniform distribution of samples. A classical example would be a partial fibration where an inverse IK solver lifts workspace points to the configuration space. The fibers involved could represent subsets of the full configuration space, lower-dimensional manifolds, or even zero sets.
To handle such fibrations, we could add an additional distance check, which estimates the volume of a fiber by comparing samples lying on a single fiber. Depending on this check, we could adjust the global sampling distribution for graph- and path-restriction sampling accordingly. 


\noindent\textbf{Parallel execution of parallel fibrations}
Nodes in a parallel fibrations are usually independent, and could be computed in parallel to speed the method up. This could be handled by analyzing the fibration tree and creating a new thread for each decomposition split. This would create a set of independent subtrees, which each can then be run in parallel.

\noindent\textbf{Integration of specialized solutions} 
While we propose a unified framework, there are some cases which can be solved using specialized methods. This includes:
\begin{itemize}
    \item The pebbles-on-a-graph reduction for parallel fibrations with identical base spaces (homogeneous decompositions)~\citep{kornhauser_1984, yu_2016, ryan_2010, luna_2011}.
    \item Implicit graph search for parallel fibrations with non-identical base spaces (heterogeneous decompositions)~\citep{wagner_2015, Shome2020}.
    \item Advanced path section methods for sequential fibrations~\citep{Grey2017, Reid2019, Orthey2021TRO, Ichter2019LatentSpaces}.
    \item Reciprocal velocity obstacles for simple robot shapes in path-velocity decompositions~\citep{Vandenberg2011}.
    \item Specialized planners for path-velocity decompositions, like the admissible velocity propagation (AVP) method~\citep{pham2017admissible} to forward propagate admissible velocity profiles, or the space-time RRT (ST-RRT*)~\citep{Grothe2022ICRA} to handle state spaces with unbounded time.
\end{itemize}
    
Eventually, we aim to integrate all those specialized solutions as modules into our framework to increase the efficiency of \FibrationRRT, and to make it a one-stop solution for any multi-robot motion planning problem.

\noindent\textbf{Infeasible Problems}
\FibrationRRT can currently only deal with feasible problems, but it cannot finish in finite time on an infeasible problem. Recently, we used sparse roadmaps to find probabilistic guarantees that a problem is infeasible~\citep{Orthey2021ICRA}. A combination of \FibrationRRT with sparse planners~\citep{Simeon2000, dobson_2014} would be an effective extension to tackle infeasible problems. Another valuable approach would be to investigate exact infeasibility proofs in this context~\citep{Li2023Infeasibility}, which could be beneficial for lower-dimensional projection mappings.

\noindent\textbf{Asymptotically Optimal Planner}
While \FibrationRRT does not guarantee asymptotic optimality (AO)~\citep{Gammell2020Survey}, there is in theory nothing preventing us from extending it into an AO planner. However, while this is a fruitful future venue to explore, it is beyond the scope of the current paper, since the focus is on the general structure supporting multi-robot planners and how we can efficiently exploit them.

\noindent\textbf{Additional Heuristics} 
\FibrationRRT could benefit from additional heuristics like informed sets~\citep{Gammell2020}. However, the application is not straightforward, because pruning of states might prune edges belonging to the path restriction of the global optimal solution in the total space (root node of fibration tree). A better communication between nodes could alleviate this problem by propagating solution costs downwards.

\noindent\textbf{Informative Selection}
The current implementation of \FibrationRRT selects factors for expansion based on the exponential selection criterion as described in Sec.~\ref{sec:planner-selection}. However, as previous works have shown~\citep{Orthey2019ISRR, Orthey2021ICRA}, different selection methods could improve the planning time. In future work, we envision a selection unit, which collects data about all runs of all nodes in the fibration tree. In a subsequent step, this selection unit uses the collected data to choose a node which has the most utility of leading to a successful outcome. Such a system could also be learned from previous planning data.

\noindent\textbf{Interpolation and Differentiability}
The framework of fibration trees also provides us with the opportunity to
further exploit path restrictions for more efficient planning. We envision two
forms of this. First, interpolation methods could exploit valid paths on a base
space as additional information. This information could be used to generate more
valid paths on the bundle space in comparison to geodesic
interpolation~\citep{Orthey2018IROS}. Second, graph restrictions could be used
to add differentiable information to a planner similar to the artificial potential field approach~\citep{Khatib1986}. This is left for future work.

\section{Conclusion}

We proposed fibration trees, a unified approach to simplify and decompose high-dimensional, multi-robot motion planning problems. 
Our formulation allows us to directly compare prioritization-based and decomposition-based approaches and create a single algorithm (\FibrationRRT) which is independent of the underlying structure. 
In our benchmarks, we demonstrated that \FibrationRRT has minimal overhead compared to RRT, that it can handle fibration
trees of different kinds, and that we can use \FibrationRRT to directly compare
prioritization-based and decomposition-based approaches.
We also showed that \FibrationRRT performs similar to prioritization-based approaches on prioritization structures, and can outperform decomposition-based approaches when using a decomposition structure. Finally, we showed that 
task-space constraints can be integrated into fibration trees as another
fibration type. This comprehensive evaluation on 32 scenarios shows that \FibrationRRT can tackle a wide variety of scenarios and fibration trees, thereby underlying the unified nature of fibration trees for multi-robot motion planning problems.

We believe that fibration trees will not only simplify future benchmarking of
multi-robot planners, but could also lay the foundation for a unified
multi-robot planning method which is independent of the underlying structure. 
This could establish fibration trees as a general-purpose framework for multi-robot planning problems.

\bibliographystyle{style/SageH}
\balance
\bibliography{bib/general}

\appendix
\section{Motion Planning using the Hopf Fibration\label{appendix:hopffibration}}

The Hopf fibration~\citep{hopf1931abbildungen} is a prototypical example of a
fibration. It formalizes a projection from the sphere $\S^3$ (an object
embeddable in four-dimensional space) to the sphere $\S^2$ (embeddable in
three-dimensional space). The main idea is to decompose $\S^3$ into subspaces
(also called cosets~\citep{judson2020abstract}) of identical copies of $\S^1$. Each copy of $\S^1$ is then associated to a point on $\S^2$ (taking the quotient of $\S^3/\S^1$ and associating equivalence classes of the quotient with $\S^2$). 

Its relevance to motion planning comes from the double cover of \SO{3} by the
sphere $\S^3$. \SO{3} is equivalent (homeomorphic) to $\S^3$ with antipodal
points identified, i.e. $\SO{3} \cong \S^3 / \{x \sim
-x\}$~\citep{Yershova2010HopfFibration}. This can be exploited for planning
purposes like efficiently sampling $\SO{3}$~\citep{Yershova2010HopfFibration} or
by constructing controllers for UAVs which are less prone to
singularities~\citep{Watterson2020HopfFibration}. 

In our example, we use the Hopf mapping, whereby points on $\S^3$ are associated with quaternions, which are elements $q$, represented by tuples $(a,b,c,d)$ such that $q = a + b\cdot i + c\cdot j + d\cdot k$ and $\|q\|=1$. Those quaternions are mapped onto the $\S^2$ sphere as

\begin{equation}
  f(a,b,c,d) = \begin{bmatrix}
       \atantwo(x_1, x_0)\\
       \arccos(x_2/r)
  \end{bmatrix}
\end{equation}
whereby $x_0 = a^2 + b^2 - c^2 - d^2$, $x_1=2 (a d + b c)$, $x_3 =2 (b d - a c)$, and $r = \sqrt{x_0^2 + x_1^2 + x_2^2}$ is the radius of the sphere. For visualization purposes, we employ a second projection, the stereographic projection of the quaternion $q$ onto the 3-dimensional space as
\begin{equation}
  g(a,b,c,d) = \begin{bmatrix}
       b / (1 - a)\\
       c / (1 - a)\\
       d / (1 - a)
  \end{bmatrix}
\end{equation}.

Both those maps allow us to visualize quaternions in 3-dimensional space, either
by their associated point on the sphere $\S^2$ or by their stereographic projection.

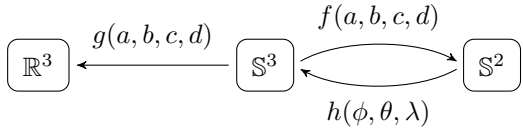
\begin{figure}[t]
    \centering
    \tikzset{
        bendArrow/.style={->,>=stealth',shorten >=3pt,shorten <=3pt},
        spaceNode/.style = {shape=rectangle, rounded corners,
        draw, align=center,
        top color=white}
    }
    \begin{tikzpicture}
        \node[spaceNode] (S3) at (0,0) {$\S^3$};
        \node[spaceNode] (S2) at (3,0) {$\S^2$};
        \node[spaceNode] (S3Proj) at (-3,0) {$\R^3$};
        \draw[bendArrow] (S3.east) to [bend left,looseness=0.8] node[above] {$f(a,b,c,d)$} (S2.west);
        \draw[bendArrow] (S2.west) to [bend left,looseness=0.8] node[below] {$h(\phi, \theta, \lambda)$} (S3.east);
        \draw[bendArrow] (S3.west) to node[above]{$g(a,b,c,d)$} (S3Proj.east);
    \end{tikzpicture}
    \caption{Hopf fibration with associated mappings. The Hopf map $f$ from $\S^3$ to $\S^2$, the inverse hopf map $h$ from $\S^2$ and fiber $\S^1$ to $\S^3$, and the stereographic projection $g$ from $\S^3$ to 3-dimensional space.}
    \label{fig:hopffibration-mappings}
\end{figure}

Finally, we require the inverse mapping from the sphere $\S^2$ to $\S^3$ to lift points from the sphere up to $\S^3$. This requires both a point on $\S^2$ as spherical coordinates and an element of $\S^1$ (the fiber). Let $\theta \in [-\pi,+\pi]$ and $\phi \in [0,\pi]$ be the coordinates of $\S^2$ (azimuth and elevation) and let $\lambda \in [0, 2\pi]$ be an element of $\S^1$. We can then lift the sphere element together with the fiber element by using
\begin{equation}
  h(\phi, \theta, \lambda) = \begin{bmatrix}
    -t_1 \sin(\lambda)\\
    +t_1 \cos(\lambda)\\
    t_2 \cos(\lambda) + t_3 \sin(\lambda)\\
    t_3 \cos(\lambda) - t_2 \sin(\lambda)
  \end{bmatrix}
\end{equation}
whereby $t_1 = \frac{(1+x)}{n}$, $t_2 = \frac{y}{n}$, and $t_3 = \frac{z}{n}$ with $n = \sqrt{2(1 + x)}$, $x = \sin(\phi) \cos(\theta)$, $y = \sin(\phi)\sin(\theta)$, and $z = \cos(\phi)$. Those mappings are summarized in Fig.~\ref{fig:hopffibration-mappings}.

\section{Proof of Completeness\label{appendix:proof-completeness}}

We provide here a proof of probabilistic completeness (PC) for
\FibrationRRT. The proof requires the following assumptions to hold:

\begin{enumerate}
    \item \textbf{Single Node Probabilistic Completeness} The underlying planner
      used to plan paths on a single node is assumed to be PC. 
      Any choice of tree-based~\citep{Kuffner2000} or
      roadmap-based planners~\citep{Kavraki1996} are possible. 
      \FibrationRRT is fulfilling this property in \texttt{PlanNode}, since it is based upon
      RRT~\citep{Kuffner2000}.
    \item \textbf{Uniform Infinite Recurrence}. Given a fibration tree, we
      assume that every node in a fibration tree is chosen infinitely many times
      when the number of iterations of the planner goes to infinity. 
      This is fulfilled by the \texttt{SelectNode} function for \FibrationRRT.
    \item \textbf{Liftable} Every fibration has to be liftable (or partially
      liftable) to ensure that information can be transfered upwards (see Sec.~\ref{sec:fibration-axioms}).
    \item \textbf{Admissibility} Every fibration in the fibration tree has to be admissible as
      detailed in Sec.~\ref{sec:fibration-axioms}. 
      This is ensured for \FibrationRRT by only using admissible projections throughout.
\end{enumerate}

Using those assumptions, we can prove the required property of \FibrationRRT.

\begin{theorem}[Probabilistic Completeness of \FibrationRRT]
    
Given Assumptions 1--4, together with a motion planning problem and an arbitrary fibration
  tree (as defined in Sec.~\ref{sec:fibration-trees}), \FibrationRRT will find a feasible path if one exists, with probability one when the number of iterations goes to infinity.

\end{theorem}

\begin{proof}

To prove PC for \FibrationRRT, we rely on the PC proofs given for RRT-like
  planners~\citep{berenson2011task, solovey2020revisiting}, which show PC by relying on a two-step method. In the first step, it is shown (or assumed) that a dense sampling sequence is given on the state space. In the second step, the series-of-balls argument is used. This involves assuming that a feasible path with clearance $\epsilon$ exists, and then covering this path with balls of size $\delta < \epsilon$. This guarantees that the $\delta$ neighborhood of the path has positive volume and is constrained-free. Using mathematical induction, it is then shown that the first ball is reached by the tree-based planner (base case), and it is then shown that ball $k+1$ is reached once ball $k$ is reached (induction step).
This relies on the fact that the dense sampling sequence will, eventually, sample a point in any positive volume ball in the state space. In fact, we can make a stronger statement: RRT-like algorithms are PC if the provided sampling sequence is dense in the \emph{constrained-free} state space. This is enough, since the balls are assumed to be lying exclusively in the constrained-free state space.

Our argument for PC for \FibrationRRT builds upon this proof. We argue that \FibrationRRT is PC given any fibration tree, because (a) the planner on the root node is PC, and (b) since the restriction sampling is dense on the \emph{constrained-free} state space. Let us show that both (a) and (b) are true. Property (a) is true by definition since RRT is used to plan on a single node and since by the uniform infinite recurrence (Assumption 2), this RRT is run for infinitely many iterations. It remains to show that (b) is true.

To show that restriction sampling is dense on the \emph{constrained-free} state space, however, requires a bit of work. 
We formulate here a proof based upon structural induction on tree heights~\citep{cormen2022introduction}. We divide this into two sections, whereby we first explain how structural induction extends induction, and we then use this to prove restriction sampling to be dense.

\noindent\textbf{From Induction to Structural Induction}

In induction, we try to prove a property $P$ defined on the natural numbers $\N$ which come with their standard ordering $<$. To prove $P$ for any $n \in \N$, we need to prove two cases: First we prove the \emph{base case}, namely that $P(0)$ is true. Second, we prove the \emph{induction step} that if $P(n)$ is true, then so is $P(n+1)$.

Structural induction generalizes this statement from the natural numbers to
  ordered structures like trees~\citep{cormen2022introduction}. If we want to prove a property $P$ on a tree $T$, we first define an ordering of trees by defining the ordering relation $T < T'$ if $T$ is a subtree of $T'$. Given this ordering, we can prove properties of a tree $T$ by proving the \emph{base case}, namely that it is true for a single leaf-node tree. Once this is proven, we can prove the \emph{induction step} namely that if the property is true for a finite set of trees $T_1,\ldots,T_N$, then it is true for $T'$, which is the tree obtained by joining $T_1,\ldots,T_N$ under a new root node. 

\noindent\textbf{Denseness of Restriction Sampling}

Let us prove, for any fibration tree, the property "Restriction sampling is dense on the constrained-free configuration space" (P1). To do that, we need to prove P1 on a single leaf node (base case), and we need to prove that P1 holds on a fibration tree $A$ which is obtained as the root node of a set of fibration trees  $B_1,\ldots,B_N$ on which P1 holds (induction step).

The base case is straightforward, since \FibrationRRT plans only on a single state space, thereby operating as the underlying RRT planner~\citep{Kuffner2000}, which is using uniform sampling, which is dense on the state space. 

The main work lies in the induction step. Let us assume that P1 is true for a set of trees $B_1, \ldots, B_N$, and that we want to show that P1 still holds on tree $A$, obtained by joining $B_1, \ldots, B_N$ under a new root node. There are three cases which we have to look at.

\begin{enumerate}
    \item \textbf{Sequential Fibration}. When $A \rightarrow B$ is sequential
      fibration, there exists only a single subtree by definition. As was shown
      in~\citep{Orthey2019ISRR}, restriction sampling on $A$ is guaranteed to
      produce a sample in any positive volume ball on the
      \emph{constrained-free} space of $A$. This is guaranteed since a
      projection of such a ball lies strictly inside the \emph{constrained-free}
      space of $B$, since the projection onto $B$ is admissible (see
      Sec.~\ref{sec:fibration-axioms}). Since the uniform infinite recurrence property holds (Assumption 2), we will, eventually, obtain a sample in any positive volume ball. See~\citep{Orthey2019ISRR} for details. 
    \item \textbf{Partial Fibration}. Partial fibrations are similar to
      sequential fibrations. However, there is one important detail, namely that
      one has to ensure that the \texttt{SampleFiber} function is dense. For
      example, if the projection is onto the end-effector, we need to ensure
      that all possible inverse kinematics solution are sampled eventually. The
      remaining argument follows the argument above for sequential fibrations.
    \item \textbf{Parallel Fibration}. Let $A \rightarrow A_1 \times \cdots \times A_M$ be a parallel fibration. By induction assumption, the sampling sequence on $A_1,\ldots,A_M$ is dense on the constrained-free space. Let $D$ be a positive-volume ball on the constrained free space $A$. Then, by admissibility (Assumption 3), the projection of $D$ onto $A_1,\ldots,A_M$ yields $M$ balls $D_1,\ldots,D_M$, which are strictly inside the respective constrained free spaces. Since the sampling sequences on those spaces are dense by induction assumption, we will, eventually, find samples in each ball $D_1\ldots,D_M$, which can be lifted into $D$.
\end{enumerate}

We have thus proven the induction step and have shown that restriction sampling
  is dense on any arbitrary fibration tree (as defined in
  Sec.~\ref{sec:fibration-trees}) which uses liftable and admissible fibrations. Thus, restriction sampling acts as a dense sampler in the constrained configuration space of the root node. Since \FibrationRRT uses RRT on the root node, this thereby shows that \FibrationRRT is PC.\hfill\qedsymbol{}
\end{proof}

\end{document}